\documentclass{article} %
\usepackage[left=1in,bottom=1in,top=1in,right=1in]{geometry}

\usepackage{amsmath,amsfonts,bm}

\def\eqref#1{equation~\ref{#1}}

\def\1{\bm{1}}

\def\vb{{\bm{b}}}

\def\vg{{\bm{g}}}

\def\vk{{\bm{k}}}

\def\vq{{\bm{q}}}

\def\vv{{\bm{v}}}
\def\vw{{\bm{w}}}
\def\vx{{\bm{x}}}
\def\vy{{\bm{y}}}
\def\vz{{\bm{z}}}

\def\mA{{\bm{A}}}

\def\mC{{\bm{C}}}

\def\mG{{\bm{G}}}

\def\mI{{\bm{I}}}

\def\mK{{\bm{K}}}

\def\mW{{\bm{W}}}
\def\mX{{\bm{X}}}
\def\mY{{\bm{Y}}}

\DeclareMathAlphabet{\mathsfit}{\encodingdefault}{\sfdefault}{m}{sl}
\SetMathAlphabet{\mathsfit}{bold}{\encodingdefault}{\sfdefault}{bx}{n}

\newcommand{\layer}{L}
\newcommand{\fixed}{\mathcal{I}}
\newcommand{\stage}{k}

\newcommand{\funcsqrt}{h_{sqrt}}

\newcommand{\RR}{\mathbb{R}}
\newcommand{\dims}{d}
\newcommand{\hid}{m}

\newcommand{\bernoulli}{\zeta}
\newcommand{\seqlen}{N}
\newcommand{\relu}{\sigma_{relu}}

\newcommand{\vgamma}{\mathbf{\gamma}}

\newcommand{\mlpfun}{f^{mlp}}
\newcommand{\attnfun}{f^{attn}}
\newcommand{\lnfun}{f^{LN}}
\newcommand{\mlplnfun}{\lnfun_{mlp}}
\newcommand{\attnlnfun}{\lnfun_{attn}}

\newcommand{\outlnattn}{\vy^{attnln}}
\newcommand{\outattn}{\vy^{attn}}
\newcommand{\outattnblock}{\vy^{attnblock}}

\newcommand{\outlnmlp}{\vy^{mlpln}}
\newcommand{\outmlp}{\vy^{mlp}}

\newcommand{\distrib}{D}
\newcommand{\embed}{\Phi}

\newcommand{\vocab}{V}
\newcommand{\netloss}{\mathcal{L}}

\newcommand{\normal}{\mathcal{N}}
\newcommand{\head}{H}

\newcommand{\expect}{\mathbb{E}}

\newcommand{\norm}[1]{\left\|#1\right\|}
\newcommand{\abs}[1]{\left|#1\right|}

\newcommand{\layerdrop}{subnetwork}
\newcommand{\method}{\textsc{RaPTr}}

\newcommand{\liploss}{\mu_{\netloss}}

\newcommand{\project}{\Pi}

\newcommand{\sphere}{\mathbb{S}}

\newcommand{\bertbase}{\textsc{BERT-Base}}
\newcommand{\bertlarge}{\textsc{BERT-Large}}

\newcommand{\progressivelayerdrop}{\text{PLD}}

\usepackage{url}
\usepackage[colorlinks,
            linkcolor=red,
           anchorcolor=blue,
            citecolor=blue,
            linktocpage
            ]{hyperref}
\usepackage{amssymb,amsmath,amsthm}
\usepackage{thmtools} 
\usepackage{thm-restate}
\usepackage{graphicx}
\usepackage{subcaption}
\usepackage{array}
\usepackage{booktabs}
\usepackage[ruled]{algorithm2e}
\usepackage{cleveref}
\usepackage{multicol}
\usepackage{multirow}
\usepackage{minitoc}
\usepackage{natbib}
\usepackage{xcolor}
\usepackage{algorithmic}

\usepackage{enumitem}

\theoremstyle{plain}
\newtheorem{theorem}{Theorem}[section]

\newtheorem{lemma}[theorem]{Lemma}
\newtheorem{corollary}[theorem]{Corollary}
\theoremstyle{definition}
\newtheorem{definition}[theorem]{Definition}
\newtheorem{assumption}[theorem]{Assumption}
\theoremstyle{remark}

\definecolor{byzantine}{rgb}{0.74, 0.2, 0.64}
\newcommand{\rebuttal}[1]{#1}

\newcommand{\apnote}[1]{{\color{red}[]}}%
\newcommand{\ns}[1]{{\color{orange}[]}}%
\newcommand{\sm}[1]{{\color{pink}[]}}%
\newcommand{\snote}[1]{{\color{blue}[]}}%
\newcommand{\srnote}[1]{{\color{blue}[]}}%
\newcommand{\sknote}[1]{{\color{magenta}[]}}%

\author{Abhishek Panigrahi$^{*,  \dagger, \beta}$, Nikunj Saunshi\footnote{Equal Contribution}  
 $^{,\alpha}$, Kaifeng Lyu\footnote{Work done during Google internship}  $^{,\beta}$ , Sobhan Miryoosefi$^{\alpha}$, \\ Sashank Reddi$^{\alpha}$, Satyen Kale$^{\alpha}$, Sanjiv Kumar$^{\alpha}$ \\
$^{\beta}$Department of Computer Science, Princeton University , \\ $^{\alpha}$ Google Research, New York \\
\texttt{ap34@princeton.edu}, \texttt{nsaunshi@google.com}
}

\title{Efficient Stagewise Pretraining via Progressive Subnetworks}

\begin{document}

\maketitle
\begin{abstract}
Recent developments in large language models have sparked interest in efficient pretraining methods. Stagewise training approaches to improve efficiency, like gradual stacking and layer dropping \citep{reddi2023efficient, zhang2020accelerating}, have recently garnered attention. 
The prevailing view suggests that stagewise \emph{dropping} strategies, such as layer dropping, are ineffective, especially when compared to stacking-based approaches. This paper challenges this notion by demonstrating that, with proper design, dropping strategies can be competitive, if not better, than stacking methods. 
Specifically, we develop a principled stagewise training framework, \emph{progressive {\layerdrop} training}, which only trains subnetworks within the model and progressively increases the size of subnetworks during training, until it trains the full network. We propose an instantiation of this framework --- \textbf{Ra}ndom \textbf{P}art \textbf{Tr}aining ({\method}) --- that selects and trains only a random subnetwork (e.g. depth-wise, width-wise) of the network at each step, progressively {\em increasing} the size in stages. We show that this approach not only generalizes prior works like layer dropping but also fixes their key issues. Furthermore, we establish a theoretical basis for such approaches and provide justification for (a) {\em increasing} complexity of subnetworks in stages, conceptually diverging from prior works on layer dropping, and (b) {\em stability} in loss across stage transitions in presence of key modern architecture components like residual connections and layer norms. Through comprehensive experiments, we demonstrate that {\method} can significantly speed up training of standard benchmarks like BERT and UL2, up to 33\% compared to standard training and, surprisingly, also shows better downstream performance on UL2, improving QA tasks and SuperGLUE by 1.5\%; thereby, providing evidence of better inductive bias.

\end{abstract}

\vspace{-0.25em}
\section{Introduction}\label{sec:introduction}
\vspace{-0.25em}

\looseness-1Large network based language models (e.g. Transformers) have revolutionized the field of NLP. Intriguingly, these language models have demonstrated remarkable \emph{emergent abilities} that only begin to manifest at large scales~\citep{wei2022emergent,schaeffer2023are}. However, training such large models is usually very slow and resource intensive~\citep{brown2020language,touvron2023llama,chowdhery2022palm}.
This has sparked interest in efficient training of large models, necessitating the development of new algorithms and paradigms for efficient pretraining. Traditionally, this was accomplished by designing better optimization algorithms (e.g., \citep{chen2023symbolic,gupta2018shampoo,shazeer2018adafactor,liu2023sophia}) that require fewer training steps to reduce the loss.
Recently, other paradigms based on \emph{stagewise training}, especially on depth, have garnered interest. Two such contrasting paradigms are: (a) layer stacking and (b) layer dropping.

Stacking based approaches have been studied extensively since \citet{chen2021bert2bert} applied it for BERT. Progressive~\citep{gong2019efficient} and gradual stacking~\citep{reddi2023efficient} are layer stacking methods that train large models in stages, starting with a small network and gradually growing the network size by stacking a subset of layers onto itself from previous stages. 
Although effective in reducing FLOPs and training time, their performance is sensitive to stacking schedules and require careful tuning. 
Furthermore, since the model grows gradually, it is not possible to assess the full model performance during earlier stages (i.e., it is not an anytime algorithm).
Additionally, using small model with way fewer parameters for part of the training can possibly hamper model quality (e.g. long term memory \citep{geva-etal-2021-transformer}), especially for single epoch training where each data sample is seen once during training. In our experiments(\textsection{\ref{sec:ul2_1b}}), we observe  that the pretraining loss and downstream performance may even be worse than baseline (full-model) training.

\looseness-1Layer dropping, on the other hand, maintains the model's identity but drops layers at random during training. This area of study remains largely unexplored. The closest work for training efficiency is progressive layer dropping (PLD) \citep{zhang2020accelerating}, a \emph{heuristic approach} where FLOPs are saved by increasingly {\em dropping more layers} as training proceeds, thus, decreasing its overall capacity over time. Such a strategy to improve training efficiency typically comes at the expense of quality, and is generally worse than stacking  approaches~\citep{kaddour2023no}.
Through theoretical analysis and experiments, we first show that there is a fundamental problem with existing dropping strategies like PLD since dropping more layers later during training can be detrimental to \emph{learning complex features}. 
Other instantiations~\citep{fan2019reducing,zhang2019fixup,liu2023dropout} do not achieve training efficiency, as they were proposed as a regularization strategy or for improving inference efficiency. 
This naturally prompts the central question of the paper:
\begin{quote}
    \emph{Is it possible to design principled and robust stagewise 
    dropping based techniques that are competitive or better than stacking techniques?} 
\end{quote}
\looseness-1 In this paper, we propose a novel generalization of stagewise dropping approaches called {\em progressive subnetwork training} that \textit{fundamentally addresses the above concerns}, and provides an affirmative answer to this question. The key components of this framework are:
\begin{enumerate}
    \item[(P1)] maintain a common base model of interest and train {\em subnetworks of varying complexities} within the model. 
    \item[(P2)] progressively {\em increase} the complexity of subnetworks in each stage to explicitly impose this simple-to-complex behavior. 
\end{enumerate}
 For (P1), one can use very general subnetworks in this framework (e.g. subset of layers or layers with smaller width). This strictly generalizes layer dropping techniques where just layers are dropped to derive efficiency. Furthermore, (P2) is motivated by the theoretically and empirically studied phenomenon that gradient-based training learns functions of increasing complexity over time~\citep{kalimeris2019sgd,abbe2022merged} and starkly diverges from earlier layer dropping ideas where more layers are dropped as training proceeds. As we shall see soon, even a simple instantiation  of this idea can already be competitive or better than stacking approaches that have been extensively studied. %

Given the above framework, a natural question arises: how do we select the subnetworks? As a concrete instantiation, we choose subnetworks to be  \emph{random} parts of the network, and in each stage we perform forward and backward passes over them. The size of the subnetworks is gradually increased over time with a fixed schedule.
This approach, which we call \textbf{Ra}ndom \textbf{P}art \textbf{Tr}aining (\method), reduces the total number of FLOPs and wall-clock time in the initial stages, but unlike stacking, every parameter of the full model contributes to the training loss at all stages, and it allows us to track the full model throughout training.
In this paper, we demonstrate that this simple, yet powerful, strategy not only has a solid theoretical foundation, but also shows strong experimental results for BERT and UL2 pretraining, both in terms of training efficiency and improving model quality. We summarize the main contributions of the paper below:

\begin{itemize}[leftmargin=*]
    \item We introduce a novel stagewise training called progressive subnetworks that generalizes prior dropping strategies in \textsection{\ref{sec:grad_layerdrop}}. 
Specifically we explore Random Part training ({\method}) that trains a part comprised of a random subnetwork, with the size of the subnetwork progressively increasing in stages. We specifically  study variants of {\method} where the random subnetwork selection is restricted to single model axis e.g. depth (\cref{sec:raptr}) or intermediate MLP width (\cref{sec:beyond-depth}). We leave the combination of multiple axes for future investigation. 

    \item \looseness-1Through analysis on polynomial data in \textsection{\ref{sec:analysis_poly}}, we demonstrate a fundamental problem with prior dropping techniques like progressive layer dropping -- dropping more layers towards the end can hurt the models' ability to capture complex correlations in the data. Furthermore in this polynomial setting, we show, both theoretically and empirically,  that  \method{} can effectively learn higher order components of the underlying ground-truth function much better than progressive layer dropping. This provides a strong justification for progressively increasing complexity of subnetworks in stages. The polynomial setting can potentially aid future research in efficient training.
    
    \item  \looseness-1We conduct extensive experiments on BERT \citep{devlin2018bert} and UL2 language models  \citep{tay2022ul2} to compare a depth instantiation of \method{} with dropping techniques (PLD), and gradual stacking and its improved variants.
    On BERT-Base (\textsection{\ref{sec:bert}}), {\method} demonstrates competitive performance to state-of-the-art stacking strategies at similar FLOPs, while being better than baseline\footnote{We refer to full-model training as baseline training in our paper} training with $1.33\times$ fewer FLOPs. 
    For UL2-$1.6B$ language model (\textsection{\ref{sec:ul2_1b}}), {\method} matches the pretraining perplexity of baseline training by requiring $1.2\times$ fewer FLOPs.
    Furthermore, despite having the same perplexity, \method{} has much better performance on a suite of $12$  downstream benchmarks, evaluated in $1$-shot and $5$-shot in-context settings. This suggests a desirable inductive bias of \method{} that improves its generalization abilities beyond perplexity.
    
    \item We discuss a novel implementation strategy for subnetwork training translates FLOPs improvements to wall-clock speedups. This is particularly valuable for distributed training where naive conditional dropping implementations either fail or do not speedup training (\cref{sec:efficient_raptr_impl}).
    
    \item \looseness-1 We establish the \emph{first} theoretical basis for stagewise training based on dropping of layers that studies the behavior at stage transitions. In particular, we characterize stability conditions under which {\method} yields smooth transition in loss across stage transitions (\textsection{\ref{sec:theory}}). We connect this to the idea of loss preservation \citep{chen2021bert2bert} and show that \method{} has an even {\em stronger property of loss improvement} in each stage. Through an illustrative setting with linear networks, we theoretically demonstrate the importance of modern architecture components, like layernorms and residual connections, for such stability.

\end{itemize}

\section{Progressive {\layerdrop} Training} \label{sec:grad_layerdrop}

{\bf Notation.} We use $[n]$ denotes the set $\{1, 2, \dots, n\}$. $a_{1:k}$ is used to denote a sequence of $k$ scalars $a_1, \dots, a_k \in \RR$, and $\vx_{1:k}$ to denote a sequence of $k$ vectors $\vx_1, \dots, \vx_k \in \RR^{\dims}$. We alternately also use $\mX \in \RR^{n \times \dims}$ to denote a matrix containing a sequence  $\vx_{1:n}$ as rows. 

We consider the setting where the goal is to
learn an $\layer$-layer sequence-to-sequence neural network. (e.g. standard Transformer network).

\begin{definition}[$\layer$-layer sequence-to-sequence model] \label{def:functiondef}
    Let $f_{1}, \cdots, f_{\layer}: \RR^{\seqlen \times \dims} \to \RR^{\seqlen \times \dims}$ denote a set of $\layer$ sequence-to-sequence functions 
    Then, $F : \RR^{\seqlen \times \dims} \times \RR^{\layer} \to \RR^{\seqlen \times \dims}$  with a sequence input $\mX \in \RR^{\seqlen \times \dims}$ and $\layer$ scalars $\alpha_{1:\layer}$, outputs a sequence $\mY$ after a recursive computation with intermediate sequences $\mY^{(1)}, \cdots, \mY^{(\layer-1)}, \mY^{(\layer)}$, given by
    \begin{align*}
        \mY^{(\ell)} = \mY^{(\ell-1)} + \alpha_j f_{j} ( \mY^{(\ell-1)} ) \text{ for all } 1 \le \ell \le \layer.
    \end{align*}  
    Here $\mY^{(0)}$ denotes $\mX$ for simplicity and output $\mY$ is equal to $\mY^{(\layer)}$.
\end{definition}
Standard model training is represented by the backbone function $F$, where $\alpha_i=1$ for $i \in [L]$. For simplicity, given an input sequence $\mX$, we will use $F(\mX)$ as the output and $\mY^{(1)}, \mY^{(2)}, \cdots, \mY^{(\layer)}$ as the intermediate layer outputs of $F$ under standard model training, unless specified otherwise. The output of the model is passed through a function $\head: \RR^{\dims} \to \RR^{\vocab}$ that projects the model output into $V$-dimensional logits, where $V$ is the vocabulary size.
For an example $\mX$, cross-entropy loss based on true labels is denoted as $\netloss(\head\circ F(\mX))$ (or $\netloss(F(\mX))$ when $\head$ is clear from context). Note that this notation hides the true labels for the ease of exposition. The final loss is then defined as $\netloss(F) = \expect_{\mX\sim\mathcal{D}}[ \netloss(F(\mX))]$ for the input distribution $\mathcal{D}$. 

Given such a setting, \emph{progressive subnetwork training}, a stagewise training framework, consists of two crucial components:
\begin{enumerate}
    \item {\bf Subnetwork selection \& rescaling:} At each iteration, we select a (possibly random) subnetwork of the $\layer$-layer neural network. The forward and backward passes are \rebuttal{computed} based on this subnetwork network along with appropriate scaling (e.g. Section~\ref{sec:raptr})).
    \item {\bf Progressive subnetworks:} In stages, we progressively increase the size of the subnetworks starting from small subnetworks and end with the complete network in the final stage.
\end{enumerate}

We note that this framework is quite general. To make it more concrete, in the following section, we provide a simple but powerful instantiation of the framework based on \emph{random paths}.

\subsection{Random Path Training: {\method}} 
\label{sec:raptr}

One simple approach to select a subnetwork is by skipping a few layers and choosing a \emph{path}. More concretely, given a network $F$ to be trained, we pick a random subset of layers for the forward pass and bypassing the rest of the layers using the residual connections.
Let $p$ denotes the probability of randomly selecting a layer {\layerdrop} and $\fixed$ denotes a subset of layers that are always included, i.e. they are never bypassed. We define the following before describing $\method$ in more detail.

\begin{definition}[$(p, \fixed)$-subnetwork]
\label{def:subnetwork}
Let $\bernoulli_{1:\layer}$ be Bernoulli samples, with $\bernoulli_i = 1$ for all $i \in \fixed$, while $\bernoulli_i \sim B(p)$ for all $i \not\in \fixed$.
The set of all layers for which $\zeta_i = 1$ represents a random {\layerdrop} while layers with $\zeta_i = 0$ are bypassed. The output of the selected {\layerdrop} on a sequence input $\mX$ is equivalent to the output of the model given by $F (\mX, \bernoulli_{1:\layer})$.

\end{definition}

\begin{minipage}{0.45\textwidth}
\begin{center}
    \includegraphics[width=\textwidth]{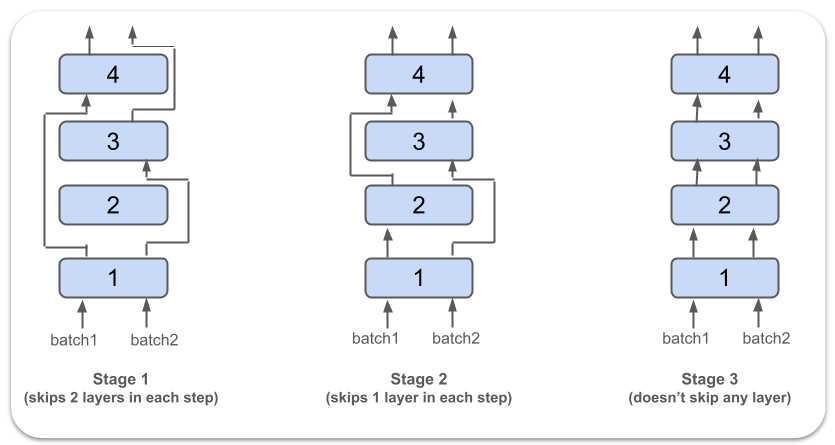}
    \captionof{figure}{Pictorial description of stagewise \method{} where the number of layers being skipped progressively decreases over stages.}
    \label{fig:raptr_desc}
\end{center}
\end{minipage}\hfill
\begin{minipage}{0.45\textwidth}

\scalebox{.9}{
\begin{algorithm}[H]
\small
\caption{$\stage$-stage {\method}}\label{alg:layerdrop}

\begin{algorithmic} 
\REQUIRE Blocks $f_{\theta_1}, \cdots, f_{\theta_L}$, Schedule $ T_{1:\stage} \times (p_{1:\stage}, \fixed_{1:\stage})$, total steps $T$, data distribution $\distrib$. 
\STATE Initialize $\theta_1, \cdots, \theta_L$, $T_{\stage+1}=T$.
\FOR{$s=1 \to \stage$}
\FOR{$t=T_s \to T_{s+1}$}
\STATE Sample batch $(\mX, \mY) \sim \distrib$.
\STATE Sample $\bernoulli_i \sim B(p_s)$ for $i \not\in \fixed_s$, set $\bernoulli_i = 1$ for $i \in \fixed_s$.
\STATE Set $\mY^{(0)} \leftarrow \mX$
\FOR{$j = 1 \to \layer$}
\STATE $\mY^{(j)} \leftarrow \mY^{(j-1)}$
\STATE \textbf{if} ~ $\bernoulli_{j}$: $\mY^{(j)} \leftarrow \mY^{(j)} + f_{\theta_j}( \mY^{(j-1)} )$
\ENDFOR
\STATE Compute loss with $\mY^{(L)}$ and take a optimizer step.
\ENDFOR 
\ENDFOR
\end{algorithmic}
\end{algorithm}
}
\end{minipage}

The pseudo-code for $\method$ is provided in Algorithm~\ref{alg:layerdrop}.
On a high level, the total training is split into $k$ stages. Each stage $s$ uses $(p_s, \fixed_s)$-{\layerdrop} between steps $T_{s}$ and $T_{s+1}$. We denote this schedule by $T_{1:\stage} \times (p_{1:\stage} ,\fixed_{1:\stage})$.  Importantly, we progressively increase the random {\layerdrop} selection pattern across stages i.e., the expected length of the selected random subnetworks is increased progressively towards full model training. This can be achieved by either increasing the probability of including each layer in random {\layerdrop} or fixing more layers in $\fixed_{\cdot}$ or both. More formally, we use schedules $T_{1:\stage} \times (p_{1:\stage}, \fixed_{1:\stage}) $  that satisfies: (a) $p_s \le p_{\Tilde{s}}$, and (b) $\fixed_s \subseteq \fixed_{\Tilde{s}}$ for all $1 \le s \le \Tilde{s} \le \stage$.

Note that each training step in {\method} involves computing forward and backward passes only on random subnetwork. In expectation, training a $(p, \fixed)$ random {\layerdrop} will require $\frac{\abs{\fixed} + (\layer - \abs{\fixed})p}{\layer}$ FLOPs relative to standard training that uses all $\layer$ layers. By using $p_s \ll 1$ for most part of training, $\method$ can significantly increase the training efficiency.

\section{Illustration: Problems with earlier subnetwork training instantiations}
\label{sec:analysis_poly}

\begin{figure*}[t]%
\centering
    \begin{subfigure}{0.23\textwidth}
    \centering    \includegraphics[width=\textwidth]{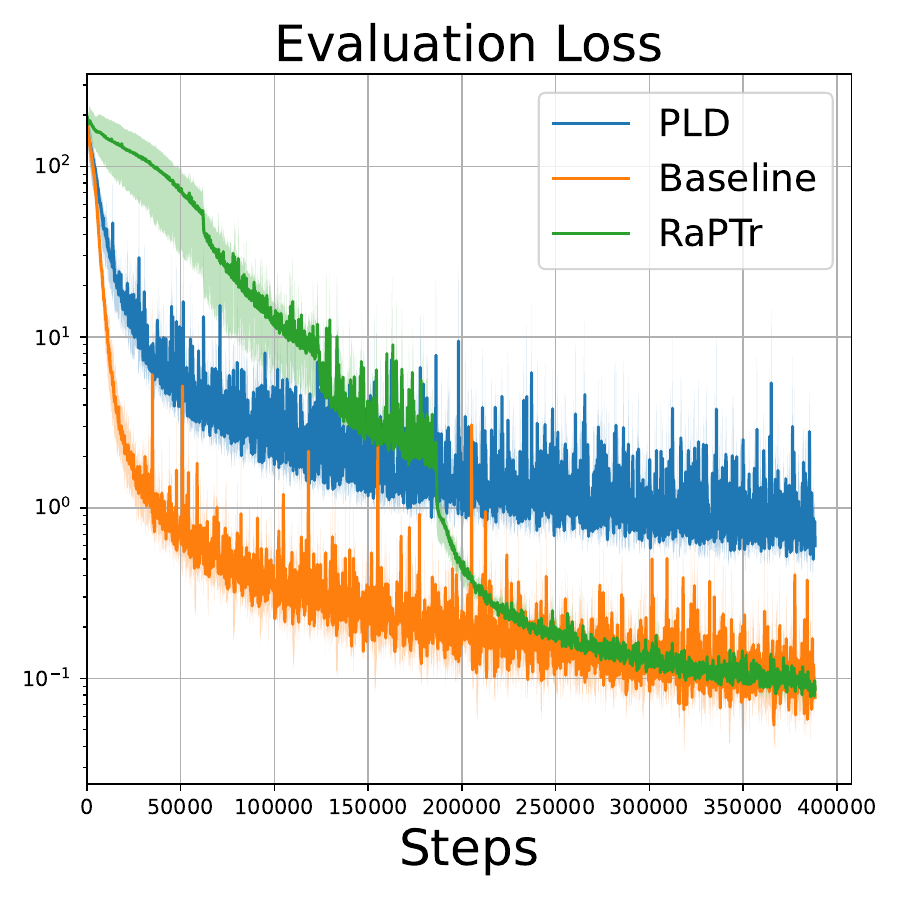}
    \label{fig:adam_comp_acc}
    \end{subfigure}\hfill
    \centering
    \begin{subfigure}{0.69\textwidth}
    \centering    \includegraphics[width=\textwidth]{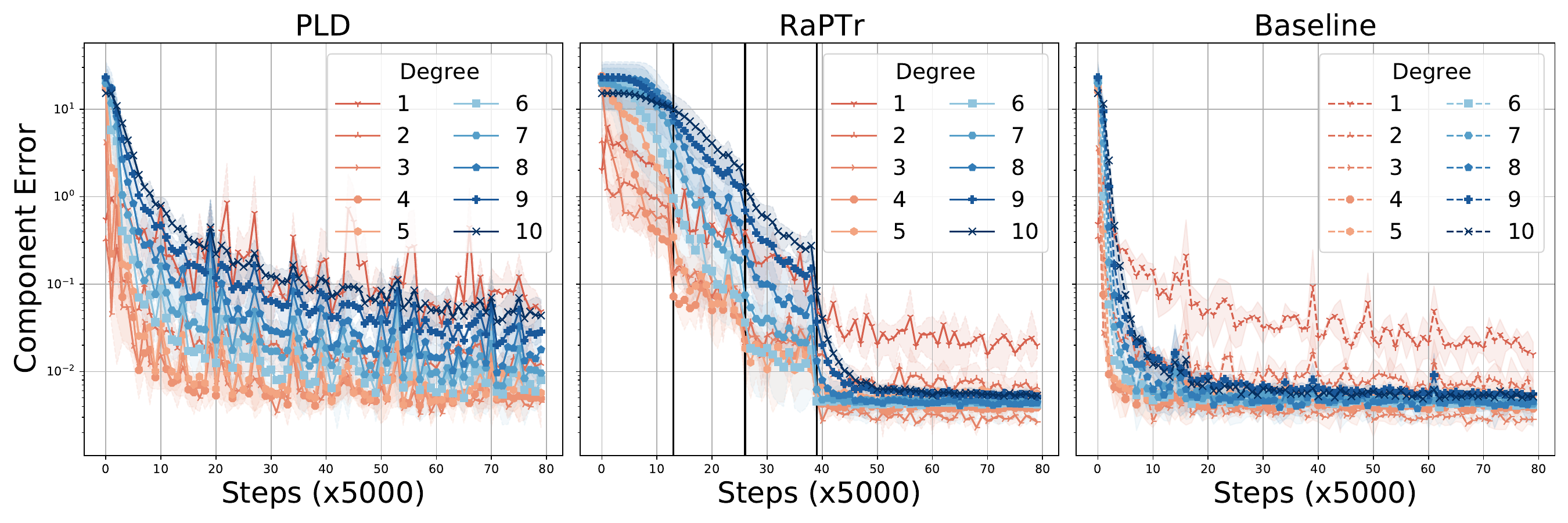}
    \label{fig:adam_comp}
    \end{subfigure}\hfill
    \caption{\looseness-1Evaluation loss (left) and component error (3 on the right) on basis polynomials of different degrees for a $20$ layer residual network trained with different methods. Labels are generated from a composition of polynomials of degrees $1$ to $10$ (Eq.\ref{eq:true_label}). The schedules for \method{} and \progressivelayerdrop{} are selected to have 20\% fewer FLOPs compared to baseline (phase transitions for \method{} have been marked with dark vertical lines). (more details in \textsection{\ref{sec:toy_setting}}). 
    Observations: (a) \method{} reaches same evaluation loss of baseline, while \progressivelayerdrop{} performs much worse. (b) \method{} learns lower order terms faster and picks up higher order terms in the later stages. \progressivelayerdrop{} is worse at capturing higher degree terms owing to its reduced expressivity towards the end. }
    \label{fig:raptr_v_layerdrop_toy}
    \vspace{-0.1in}
\end{figure*}

In this section, we discuss problems with existing subnetwork training instantiations like progressive layer dropping (\progressivelayerdrop{}).
Recall that \progressivelayerdrop{} drops {\em more} layers as training proceeds to enable FLOPs efficiency (see \textsection{\ref{sec:progressivedrop}} for details). This is in contrast to the depth variant of $\method$, that drops fewer layers as training proceeds.  While this difference may seem superficial at first, it has deep implications from both conceptual and performance standpoints. Here, we argue through theoretical and empirical analysis on polynomial data, that this can have very significant impact on learning complex correlations in data. In the BERT and UL2 experiments in \textsection{\ref{sec:experiments}}, we similarly observe that progressive layer dropping performs significantly worse than \method{}.

\looseness-1 We demonstrate the importance of progressively {\em increasing} the complexity of subnetworks in \method{} (as opposed to progressively decreasing it in PLD) through a simple, yet instructive, setting: learning neural networks on polynomial data. For simplicity, we restrict the discussion to sequences of length $\seqlen = 1$. 

{\bf Polynomial learning setting.} Suppose $\vx$ are sampled from a uniform distribution over $\{\pm1\}^{\dims}$. The ground-truth label is a polynomial $f^{\star}$ of degree $k$ defined as 
\vspace{-0.1in}
\begin{align}
    F^{\star} (\vx) = \sum_{l=1}^{k} \sum_{j=1}^{m} c_{l, j} \prod_{i \in S_{l, j}} x_{i}, \label{eq:true_label}
\end{align}
where $c_{l, j} \sim \mathcal{N}(0, 1)$ and $S_{l, j}$ are random subsets of $[d]$ of cardinality $\ell$. Such polynomials have been studied in great detail \citep{abbe2022merged,abbe2023sgd} where the higher degree terms capture more complex correlations and are harder to learn. We consider training of $20$-layer residual network \cite{he2016deep} $F$ on such data using square loss: $\mathbb{E} (F(\vx) - F^{\star}(\vx))^{2}$. Here, each residual block consists of a single $4d$-hidden dimension MLP layer with ReLU activation.

\looseness-1{\bf Empirical Observations.} 
We are interested in measuring how well each method captures the higher degree terms. To this end, we estimate the component of a learned function $F$ onto each basis polynomials. This can be done simply using
$\Hat{c}_{l, j} = \mathbb{E}_{\vx \sim \{\pm 1\}^{\dims}} \left[F(\vx) \prod_{i \in S_{l, j}} x_{i}\right],$
since basis polynomials are orthogonal under the uniform distribution of the boolean data \citep{o2014analysis}.
For each degree $l \le k$, we define the error as 
$\tfrac{\sum_{j=1}^{m} (c_{l, j} - \Hat{c}_{l, j})^2} { \sum_{j=1}^{m} c_{l, j}^2 }$.
In Figure~\ref{fig:raptr_v_layerdrop_toy}, we observe the following:
\begin{itemize}[leftmargin=*]
 \setlength{\itemsep}{1pt} %
  \setlength{\parskip}{1pt} %
    \item \progressivelayerdrop{}, that drops {\em more} layers as training proceeds, fails to effectively learn  higher order complex correlations and, ultimately, hurts performance.
    \item \method{} is competitive to baseline in terms of evaluation loss; in fact, \method{} learns \emph{every component} effectively.
Furthermore, \method{} quickly learns lower order terms and then picks up higher order terms during later stages.
\end{itemize}

The learning pattern of \method{} is reminiscent of the implicit bias of standard SGD training to learn functions of increasing complexity \citep{kalimeris2019sgd, xu2019frequency, rahaman2019spectral, cao2019towards, abbe2022merged}.
\method{} naturally imposes this simple-to-complex behavior through size of subnetworks and also provides training efficiency in the process. 

{\bf Theoretical Analysis.} To further illustrate our point, we characterize the behavior of \method{} for a simple 2-layer residual network, where each block consists of single neuron with a non-linear sine activation.  We consider simple label function $F^{\star}(\vx)=\frac{\sqrt{3}}{2} + \frac{\sqrt{3}}{2} x_1 - x_1 x_2$ for this analysis. %
\begin{lemma}[Informal, cf \cref{thm:raptr_stagewise}]
    For a small enough learning rate, 2-phase \method{} first learns lower degree component and then the higher degree component. 
\end{lemma}
\looseness-1In contrast, one layer is not expressive enough to represent the true label function (\ref{thm:layerdrop_expressivity}). Thus, progressively dropping more layers (e.g., \progressivelayerdrop{}) will  reduce its expressivity. \textbf{Remark:} Stacking \citep{reddi2023efficient} approaches, that grow the model size, can be shown to have similar simple-to-complex inductive bias in the above setting. We leave a detailed study of this connection to future work.

\vspace{-0.5em}
\section{Experiments}
\vspace{-0.5em}
\label{sec:experiments}
\begin{table}[t]
\centering
\small
\caption{\looseness-1Loss and fine-tuning results for BERT-Base after $675$K steps and using equal peak learning rate ($10^{-4}$). Key observations: (a) {\method} achieves similar loss as baseline at $1.33\times$ reduced FLOPs. (b) {\method} achieves $0.02$ better evaluation loss than gradual stacking for same schedules. (c) \method{} has a slight edge on
downstream fine-tuning task.}
\label{tab:bert_12}
\scalebox{0.9}{
\begin{tabular}{@{}cccc|cccc@{}} 
\toprule
& & Rel. & Eval  &  &  & &  \\
Method & Schedule & FLOPs & Loss & MNLI & QNLI & SST-2 & Avg. \\
\toprule
Baseline & - & 1.33 & 1.76 & 81.5 & \textbf{90.5} & 91.4 & 87.8\\
\progressivelayerdrop{} & - & 1. & 1.8 & 81.7 & 89.0 & 90.8 & 87.2\\
\midrule

\multirow{ 2}{*}{Stacking} & 6-8-10-12 & 1 & 1.78 & - & - & - & -\\
 & 6-9-12 & 1 & 1.77 & 80.9 & 89.8 & 91.1 & 87.3 \\
\midrule
\multirow{ 2}{*}{\method{}} & 6-8-10-12 & 1 & \textbf{1.75} & \textbf{82.1} & 89.8 & \textbf{92.4} & \textbf{88.1} \\
& 6-9-12 & 1 & 1.75 & 82.3 & 89.2 & 91.0 & 87.5 \\
\bottomrule
\end{tabular}
}
\vspace{-0.1in}
\end{table}

\looseness-1We present comprehensive experiments that compare $\method$ with baseline (full-model) training, gradual stacking and progressive layer dropping (\progressivelayerdrop{}) for BERT and UL2. Since our primary goal was to make subnetwork based strategies competitive to stacking, we primarily focus on comparisons to gradual stacking. However, we also show comparisons on BERT with improved versions of gradual stacking, used in MSG \citep{yao2023masked}, LEMON \citep{wang2023lemon}, and bert2BERT \citep{chen2021bert2bert}. The details of relation between gradual stacking and these methods are in \textsection{\ref{sec:stacking_variants_app}}.

\textbf{\method{} setup:} We describe the experimental setting of \method{} from \textsection{\ref{sec:raptr}}. We keep $\fixed$ equal to the first and last layer ($\fixed=\{1, L\}$) across all stages\footnote{Please see Appendix Table \ref{tab:bert_different_fixed} and the discussion later  for the relevant ablation study}. Additionally, scaling layer outputs during initial stages helps track validation loss better, though its final effect is minimal (details in \textsection{\ref{sec:scaling}}). Finally, we use a constant learning rate in all phases, except in the final phase where we decay the learning rate, following gradual stacking \citep{reddi2023efficient}.

\looseness-1{\bf Notations for {\method} and stacking schedules:} We succinctly denote the stage schedules by a set of \layerdrop{} lengths separated by hyphens. For stacking ({\method}) training for a $24$-layer model, 3-6-9-12 refers to $4$ stages with $3, 6, 9, $ and $12$ layer (average subnetwork) training respectively. 
The length of each stage is chosen based on the reduction in FLOPs
to achieve during training (see \textsection{\ref{sec:schedule_selection}}). Unless specified otherwise, we split the training steps equally across the stages.

\looseness-1{\bf Wall clock speedups:} In 
In \textsection{\ref{sec:speedups}}, we report wall-clock times for UL2 and BERT experiments. For UL2, a $1.2\times$ FLOPs reduction leads to a $1.19\times$ speedup, closely matching FLOPs improvements. For BERT, a $1.33\times$ FLOPs reduction gives a $1.26\times$ speedup. We also introduce a new implementation strategy for \method{} in \textsection{\ref{sec:efficient_raptr_impl}}, essential for achieving speedups in distributed training where naive conditional dropping fails or doesn't give any speedup. As runtime depends on architecture and hardware, we focus on FLOPs reductions in the following discussion.

\begin{table}[t]
\centering
\small
\caption{We compare the performance of {\method} against depth-wise stacking operations used in bert2BERT, LEMON, and MSG, which implement slight variations of gradual stacking (see \cref{sec:stacking_variants_app}) after $300K$ steps of training. {\method} demonstrates competitive or superior performance across all variants, achieving comparable loss to the best-performing version of gradual stacking. }
\label{tab:raptr_v_stacking}
\scalebox{0.9}{
\begin{tabular}{lcc|cccc}
\toprule
 & Schedule & Eval Loss  & SST-2 & QNLI & MNLI & Average \\
\midrule
\midrule
bert2BERT & 3-6-12 & 1.85 & 89.3 {\tiny(0.1)} & 84.5 {\tiny(0.1)} & 76.9 {\tiny(0.0)} & 83.6 \\
MSG & 3-6-12 & 1.87 & 88.2 {\tiny(0.2)} & 85.1 {\tiny(0.1)} & 76.7 {\tiny(0.1)} & 83.3 \\
\midrule
MSG & 6-9-12 & 1.84 & 89.7 {\tiny(0.3)} & 85.6 {\tiny(0.2)} & \textbf{77.7} {\tiny(0.0)} & 84.3 \\
LEMON & 6-9-12 & \textbf{1.83} & \textbf{89.8} {\tiny(1.1)} & \textbf{87.2 {\tiny(0.0)}} & 77.3 {\tiny(0.1)} & \textbf{84.7} \\
MSG +  Layer Interpolation & 6-9-12 & 1.84 & 89.0 {\tiny(0.7)} & 85.1 {\tiny(0.0)} & 77.5 {\tiny(0.3)} & 83.9 \\
Stacking & 6-9-12 & 1.84 & \textbf{90.0 {\tiny(0.4)}} & 85.3 {\tiny(0.1)} & 77.0 {\tiny(0.1)} & 84.1 \\
\midrule
{\method} & 6-9-12 & \textbf{1.83} & 89.3 {\tiny(0.6)} & \textbf{85.8} {\tiny(0.3)} & \textbf{78.6 {\tiny(0.2)}} & \textbf{84.6} \\
\bottomrule
\end{tabular}
}
\end{table}

\vspace{-0.25em}
\subsection{Experiments with BERT}
\vspace{-0.25em}
\label{sec:bert}

\begin{table*}[t]
\centering
\small
\caption{Validation loss and 1-shot downstream evals on UL2-1.6B. Key observations: (a) {\method} with 30k initial training improves performance by atleast $2-4\%$ on Trivia QA and SquADv2. {\method} is atleast $5\%$ better than gradual stacking on various downstream tasks, (b) Compared to baseline, {\method} is $1-2\%$ better on all downstream tasks on average. 
(c) Stagewise training  achieves ($1-2\%$) lower variance on Tydi QA and SquADv2, implying implicit benefits of stagewise training.   \textbf{On an extensive suite of $12$ tasks and multi-shot settings, we observe a $1.5\%$ improvement in performance on average.} See \cref{tab:ul2_1b_extendedtests} for  extensive evaluations. %
}
\label{tab:ul2_1b}
\scalebox{0.9}{
\begin{tabular}{lcc|ccccc} \toprule
Method & Rel. FLOPs & Eval Loss & Trivia QA & Tydi QA & SQuADv2 &  SuperGLUE  & Avg. \\ 
\midrule
Baseline  & 1.2 & 2.06 {\tiny (0.01)} & 25.0 {\tiny(0.2)} & 34.4 {\tiny(3.1)} & 42.1 {\tiny(2.9)} & 60.0 {\tiny(0.4)} & 40.4 \\
\progressivelayerdrop{}  & 1 & 2.09 {\tiny (0.00)} & 21.3 {\tiny (0.3)} & 32.4 {\tiny (2.1)} & 40.2 {\tiny (0.9)} & 59.9 {\tiny (0.2)} & 38.5 \\
12-16-20-24 Stacking & 1 & 2.08 {\tiny(0.00)} & 20.1 {\tiny(1.3)} &  28.6 {\tiny(2.4)} & 36.0 {\tiny(1.9)} & 60.4 {\tiny(0.9)} & 36.3 \\
\midrule
12-16-20-24 {\method} & 1 & \textbf{2.06} {\tiny (0.00)} & \textbf{25.8} {\tiny(0.2)} & \textbf{36.7} {\tiny(1.0)} & \textbf{44.1} {\tiny(0.5)} & \textbf{60.9} {\tiny(0.2)} & \textbf{41.9} \\
\bottomrule
\end{tabular}
}
\vspace{-0.1in}
\end{table*}

\begin{table}[htbp]
\small
\centering
\caption{\rebuttal{We extensively compare all trained models from \cref{tab:ul2_1b} on multiple downstream tasks and few-shot in-context settings. We follow prompt design from \cite{chowdhery2022palm} in each setting. For tasks marked with $^{\dagger}$ and $^{\star}$, we report Exact match and F1 scores respectively. For the rest, we use accuracy. The tasks have been sub-divided into 4 major groups, memorization QA, QA with context, completion, and SuperGLUE. On average, {\method} is 1-2\% better than  baseline and stacking.
}  }
\label{tab:ul2_1b_extendedtests}
\begin{tabular}{|l|ccc|ccc|} \toprule
& \multicolumn{3}{|c|}{1-shot}  & \multicolumn{3}{|c|}{5-shot} \\
\midrule
& Baseline &  Stacking &  {\method}  & Baseline & Stacking &  {\method}  \\
\midrule
Trivia QA$^{\dagger}$ & 25.0 {\tiny (0.2)}  & 20.1 {\tiny (1.3)} & 25.8 {\tiny (0.2)} & 26.5 {\tiny (1.1)}  & 21.1 {\tiny (1.3)} & 25.1  {\tiny (0.5)} \\
Web QA$^{\dagger}$ & 6.4 {\tiny (0.4)}  & 5.8 {\tiny (0.6)} &  7.6 {\tiny (0.5)} & 10.6 {\tiny (0.4)} & 9.2 {\tiny (0.5)} &  11.2 {\tiny (0.2)} \\
Natural QA$^{\dagger}$ & 4.2 {\tiny (0.5)}    & 3.4 {\tiny (0.4)} & 4.4 {\tiny (0.1)} & 5.7 {\tiny (0.1)}  & 4.6 {\tiny (0.3)} & 6.0 {\tiny (0.3)} \\
\midrule
Tydi QA$^{\dagger}$ & 34.4 {\tiny (3.1)}  &  28.6 {\tiny (2.7)} & 36.7 {\tiny (1.0)} & - & - & - \\
SQuaDv2$^{\dagger}$ & 42.1 {\tiny (2.0)}  &  36.0 {\tiny (0.9)} & 44.1 {\tiny (0.5)} & 43.2 {\tiny (3.0)}  & 36.2 {\tiny (2.6)}  & 44.5 {\tiny (1.3)} \\
DROP$^{\dagger}$ & 21.4 {\tiny (0.8)}  & 19.5 {\tiny (0.6)} & 23.0 {\tiny (0.4)} & - & - & -  \\
CoQA$^{\star}$ & 49.2 {\tiny (0.7)} &  43.9 {\tiny (0.8)} & 52.4 {\tiny (0.7)} & - & - & -  \\
QuAC$^{\star}$ & 18.1 {\tiny (0.5)}  & 16.8 {\tiny (0.6)} & 18.1 {\tiny (0.4)} & - & - & -  \\
\midrule
LAMBADA & 13.7 {\tiny (2.9)}  & 12.0 {\tiny (1.1)} & 18.7 {\tiny (3.1)} & 30.5 {\tiny (3.4)}  & 29.5 {\tiny (2.1)} &  38.7 {\tiny (2.2)} \\
Storycloze & 72.9 {\tiny (0.4)}  & 71.0 {\tiny (0.4)} & 73.3 {\tiny (0.4)} & 75.1 {\tiny (0.1)}  & 72.6 {\tiny (0.8)} & 75.3 {\tiny (0.6)} \\
Hellaswag & 58.3 {\tiny (0.2)}   & 56.1 {\tiny (0.1)} & 58.5 {\tiny (0.3)} & 58.3 {\tiny (0.2)}   & 56.1 {\tiny (0.1)} & 58.4 {\tiny (0.3)} \\
\midrule
SuperGLUE & 60.0 {\tiny (0.4)}  & 60.4 {\tiny (0.9)} & 60.9 {\tiny (0.2)} & 60.7 {\tiny (0.3)}   & 58.8 {\tiny (0.5)} & 62.1 {\tiny (0.2)} \\
\midrule
\midrule
Average & 33.8 & 31.1 & 35.3 & 38.8 &  36.0 &  40.2 \\
\bottomrule
\end{tabular}
\end{table}

{\bf Experiment setup.}  We pretrain BERT models~\citep{devlin2018bert} on Wikipedia + Books dataset with AdamW optimizer \citep{loshchilov2018decoupled}.  We report the evaluation loss and fine-tuning performance of the trained models on $3$ GLUE tasks~\citep{wang2018glue} (refer to \textsection{\ref{sec:add_details}} for details). 

{\bf Comparisons to gradual stacking and baseline:} We observe that {\method} achieves similar or better evaluation loss to baseline training, despite baseline training using $1.33\times$ more FLOPs than {\method} for \bertbase{}. Compared to gradual stacking at similar FLOPs, {\method} again has better validation loss. Additionally, all three methods exhibit similar performance in downstream fine-tuning tasks
(see \Cref{tab:bert_12}). However, \progressivelayerdrop{} performs significantly worse compared to others.

\textbf{Comparisons with variants of stacking:} 
We further compare {\method} to variants of gradual stacking used in MSG \citep{yao2023masked}, LEMON \citep{wang2023lemon}, and bert2BERT \citep{chen2021bert2bert} (details are in \cref{sec:stacking_variants_app}). We focus on depth-only stacking operations in these frameworks in \cref{tab:raptr_v_stacking}. We show that {\method} performs competitive or even better than all variants of stacking, achieving same evaluation loss as the best tuned variant of stacking.

\textbf{Loss behavior during phase transitions:} Perhaps surprisingly, we observe that  \textbf{the training loss decreases quite smoothly} when transitioning between two stages in \method{} (see Fig. \ref{fig:bert24_behavior} (a) for BERT training under \method{}). This is in contrast to stacking, where new variants have been developed to explicitly impose loss preservation across stage transitions \citep{shen2022staged,wang2023lemon}. We delve deeper into this favorable behavior of \method{} in \textsection{\ref{sec:theory}}.

\textbf{Ablations with \method{} parameters:} 
We observe robustness of the final trained model to different \method{} schedules (Appendix Table \ref{tab:bert_different_schedules}). Furthermore, we observe that fixing the first and last layers at all steps during training helps \method{} (Appendix Table \ref{tab:bert_different_fixed}). Since \method{} is robust to schedule selections (\Cref{tab:bert_different_schedules}), we recommend the schedule in the box below for all practical purposes.
\begin{figure*}[!htbp] %
    \centering
    \fbox{ %
        \begin{minipage}{0.95\textwidth} %
            \centering
            \vspace{2pt}
            \textbf{Recommended \method{} schedule selection} \\ %
            \vspace{4pt} %
Schedule contains $4$ phases. Initially, random subnetworks of size $L/2$  are selected, and we increase by $L/6$ layers in each transition until the full model is trained in the final phase.
\vspace{2pt}
        \end{minipage}
    }
\end{figure*}

\vspace{-0.25em}
\subsection{Experiments with UL2-1.6B} \label{sec:ul2_1b}
\vspace{-0.25em}

Experiments on BERT (\textsection{\ref{sec:bert}}) show that \method{} can be competitive or better than gradual stacking and can outperform baseline training. We ask whether the same observations hold when training billion-parameter language models on large text corpora. We pretrain a 1.6B decoder-only UL2 model~\citep{tay2022ul2} with 24 layers for $400k$ steps with batch size $512$.

\Cref{tab:ul2_1b} reports the validation loss and downstream 1-shot performance without fine-tuning. Please see \textsection{\ref{sec:ul2_expdetails}} for experimental details and \cref{tab:ul2_1b_extendedtests} for more extensive evaluations.

{\bf Schedules for {\method} and gradual stacking.}\label{sec:schedules} In \cref{tab:ul2_1b}, we report for a schedule with $4$ stages, denoted by 12-16-20-24. The length of each stage has been adjusted appropriately to use an average \layerdrop{} length of $20$ out of $24$, giving $1.2\times$ FLOP reductions compared to the baseline training (see \textsection{\ref{sec:schedule_selection}}). \method{} uses an initial full-model training phase for $30$K steps before transitioning to its schedule. For parity in FLOPs, we reduce $30$K steps from the final phase. This is useful for \method{}--we attribute this to the fast drop in loss that large models can achieve initially. The ability to include full model training is another benefit of \method{} over stacking due to its flexibility.

We summarize the key findings from \cref{tab:ul2_1b} below.
\begin{itemize}
 \setlength{\itemsep}{1pt} %
  \setlength{\parskip}{1pt} %
    \item {\bf Perplexity results.} At $1.2\times$ FLOP reduction, \method{} achieves similar evaluation perplexities to baseline training and also performs better than stacking (\cref{tab:ul2_1b}).
    \item {\bf Downstream inductive bias.} Intriguingly, we find that despite having similar perplexity, {\method} has much better downstream metrics compared to both baseline and stacking.
    The improvements are particularly large for TydiQA (3.8\%) and SQuADv2 ($2.0\%$). This perhaps suggests {\method} has a desirable inductive bias of \method, a phenomenon that has been recently studied for other pretraining settings \citep{saunshi22understanding,liu2023same}.
    \item {\bf Lower variance.} Another notable observation is that {\method} has lower variance on all metrics and is thus more stable to training runs. We believe these inductive bias benefits of {\method} deserve further exploration.
\end{itemize}

\subsection{Subnetwork training beyond depth}
\label{sec:beyond-depth}
Stacking based approaches have been studied for other dimensions beyond depth (e.g. width) by \citep{gu2020transformer,shen2022staged,wang2023learning,gesmundo2023composable}. 
In this section we demonstrate the flexibility of the subnetwork training framework to handle other growth operators.

We extend {\method} to intermediate MLP dimensions and attention heads. For attention heads, we apply stagewise training by randomly selecting a subset of heads in each layer, increasing the number of heads across stages. For MLP dimensions, we divide each MLP into $4$ sub-MLPs, pick random subset of sub-MLPs in each layer to define a random subnetwork, and gradually increase the number selected at each stage. We compare these variants, named Attention-{\method} and MLP-{\method}, to their MSG counterparts.
Attention-{\method} and MLP-{\method} perform better than their corresponding operators in MSG, providing further evidence of better performance than stacking-based approaches (\cref{tab:comp_to_stacking_variants}). A more comprehensive study of {\method} to other axes is left as future work.

\begin{table}[!t]
\centering
\small
\caption{Comparisons at $300K$ steps of training on \bertbase{}. \textbf{(Top)}: Comparison between Attention-RaPTr and MSG's attention head growth operator \citep{yao2023masked}. Both use schedules that use an average of 9 heads per layer, reducing FLOPs in self-attention by $1.33\times$.
\textbf{(Bottom)}: Comparison between MLP-RaPTr and MSG's MLP growth operator. Both use schedules that use an average intermediate MLP dimension of 2048, reducing MLP FLOPs by $1.7\times$.}
\label{tab:comp_to_stacking_variants}
\scalebox{0.9}{
\begin{tabular}{l|cc|cccc}
\toprule
 & & Eval Loss & SST-2 & QNLI & MNLI & Avg. \\
\midrule
Head & MSG & 2.07 & 87.5 {\tiny(0.1)} & 83.4 {\tiny(0.1)} & 73.3 {\tiny(0.3)} & 81.4 \\
Operators (6-9-12) & Attention-RaPTr & \textbf{1.84} & \textbf{89.8 {\tiny(0.6)}} & \textbf{86.5 {\tiny(0.3)}} & \textbf{77.5 {\tiny(0.1)}} & \textbf{84.6} \\
\midrule
\midrule
Intermediate MLP size & MSG & 1.88 & 88.1 {\tiny(0.4)} & \textbf{85.6 {\tiny(0.4)}} & \textbf{77.4 {\tiny(0.1)}} & 83.7 \\
Operators (768-1536-3072) & MLP-RaPTr & \textbf{1.84} & \textbf{88.7 {\tiny(0.3)}} & 85.3 {\tiny(0.1)} & 77.3 {\tiny(0.3)} & \textbf{83.8} \\
\bottomrule
\end{tabular}
}
\end{table}

\section{Loss Preserving Behavior of \method{}} \label{sec:theory}
\vspace{-0.5em}

\looseness-1
\method{} provides a general framework, as we can adjust the average size of random subnetworks in discrete steps between stages (\textsection{\ref{sec:raptr}}). Training time instabilities are a possibility, as one might expect an arbitrarily large spike in the loss at stage transitions, due to the discrete shift in the training of the model. Remarkably, we observe that \textbf{the training loss decreases quite smoothly} when transitioning between two stages (see Fig. \ref{fig:bert24_behavior} (a) for BERT training under \method{}).
Hence, we ask the following:

\begin{quote}
    \emph{What are the general conditions under which training loss under \method{} stays stable at stage transitions?}
\end{quote}
This question has been been extensively studied for stacking based approaches \citep{wang2023lemon,chen2021bert2bert,chen2015net2net,shen2022staged}.
Here, we introduce, to our knowledge, the \emph{first} framework to analyze models under discrete dropping based training strategies, and understand the necessity of different architecture components for robustness to discrete transitions. Related works look into the generalization theory for dropout \citep{srivastava2014dropout} as regularization~\citep{gal2016dropout,arora2021dropout,wager2013dropout,hu2019simple,helmbold2015inductive}. Building on these frameworks, our analysis relies on a crucial notion of stability \citep{arora2018stronger}. The stability analysis, when instantiated for linear networks highlights the importance of two popular architectural choices in Transformer models -- residual connections and layer normalization.

\vspace{-0.5em}
\subsection{Layer stability}
\vspace{-0.5em}

\looseness-1We focus on sequences of length $\seqlen = 1$; generalization to $\seqlen > 1$ is fairly easy. Let $\netloss: \mathbb{R}^{\dims} \to \RR$ represent the training loss of the backbone $F$. For transformers, $F$'s output passes through a Layer normalization layer before projecting to logits. This makes $\netloss$ scale-invariant~\citep{ioffe2015batch,ba2016layer}, which gives a favourable stability condition as any perturbation in $F$'s output results in a proportional loss perturbation relative to the norm of $F$'s output. Thus, we make the following assumption.
\begin{assumption}[Relative-Lipschitzness] \label{ass:scale_invariance}
    There exists a constant $\liploss > 0$ such that for any input $\vx \in \RR^\dims$, label $\vy$ and perturbation $\eta \in \RR^{\dims}$, $\netloss$ satisfies
     $\netloss(\vx + \eta, \vy) - \netloss(\vx, \vy) \le \liploss \left(\norm{\eta}_2 / \norm{\vx}_2\right).$
\end{assumption}

To understand loss changes across stage boundaries of \method{}, \rebuttal{we consider two stage {\method}: the first stage trains with {\layerdrop} of length $L-1$ by dropping a single layer at random, and the second stage trains the entire model.} Suppose $\netloss_{1}$ and $\netloss_{2}$ denote the effective loss functions being minimized in the two stages. 
The aim is to show that $F$ learned in the first stage  by minimizing $\netloss_{1}$ should also result in small value for $\netloss_{2}$. 

\begin{definition}\label{eq:lipstacklayers}
    \looseness-1Let $F_{-\ell}$ denote the subnetwork after skipping layer $\ell$.
    The stability of network output with respect to dropping layer $\ell$ is defined as $\Psi_{\ell}(\vx) = \|F_{-\ell}(\vx) - F(\vx)\|_{2}$.
\end{definition}
Next theorem shows that $\netloss_{2}(F) - \netloss_{1}(F)$ is small if stability scales slower relative to output norm.

\begin{theorem}[Informal, cf \cref{thm:err_prop}]\label{thm:err_prop_inf}
Under assumption \ref{ass:scale_invariance} of loss $\netloss$,
$|\netloss_2(F) - \netloss_1(F)|$ is upper bounded by $(\liploss/L) \expect_{\vx \in \distrib}  (\sum_{\ell=1}^{\layer} \Psi_{\ell}(\vx)) / \norm{F(\vx)}_2$.
\end{theorem}

The proof follows by observing that $\netloss_{1}(F) = \frac{1}{\layer} \sum_{\ell\in[\layer]}\netloss(F_{-\ell})$ and $\netloss_{2}(F) = \netloss(F)$. Thus, the losses are close if $F_{-\ell}$ is close enough to $F$ for a random $\ell$. 
So, the relative stability of the network determines the success of the approach. 

\looseness-1{\bf Empirical verification with BERT:} In \cref{fig:bert24_behavior} (b, c, d), we report the stability of \bertlarge{} during training. We observe that the norm of the layers grow linearly with depth, while the stability $\Psi_{\ell}$ grows sub-linearly with depth. This implies that our upper bound on the loss gap $\netloss_{2}(F) - \netloss_{1}(F)$ will show a rough linear decrease with the depth of the model.

\begin{figure*}[!htbp] %
    \centering
    \fbox{ %
        \begin{minipage}{0.95\textwidth} %
            \vspace{2pt}
            We find that $\netloss_{2}(F)$ is lower than $\netloss_{1}(F)$, even when training with subnetworks of length $L-1$. This shows that {\method} can reduce training losses at stage transitions without requiring explicit patterns, unlike stacking approaches that need carefully designed growth functions to avoid loss spikes \citep{wang2023lemon,shen2022staged}. 
            Moreover, this loss reduction surpasses the concept of loss preservation \citep{chen2021bert2bert}, which opens up interesting future directions.
        \end{minipage}
    }
\end{figure*}

 Instead, we focus on the following  question, "when is a network stable (by def. \ref{eq:lipstacklayers}) across the stages of \method{}?" For simplicity, we consider the special case of linear residual networks. We show that layer normalization layers and the residual connection helps maintain the stability of \method{}.

\looseness-1{\bf Illustrative example: linear networks}
We present a more concrete instantiation of Theorem~\ref{thm:err_prop_inf} for residual network where the layers $f_{1:\layer}$ are linear with parameters $\mW_{1:\layer}$. %
The layer output is $\vy^{\ell} = \vy^{(\ell-1)} + \mW_{\ell} \vy^{(\ell-1)}$ or $\vy^{\ell} = \vy^{(\ell-1)} + \mW_{\ell} \vy^{(\ell-1)} / \norm{\vy^{(\ell-1)}}_{\rebuttal{2}}$ depending on whether layernorm %
is enabled or not respectively . 
We also study another setting with layernorm but no residual connection; so $\vy^{\ell} = \mW_{\ell} \vy^{(\ell-1)} / \norm{\vy^{(\ell-1)}}_{\rebuttal{2}}$. 

\begin{lemma}\label{lem:random}
    For $\delta > 0$, a finite input set $S$ and dimension $\dims \ge \Omega(|S|L \log(1/\delta))$, w.p. atleast $1-\delta$ w.r.t. random initialization $\mW_\ell \sim \mathcal{N}(0, \dims^{-1/2} \mI)$, the following holds true for all input $\vx \in S$.
    \begin{enumerate}[label=(\alph*),leftmargin=*]
  \setlength{\itemsep}{1pt} %
  \setlength{\parskip}{1pt} %
        \item With residual connection \& layernorm, $\Psi_{\ell} (\vx) = \mathcal{O}( \sqrt{ \layer / \ell} )$ \& $\norm{F(\vx)}_{\rebuttal{2}} = \Omega(\sqrt{\layer})$. Then the gap in losses between stages scales as $\mathcal{O}(1/\sqrt{\layer})$.
        \item Without residual connection, $\Psi_{\ell} (\vx) = \Omega(1)$ \& $\norm{F(\vx)}_{\rebuttal{2}} = \mathcal{O}(1)$. 
        Thus the gap in losses between stages can be $\Omega(1)$. 
        
        \item Without layernorm, we have $\Psi_{\ell} (\vx) = \Omega(2^{(\layer-1)/2})$ and $\norm{F(\vx)}_{\rebuttal{2}} = \mathcal{O}(2^{\layer/2})$ . 
        Thus the gap in losses between stages can be $\Omega(1)$. 
    \end{enumerate}
   \end{lemma}

In the Appendix, we show similar results for perfectly aligned layers (Lemma~\ref{lem:identical}). We can consider even more general scenarios, where the layers parameters are expressed as $\tau$-combinations of a Gaussian and a shared matrix (Appendix Fig. \ref{fig:linear_behavior}). 
We run simulations and observe that for each $\tau$, the loss gap between a $\layer-1$ random {\layerdrop} and the full model scales as $\mathcal{O}(\layer^{-0.4})$.

\begin{figure*}[!htbp]%
    \centering
    \begin{subfigure}[b]{0.24\linewidth}
    \centering
    \includegraphics[width=\textwidth]{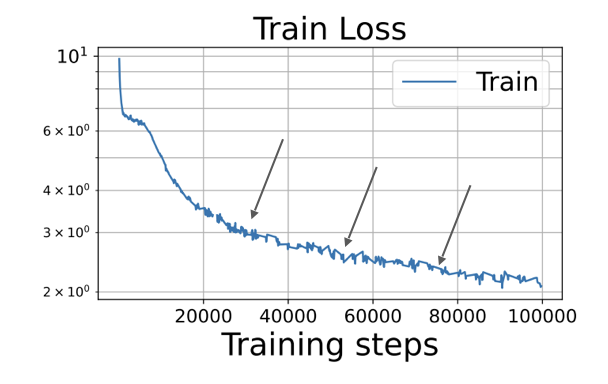}
    \label{fig:bert_train}
    \end{subfigure}\hfill
    \begin{subfigure}[b]{0.24\linewidth}
    \centering
    \includegraphics[width=\textwidth]{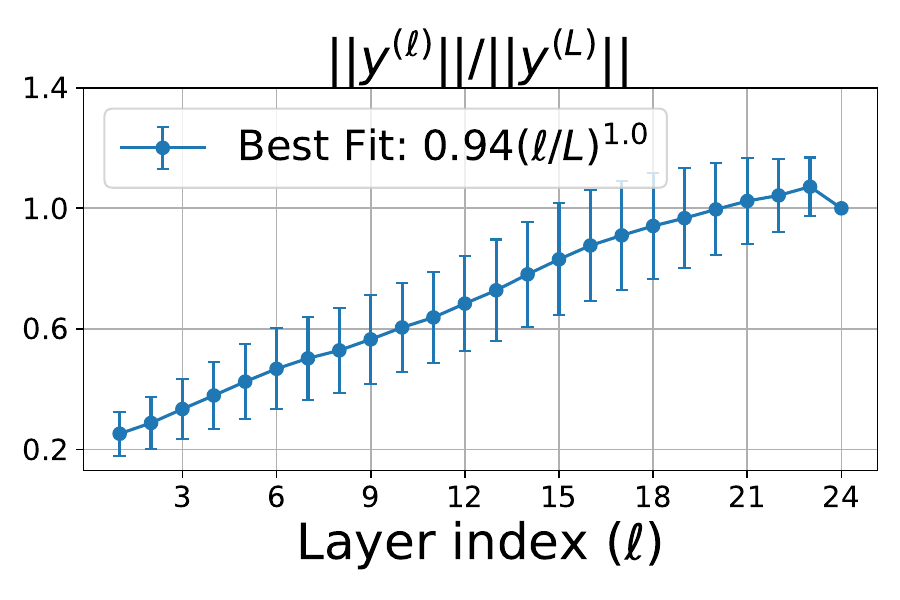}
    \label{fig:24bert_n}
    \end{subfigure}\hfill
    \begin{subfigure}[b]{0.24\linewidth}
    \centering
    \includegraphics[width=\textwidth]{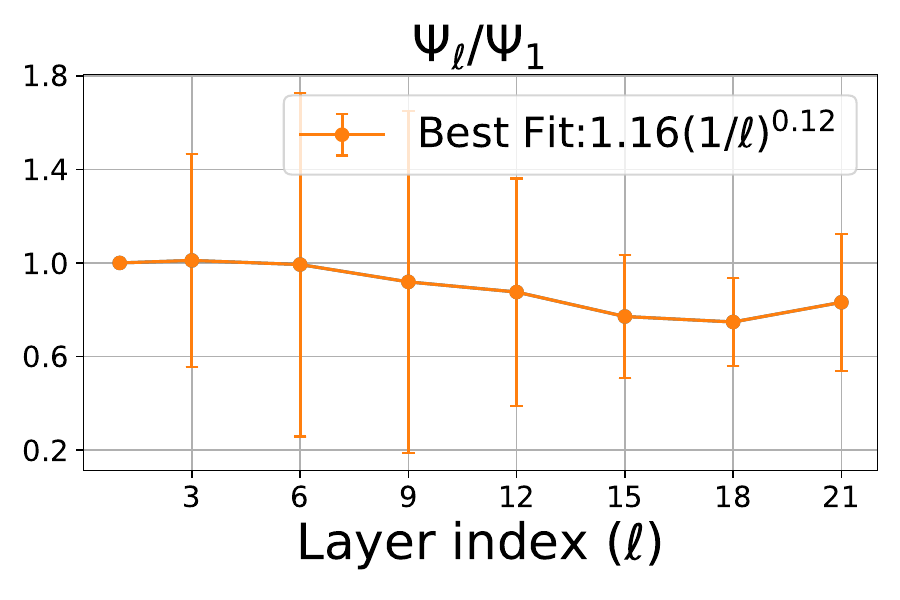}
    \label{fig:24bert_p}
    \end{subfigure}\hfill
    \begin{subfigure}[b]{0.24\linewidth}
    \centering
    \includegraphics[width=\textwidth]{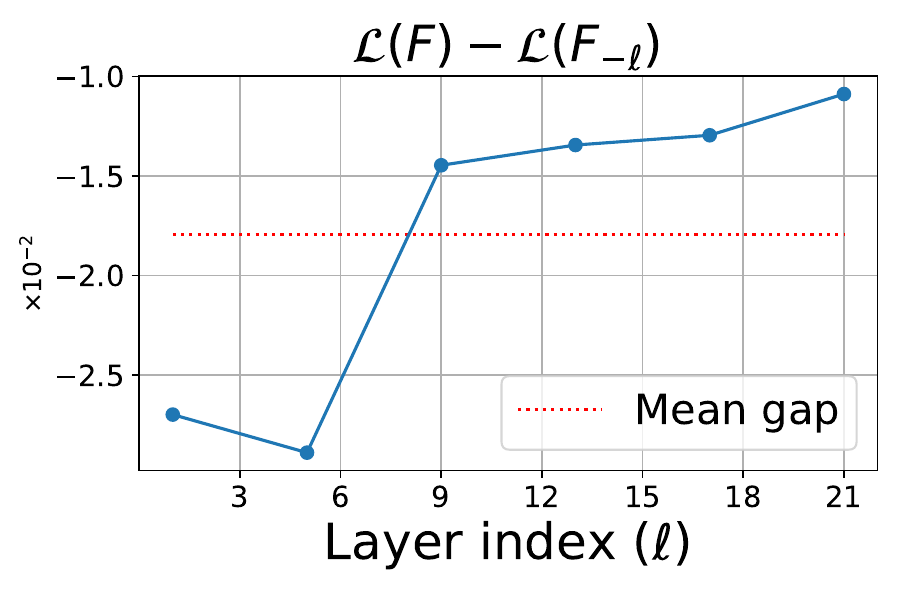}
    \label{fig:24bert_l}
    \end{subfigure}
    \caption{\looseness-1 
    (a): Training trajectory of BERT under $4$-stage 6-8-10-12 \method{}, with the stage transitions (denoted by arrows);
    (b), (c), (d): Stability study on \bertlarge{} trained for $50k$ steps by {\method} with subnetworks of length $L-1$. Behavior of (b) norms of intermediate activations $\vy^{(\ell)}$, (c) $\Psi_{\ell} / \Psi_{1}$ (\cref{eq:lipstacklayers}), and (d) Loss gap between different random {\layerdrop} $F_{-\ell}$ and model $F$, given by $\mathcal{L}(F) - \frac{1}{\layer} \sum_{\ell=1}^{\layer} \mathcal{L}(F_{-\ell})$. Key observations: (b) Norms of the intermediate activations grow linearly with $\ell$, (c) $\Psi_{\ell}$ \rebuttal{changes} slowly with $\ell$ as $(\frac{\layer}{\ell})^{0.12}$, 
    suggesting a worse-case bound of $\mathcal{O}(\layer^{-0.88})$  on $\mathcal{L}(F) - \frac{1}{\layer} \sum_{\ell=1}^{\layer} \mathcal{L}(F_{-\ell})$ based on \Cref{thm:err_prop_inf} (d) Interestingly, $\mathcal{L}(F) \le \frac{1}{\layer} \sum_{\ell=1}^{\layer} \mathcal{L}(F_{-\ell})$, even when model is trained with $L-1$ random subnetworks.}
    \label{fig:bert24_behavior}

    \vspace{-0.1in}
\end{figure*}

\vspace{-0.5em}
\section{Related works}
\label{sec:related}
\vspace{-0.5em}
Here, we focus on the most relevant works. Refer to \cref{sec:add_related_works_App} for more related works.

\looseness-1{\bf Stochastic depth.}  
\method{} is similar to stochastic depth, which drops layers with a {\em fixed probability} during training to reduce the cost of training deep networks \citep{huang2016deep}. A key distinction is that the probability of dropping layers is fixed during training and can thus be viewed as a regularization \citep{pham2019very,steiner2021train,tolstikhin2021mlp,liu2023dropout}. Fixed stochastic depth has also been used for inference efficiency \citep{fan2019reducing}. \citep{devvrit2023matformer,valipour2023sortednet} explore nested training for dynamic inference with multiple subnetworks.

\looseness-1{\bf Subnetwork training.} %
Training random paths and subnetworks has been used in other contexts, such as parameter-efficient fine-tuning \citep{houlsby2019parameter,pfeiffer-etal-2021-adapterfusion,hu2022lora,liu2022few} to reduce memory footprint, distributed and federated training \citep{dun2022resist}, and incremental learning \citep{jathushan2019random} to avoid forgetting in continual learning. However, these approaches are not aimed at reducing FLOPs during pretraining.

 \vspace{-0.5em}
\section{Limitations, Conclusion and Future Work}\label{sec:conclusion}
\vspace{-0.5em}

\looseness-1We propose a stagewise training framework of \emph{progressive subnetworks} for efficient pretraining, and evaluate a natural instantiation of this framework ({\method}) based on training random paths/subsets of layers.
Overall {\method} yields better quality language models than baseline training, while showing a reduction of $1.2-1.33 \times$ the total FLOPs.
At the same speedup, it is also better than prior layer dropping and competitive to stacking based approaches. 

The current analysis provides insights into loss stability at stage boundaries but does not explain the observed decrease in loss during transitions. Furthermore, schedule selection is not well understood due to expensive compute necessity and deserves more exploration. Our analysis does not explain the role of initial full-model warmup for UL2 and the desirable inductive biases of {\method} style training, such as better downstream performance and smaller variance. A deeper understanding of these phenomena could lead to the development of more efficient training algorithms.

\newpage

\bibliography{arxiv}
\bibliographystyle{iclr2025_conference}

\newpage
\appendix
\appendix
\tableofcontents

\section{Broader Impact}\label{sec:broader_impact} 

Training LLMs demands extensive computational resources and infrastructure. Our paper proposes an algorithm that aims to accelerate their pre-training. Accelerating the training of large language models (LLMs) can significantly reduce their environmental impact by lowering energy consumption and minimizing carbon footprints. Optimizations in algorithms and infrastructure will lead to more efficient use of computational resources and will promote sustainability of the field of AI.

Furthermore, we observe benefits of structured pre-training (simple-to-complex) of these models on downstream tasks. Hopefully, these ideas can lead to better pre-training strategies for LLMs.

\begin{figure*}[t]%
    
    \centering
    \begin{subfigure}{0.3\textwidth}
    \centering
    \includegraphics[width=\textwidth]{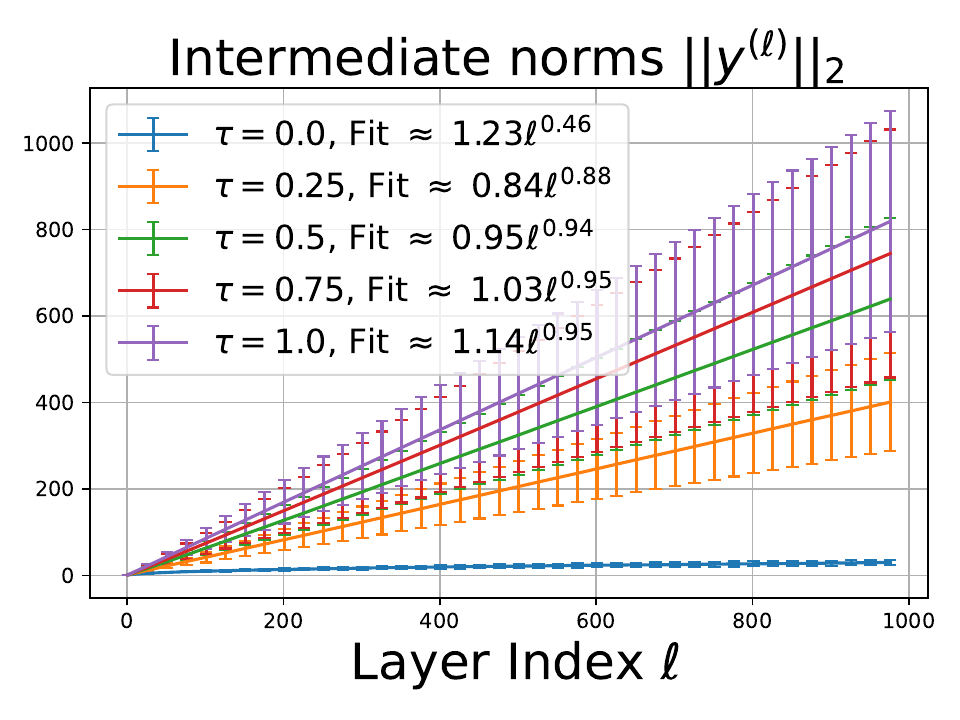}
    \label{fig:cls_vs_prompt_a}
    \end{subfigure}\hfill
    \begin{subfigure}{0.3\textwidth}
    \centering
    \includegraphics[width=\textwidth]{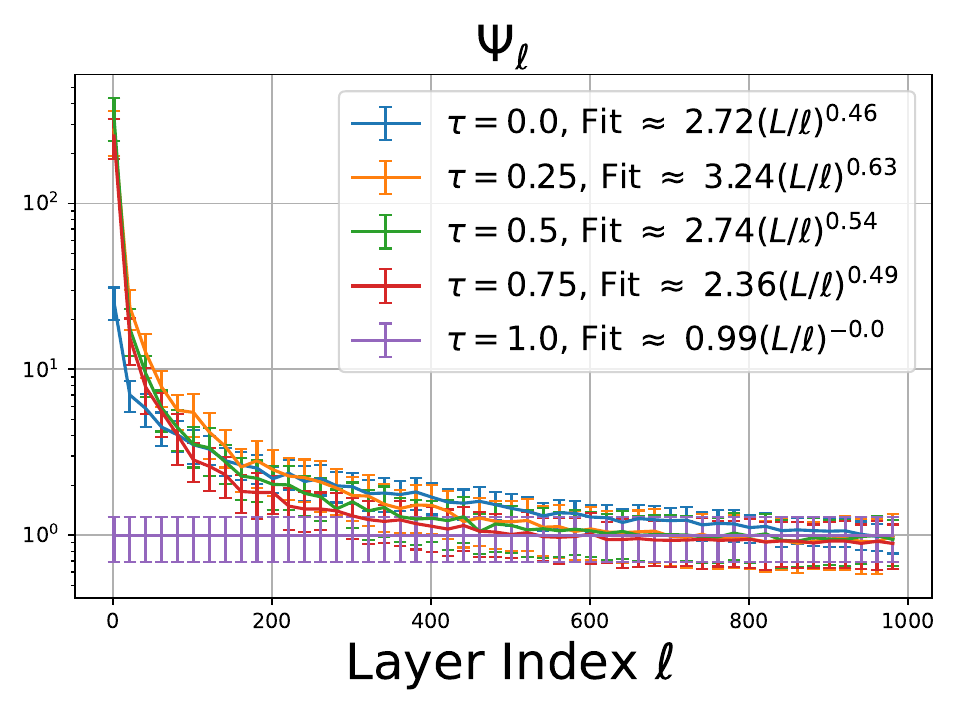}
    \label{fig:cls_vs_prompt_b}
    \end{subfigure}\hfill
    \begin{subfigure}{0.3\textwidth}
    \centering
    \includegraphics[width=\textwidth]{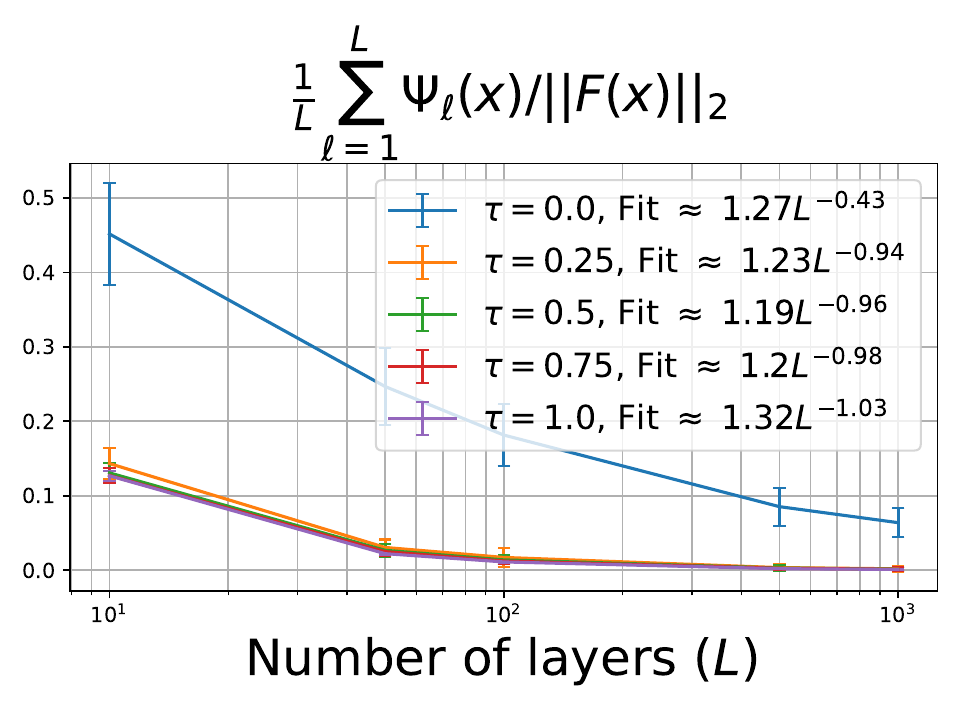}
    \label{fig:cls_vs_prompt_c}
    \end{subfigure}
    \caption{\looseness-1Behavior on a linear residual network with normalization layers with $100$ random samples from $\sphere^{\dims-1}$, and dimension $\dims=100$. The parameters of each layer $\ell$ is represented as $\sqrt{\tau} \mA + \sqrt{1-\tau} \mG^{(\ell)}$ for a shared matrix $\mA \in \RR^{\dims \times \dims}$ with $\norm{\mA}_2 \le 1$ and $\mG^{(\ell)} \sim \mathcal{N}\left(0, \dims^{-1/2}\mI\right)$. Left to right: Behavior of (a)  the norms of intermediate activation $\vy^{(\ell)}$ with index $\ell$, (b)  $\Psi_{\ell}$ (\cref{eq:lipstacklayers}) for each stack of layers $F_{\ell:\layer}$, and (c) $\frac{1}{\layer} \sum_{\ell=1}^{\layer} \Psi_{\ell} (\vx) / \norm{ F(\vx) }_2$ that appears in our bounds in \cref{thm:err_prop_inf}.
    }
    \label{fig:linear_behavior}
\end{figure*}

\section{More extensive downstream tests on UL2 models}
\rebuttal{We conduct extensive downstream evaluation of  the trained models from \cref{tab:ul2_1b} on additional downstream tasks and under higher shot in-context setting. Additional tasks include QA tasks like natural QA \citep{kwiatkowski-etal-2019-natural}, web QA \citep{berant2013semantic}, DROP \citep{dua2019drop}, CoQA \citep{reddy2019coqa}, and QuAC \citep{choi2018quac}; and completion tasks like  Storyclose \citep{mostafazadeh2017lsdsem}, Hellaswag \citep{zellers2019hellaswag}, and LAMBADA \citep{paperno2016lambada}.}

\section{Additional related works}\label{sec:add_related_works_App}

{\bf Residual networks as ensembles.} 
Training with random paths can be weakly viewed as ensembles over paths, and the connection between ResNets \citep{he2016deep} and ensembles of shallow networks was first point out in \citep{veit2016residual} for vision models. \citep{dong2021attention} showed the same phenomenon for self-attention models, where longer paths lead to rank-collapse in the self-attention module. \cite{chang2023revisiting} study vision transformers~\citep{dosovitskiy2020image} as a cascade of multiple paths and propose pruning and self-distillation to remove long paths and improve performance. All these works mainly provide a novel perspective or inference efficiency but do not focus on training efficiency.

{\bf Learnable scaling of residual blocks.} \cite{bachlechner2021rezero}, \cite{zhang2019fixup}, \cite{touvron2021going} consider learnable scales on the output of the residual blocks. These works aim to understand favorable initialization of the learnable scales under various constraints for faster training, which is very different from the role of scaling in our algorithm.

\paragraph{Early exit for efficient inference} A lot of recent works have focused on improving inference efficiency for large language models. ~\citep{lei2023conditional,tay2022ul2,del2023skipdecode,xin2020deebert,zhou2020bert,hou2020dynabert}. However, none of these works focus on efficient fine-tuning.  \cite{lei2023conditional} modified pre-training by substituting resource-intensive MLP computations with straightforward classifiers for a predetermined fraction of tokens within the sequence. It's worth noting that their primary focus was on accelerating inference time rather than enhancing the efficiency of the pre-training process.

\begin{table}[tbp]
\centering
\small
\caption{Additional rows from \cref{tab:ul2_1b}: we compare baseline runs at two learning rate schedules, and run compare {\method} and Gradual stacking with a new 6-12-18-24 schedule. \textbf{Key Observations}: (a) Cosine decay learning works best for the baseline, compared to Square-root schedule proposed by \cite{tay2022ul2}, (b) The performance of gradual stacking improves with the new schedule, implying a performance dependence with appropriate scheduling, and (c) {\method}'s performance doesn't change much with the new schedule.
}
\label{tab:ul2_1b_appendix}
\begin{tabular}{lccccccc} \toprule
& Rel. FLOPs & Eval Loss & Trivia QA & Tydi QA & SQuADv2 &  SGLUE  & Avg.\\ 
\midrule
Baseline ({\tiny Square-root LR decay}) & 1.2 & 2.07 & 23.4 & 31.9 &  44.3 &  60.0 & 39.9 \\
Baseline ({\tiny Cosine LR decay}) & 1.2 & 2.06  & 25.0  & 34.4  & 42.1 & 60.0 & 40.4\\
\midrule
6-12-18-24 Stacking & 1 &  2.06 & 22.1  & 34.6  & 38.0  & 60.5 & 38.8\\
12-16-20-24 Stacking & 1 & 2.08  & 20.1  &  28.6  & 36.0 & 60.4 & 36.3\\
\midrule
12-16-20-24 {\method} & 1 & 2.08  & 22.2  & 38.2 & 40.6 & 60.1 & 40.3\\
(+{\tiny 30k initial full-model train}) & 1 & 2.06  & 25.8 & 36.7  & 44.1  & 60.9 & 41.9\\
6-12-18-24 {\method} & \\
(+{\tiny 30k initial full-model train}) & 1 & 2.06  & 24.2   & 37.3  & 42.3  & 61.1 & 41.2\\
\bottomrule
\end{tabular}
\end{table}

\begin{figure}[!htbp]%
    \centering
    \begin{subfigure}{0.45\textwidth}
    \centering
    \includegraphics[width=\textwidth]{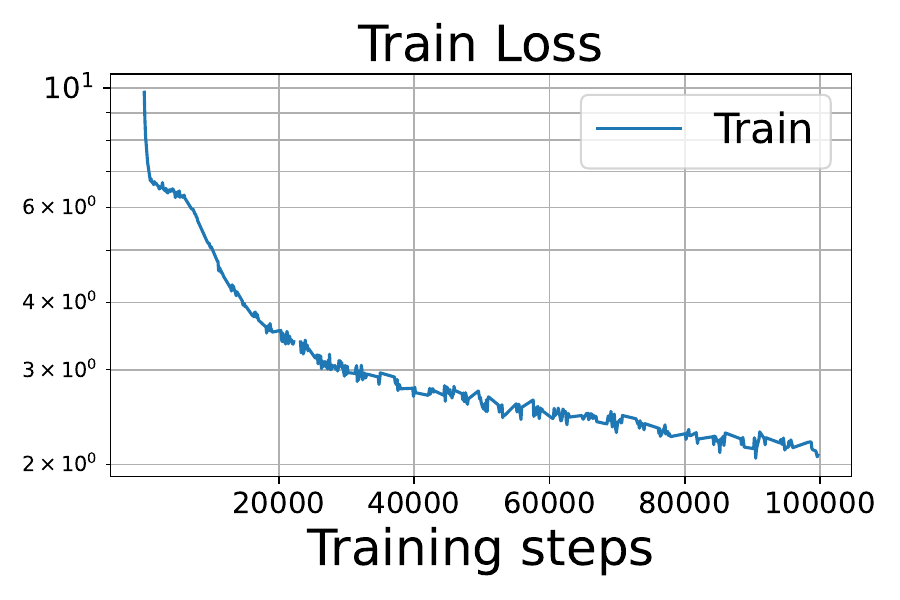}
    \label{fig:train_ablate_app}
    \end{subfigure}\hfill
    \begin{subfigure}{0.45\textwidth}
    \centering
    \includegraphics[width=\textwidth]{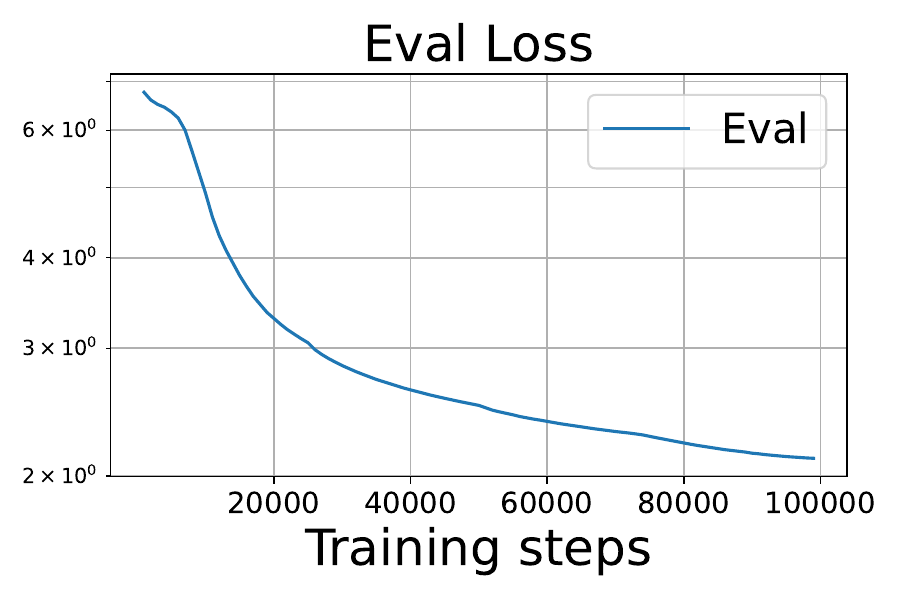}
    \label{fig:eval_ablate_app}
    \end{subfigure}\hfill
    \begin{subfigure}{0.32\textwidth}
    \end{subfigure}
    \caption{\looseness-1Train and Evaluation Loss behavior for a \bertbase{} model trained with {\method} for 100k steps. We have $4$ stages with 6-8-10-12 schedule (see \cref{sec:experiments} for details). The boundaries are at 25k, 50k, and 75k. Key observation: the model's train and evaluation loss change smoothly across stage transitions, indicating stability of the model to {\layerdrop} training.}
    \label{fig:bert100kruns_traintest}
\end{figure}

\section{Additional Experiment details}\label{sec:add_details}

\subsection{Details on \progressivelayerdrop{}} \label{sec:progressivedrop}
The progressive layer dropping algorithm \cite{zhang2020accelerating} is given in \cref{alg:pld_alg}.
\progressivelayerdrop{} employs a time and depth schedule that determines the probability of dropping each block. The time schedule begins with a zero probability of dropping each block. Then it increases this probability throughout the training process until it reaches a maximum of $(1-\bar{\alpha})$, where the hyperparameter $\bar{\alpha}=0.5$. We vary $\bar{\alpha}$ depending on the average FLOP perplexity we want to achieve in the end. With $\bar{\alpha} = p$, the FLOPs reduce by $1-(p/2)$.

The depth schedule ensures that blocks located earlier in the model are dropped with a lower probability than those located later/deeper. An important hyperparameter of the depth schedule in \progressivelayerdrop{} is $\gamma_f$, which controls the rate at which the probability of dropping layers increases. A higher value of $\gamma_f$ leads to a quicker increase in the probability. \citep{zhang2020accelerating} set $\gamma_f$ to 100 in their experiments.

\subsection{Details on variants of gradual stacking} \label{sec:stacking_variants_app}
Here, we give the details of the variants of gradual stacking, which we use to compare \method{} against in \cref{tab:raptr_v_stacking}. We simply initiate the depth-wise growth operator of each method to compare to \method{}.

\paragraph{bert2BERT} \citep{chen2021bert2bert} bert2BERT is a stacking based strategy that aims to progressively grow BERT models along width and depth at different stages of training. bert2BERT used progressive stacking \citep{gong2019efficient} to grow the depth of bert models during the course of training. Here, the model's depth is doubled at each transition phase (e.g. 3 layer BERT $\to$ 6 layer BERT $\to$ 12 layer BERT).

\paragraph{MSG} \citep{yao2023masked} MSG is a stacking based strategy that aims to progressively grow BERT models along $4$ dimensions: width, depth, number of attention heads, and MLP intermediate dimension at different stages of training. The main recipe is a small warm-up during phase transitions, where the model is slowly allowed to grow from its previous configuration to its new one. This is done to preserve the loss during phase transitions. MSG uses gradual stacking to increase the depth of the model, where top few layers are repeated to initialize new models at phase transitions.

\paragraph{LEMON} \citep{wang2023lemon} LEMON modified gradual stacking to instead use consecutive stacking. Here, the new layers are put consecutive to the source layers, whose parameters are copied to initialize the new ones. For example, for a schedule that grows a model from $6$ layers to $9$ layers, we choose the top $3$ layers as source layers and insert new layers between layers $(4, 5)$, $(5, 6)$, $(6,)$ repeating the parameters of layers $4, 5, 6$ respectively.

\paragraph{MSG with layer interpolation} We modify MSG's depth growth operator to LEMON's growth opertor. Here, instead of using the top few layers to stack on top of the existing model, we use consecutive stacking to put the newly initialized layers close to their source layers.

\textit{Remark:} In order to keep the comparisons uniform across the methods, we use the same learning rate schedule as \method{} across all methods.

\subsection{Details for polynomial training}\label{sec:toy_setting}
We present in   \cref{fig:raptr_v_layerdrop_toy} the results from a 20-layer residual network trained on 100-dimensional data. Our $f^{\star}$ has max degree $k=10$ restricted to the first $t=20$  variables ($\vx_{1:20}$) in the input. We pick $m=20$ random basis polynomials for each degree $\ell \le 10$.  
We set the average drop rate of \progressivelayerdrop{} to $\bar{\alpha}=0.6$, such that the average FLOPs is $80\%$ lower to a baseline run. For \method{}, we use $4$-stage training, with the stages represented by its $(p, \fixed)$ as $(8/20, \{\}), (12/20, \{\}), (16/20, \{\}),$ and$(20/20, \{\})$ respectively.  
The lengths of each stage have been set such that the average FLOPs is $80\%$ lower to a baseline run (please see \cref{sec:schedule_selection}).

\subsection{Details for BERT} 
\paragraph{Equal step Pretraining}  The batch-size is set to $256$ and $512$ for \bertbase{} and \bertlarge{} respectively. Following \citep{reddi2023efficient}, we use a learning rate warmup of 10k steps and keep learning rate constant in all phases, followed by a linear learning rate decay to $0$ in the final full model training phase. We set the peak learning rate as $10^{-4}$ in all our experiments.

\paragraph{Equivalent flop Pretraining} To make a fair comparison across all methods, we tune the peak the learning rate in the grid $\{1, 1.6, 2.5, 4\}\times 10^{-4}$  for \bertbase{} and grid $\{1, 1.6\} \times 10^{-4}$ for \bertlarge{}.

\paragraph{Finetuning} We fine-tune BERT models for $3$ epochs with batch-size $32$, using a learning rate grid search over $\{1, 2, 3, 5\} \times 10^{-5}.$ We use a linear warmup schedule over initial $6\%$ training steps, followed by constant learning rate schedule for the rest of the training.

\begin{minipage}[H]{\textwidth}
          \begin{algorithm}[H]
    \centering
    \caption{Progressive layer dropping (\progressivelayerdrop{}) \cite{zhang2020accelerating}}
    \label{alg:pld_alg}
      \begin{algorithmic}[1]
        \STATE \textbf{Input:} iterations $T$, layer keep probability $\bar{\alpha}$, temperature budget $\gamma_f>0$, layers $L$, functions (self-attention, layer-norm, feed-forward) $f_{ATTN}, f_{LN}, f_{FFN}$, loss function $\mathcal{L}$, data $(\mathbf{x}_0, \mathbf{y})$, output layer $f_O$
        \STATE $\gamma \leftarrow \frac{\gamma_f}{T}$
        \FOR {$t$\ $\leftarrow$ 1 to $T$}
            \STATE $p \leftarrow 1$ \algorithmiccomment{Keep probability.}
            \STATE $\alpha_t \leftarrow (1 - \bar{\alpha}) \exp(-\gamma \cdot t) + \bar{\alpha}$
            \STATE $p_d \leftarrow \frac{1 - \alpha_t}{L}$ \COMMENT{Layer decay.}
            \FOR {$i$\ $\leftarrow$ 0 to $L-1$}
                \STATE $s \sim \operatorname{Bernoulli}(p)$ \COMMENT{Keep or drop.}
                \IF {$s == 0$} 
                    \STATE $\mathbf{x}_{i+1} \leftarrow \mathbf{x}_i$ \COMMENT{Drop.}
                \ELSE 
                    \STATE $\mathbf{x}_{i}' \leftarrow \mathbf{x}_{i} + \frac{f_{ATTN}(f_{LN}(\mathbf{x}_{i}))}{p}$ %
                    \STATE $\mathbf{x}_{i+1} \leftarrow \mathbf{x}_{i}' + \frac{f_{FFN}(f_{LN}(\mathbf{x}_{i}'))}{p}$ %
                \ENDIF
                \STATE $p \leftarrow p - p_d$ \COMMENT{Decay prob.}
            \ENDFOR
            \STATE $\ell \leftarrow \mathcal{L}(f_O(\mathbf{x}_L), \mathbf{y})$
            \STATE $f_{ATTN}, f_{LN}, f_{FFN}, f_O \leftarrow \operatorname{Update}(\ell)$
        \ENDFOR
      \end{algorithmic}
     \label{alg:dropping}
     \end{algorithm}
     \end{minipage}

\subsubsection{Additional pretraining experiments}

\looseness-1{\bf Results at equal FLOPs.} Inspired by \citep{kaddour2023no}, we further compare {\method} and gradual stacking to baseline training by adjusting the number of training steps of baseline training to match its FLOPs to {\method} (\Cref{tab:bert_12_flop}).
For \bertbase{}, we observe that at shorter FLOP experiments, {\method} achieves better validation loss compared to baseline training and gradual stacking. This difference gets smaller as we move to larger horizon settings.
Similar findings for \bertlarge{}.
\begin{table}[t]
\centering
\small
\caption{Equal FLOPs comparisons for \bertbase{} and \bertlarge{} with extensive peak LR tuning. For \bertbase{} we use the best performing schedules from \cref{tab:bert_12}. For \bertlarge{}, we use 6-12-18-24 for both {\method} and stacking. FLOPs denotes the number of steps involved in baseline training.
Key observations: for \bertbase{} (a)   {\method} achieves  \rebuttal{better} loss compared to baseline at all FLOP measures, with larger differences for fewer FLOPs. (b) {\method} has $0.02$ better loss than gradual stacking at all FLOPs. (c) For \bertlarge{}, stacking and {\method} are competitive to each other at all FLOPs. Both methods have better loss than baseline at lower FLOPs. }
\label{tab:bert_12_flop}
\begin{tabular}{ccccc}
\toprule
Model & FLOPs & Baseline &  Stacking & \method{} \\
\midrule
\multirow{3}{*}{\bertbase{}} & $75k$ & 2.09 & 2.02 & \textbf{2.01} \\
& $170k$ & 1.90 & 1.88 & \textbf{1.86} \\
& $510k$ & 1.74 & 1.75 & \textbf{1.73} \\
\midrule
\multirow{3}{*}{\bertlarge{}} &$62.5k$ & 1.84 & \textbf{1.78} & 1.80\\
&$140k$ & 1.63 & \textbf{1.60} & 1.61\\
&$625k$ & \textbf{1.40} & 1.41 & 1.41 \\
\bottomrule
\end{tabular}
\end{table}

\subsection{Details for UL2}\label{sec:ul2_expdetails}

\paragraph{Pretraining} 
Similar to our experiments on BERT, we use a linear warmup of $10$K steps, followed by constant learning rate in all stages, and a cosine decay to $0.1 \times$ the peak learning rate in the final full model training stage. The peak learning rate was searched over $\{10^{-2}, 5\times 10^{-3}, 2\times10^{-2}, 3\times 10^{-2}\}$ and was fixed to $10^{-2}$.

We use Adafactor~\citep{shazeer2018adafactor} optimizer and train with a batch size $512$ for $400k$ steps on a mixture of Arxiv, C4~\citep{2020t5}, Github, and Wikipedia~\citep{wikidump} datasets, with mixing ratios $9\%, 57\%, 17\%, 17\%$ respectively (for a total of 100B tokens). This roughly corresponds to 0.8 epochs of C4.

Reported datasets in \cref{tab:ul2_1b} include Trivia QA~\citep{joshi2017triviaqa}, Tydi QA~\citep{clark2020tydi}, SQuADV2~\citep{rajpurkar2018know}, and SuperGLUE~\citep{wang2019superglue}.  For QA tasks, we report Exact match scores, and report average accuracy for SupreGLUE tasks. To reduce variance, we report average performance across $3$ runs for the most representative setting for each method.

Boundaries for 12-16-20-24 {\method} and Gradual Stacking are 40k, 120k, and 240k (decided using \cref{sec:schedule_selection}). The average  {\layerdrop} length used during training turns out to be $20$ with this schedule. When we introduce an initial 30k full-model training for {\method}, we shift the boundaries by 30k so that the average {\layerdrop} length stays $20$.

\subsubsection{Additional Pretraining experiments} In \cref{tab:ul2_1b_appendix}, we have the following additional runs compared to \cref{tab:ul2_1b} in the main paper. 

We first compare Baseline model with cosine decay learning rate to a Baseline model with Square-root learning rate decay that was used in the original paper~\citep{tay2022ul2}. We observe that the model with square-root learning rate performs significantly worse compared to the model with cosine decay learning rate. Hence, we use model with cosine decay learning rate as our baseline.

Secondly, we try another schedule for Gradual Stacking, which contains 4 stages and is represented by 6-12-18-24.  The stage lengths have been adjusted such that the average number of layers used over training is $20$. We set the stage boundaries as 17k, 74k, and 172k respectively. We observe that 6-12-18-24 Gradual Stacking performs much better than  12-16-20-24 Gradual Stacking on various downstream tasks. This signals towards the dependency of Gradual Stacking on proper schedule selection. 

On the other hand {\method} doesn't show a big difference with 6-12-18-24 schedule, when compared to 12-16-20-24 schedule. A proper analysis of both the methods on schedule dependencies is kept for future work.

\rebuttal{
\section{Schedule selection} \label{sec:schedule_selection}
}
\rebuttal{
We follow two schedules from \cite{reddi2023efficient}, Equal and Proportional, to build the lengths of each stage in {\method} and Stacking. For Equal scheduling, used in \Cref{tab:bert_12,tab:bert_12_flop}, we split training into $\stage$ equal stages.
For Proportional scheduling in \Cref{tab:ul2_1b}, we increase the length of a stage in proportion to index of the stage. For example, for 12-16-20-24 {\method} and Stacking with proportional scheduling for $400k$ training steps, the stage lengths are set as $40k$, $80k$, and $120k$ respectively.
}

\rebuttal{
The flop counts are decided on the basis of average {\layerdrop} length during training. For example, if the average {\layerdrop} length is $18$ for a model with $24$ layers, we consider relative flops of {\method} compared to baseline as $0.75$.
}

\rebuttal{
For 12-16-20-24 {\method} with Equal schedule, the average {\layerdrop} length during training is given by $18$, while for Proportional schedule, the average {\layerdrop} length during training is given by $20$. However, when either Equal and Proportional schedules don't return the target average {\layerdrop} length, we re-consider the stage lengths by taking the closest candidate among Equal and Proportional schedules and remove equal steps from each of the first $\stage - 1$ stages and add to the final full model training stage. For example, in \cref{tab:ul2_1b_appendix}, with 6-12-18-24 {\method} and target average {\layerdrop} length $20$, Proportional schedule returns an average {\layerdrop} length of $18$ and stage lengths $40k$, $80k$, $120k$ and $160k$ respectively. We then aim to find $x$ such that stage lengths of $40k - x$, $80k - x$, $120k - x$ and $160k + 3x$ returns the desired average {\layerdrop} length. On solving for $x$, we get $x = 22k$. 
}

\rebuttal{
\subsection{Ablation on different schedules for \bertbase{}}
}

\rebuttal{
We conduct experiments on {\method} with different stage schedules that have same relative flop counts in \cref{tab:bert_12}. We observe on \bertbase{} that  different schedules differ by atmost $0.01$ in pre-training loss and $0.1-0.5$\% in downstream performance on average. Thus, we conclude that {\method} is robust to the stage schedule selection.
}

\rebuttal{
For UL2, we use the following stage schedules to minimize computational overhead. We restrict the number of phases in {\method} to $4$. The first stage selects a sub-network with half of the layers at random at each step. We increase the sub-network size by an equal quantity in each stage transition, such that we train the full model in the final phase. 
}

\begin{table}[htbp]
\centering
\caption{\rebuttal{Performance of {\method} on \bertbase{}  with different schedules. The average {\layerdrop} size for each schedule is $9$, which makes each schedule $1.33$x faster  than the baseline model (provided for reference). We use the same hyperparameters as used in \cref{tab:bert_12}. We observe minor differences between different schedules, indicating robustness of {\method}.} }
\label{tab:bert_different_schedules}
\begin{tabular}{lcc|cccc} 
\toprule
  {\method} schedule & Rel. flops & Eval loss & MNLI & QNLI & SST-2 & Avg.\\
\toprule
3-6-9-12 & 1 & 1.76 & 82.0 & 89.6 & 92.0 & 87.9 \\
4-8-12 & 1 & 1.76 & 81.9 & 89.3 & 91.5 & 87.6  \\
6-9-12 &  1 & 1.75 & 82.3 & 89.2 & 91.0 & 88.0   \\
6-8-10-12 & 1 & 1.75 & 82.1 & 89.8 & 92.4 & 88.1 \\
\midrule 
Baseline & 1.33 & 1.76  & 81.5 & 90.5 & 91.4 & 87.8\\
\bottomrule
\end{tabular}
\end{table}

\rebuttal{\subsection{Ablation on fixed layer set $\fixed$ for \bertbase{}}
}
\rebuttal{
We conduct experiments on 6-8-10-12 {\method} on \bertbase{} with different fixed layer set $\fixed$ (\cref{tab:bert_different_fixed}). We observe that fixed set selection of $\{1, 12\}$ performs the best among other candidates. This selection however can be expensive to verify for other settings like UL2. Hence, in all subsequent experiments, we continue with the first and last layer fixed in all stages of {\method}.
}

\begin{table}[htbp]
\centering
\caption{\rebuttal{Performance of {\method} on \bertbase{}  with different $\fixed$ set for {6-8-10-12} {\method} run for $100k$ steps.  We use the same hyperparameters as used in \cref{tab:bert_12}, except the training steps reduced to $100k$ steps. We observe that fixing the first and the last layer at all times lead to slightly better performance, compared to other candidates.}} %
\label{tab:bert_different_fixed}
\begin{tabular}{lc|cccc} 
\toprule
  $\fixed$ set & Eval loss & MNLI & QNLI & SST-2 & Avg.\\
\toprule
 $\{\}$ & 2.12 & - & - & - & -\\
 $\{1\}$ & 2.13 & - & - & - & -\\
 $\{1, 12\}$ & \textbf{2.11} & 78.6 & 86.77 & 88.69 & 84.7 \\
 $\{1, 2, 12\}$ & 2.13 & - & - & - & -\\
 $\{1, 11, 12\}$ & 2.12 & 77.4 & 87.0 & 89.4 & 84.6 \\
 $\{1, 2, 11, 12\}$ & 2.14 & 78.0 & 86.1 & 88.8 & 84.3\\
 $\{1, 2, 3, 11, 12\}$ & 2.18 & 77.4 & 86.5 & 88.4 & 84.1\\
 $\{1, 2, 10, 11, 12\}$ & 2.16 & - & - & - & -\\
\bottomrule
\end{tabular}
\end{table}

\rebuttal{\section{System speed of different subnetworks}\label{sec:speedups}}
\rebuttal{We report the speed of training in steps/sec for UL2 models with different {\layerdrop} sizes on a TPU 16x16 chips  with percore batch size of $2$ in \cref{tab:ul2_speedups}. 
}
\rebuttal{
Suppose $t$ is the total time for $400k$ steps with a 24-layer model in UL2 training. In \cref{tab:ul2_1b}, we use 12-16-20-24 {\method} that has $20\%$ "theoretical" flop count reduction, with the stage lengths given by $40k$, $80k$, $120k$, and $240k$ respectively (\cref{sec:schedule_selection}). Theoretically, the necessary time required be $0.83t$.
Using the training speed numbers of \cref{tab:ul2_speedups}, it takes approximately $0.84t$ to complete $400k$ steps with {\method} on our machine, close to the theoretical training time.}

\begin{table}[htbp]
\centering
\caption{Training Steps/sec at different stages on 24-layer UL2}
\label{tab:ul2_speedups}
\begin{tabular}{lc} 
\toprule
  {\layerdrop} size & steps/sec \\
\toprule
12 & 2.0\\
16 & 1.6\\
20 & 1.4\\
24 & 1.1\\
\bottomrule
\end{tabular}
\end{table}

\rebuttal{
We trained our BERT-base models on $4 \times 4$ TPU cores with percore batch size of $8$. We report the speed of training in steps/sec with different {\layerdrop} sizes. 
}

\begin{table}[htbp]
\centering
\caption{Training Steps/sec at different stages on BERT-base}
\label{tab:bert_speedups}
\begin{tabular}{lc} 
\toprule
  {\layerdrop} size & steps/sec \\
\toprule
6 & 17.9 \\
8 & 14.3 \\
10 & 11.8 \\
12 &  10.2\\
\bottomrule
\end{tabular}
\end{table}

\rebuttal{
So, looking at \cref{tab:bert_12_flop}, let $t$ be the total time for completing $675k$ steps in BERT-base training. In the case of a 6-8-10-12 {\method} with Equal stage lengths, it takes approximately $0.79t$ to complete $675k$ steps on our machine. For an ideal machine that perfectly realizes the theoretical speedups based on the desired flop count reduction, the total time required will be $0.75t$.
}

\section{Efficient implementation of \method{} in distributed training}
\label{sec:efficient_raptr_impl}

\begin{figure}[!htbp]%
    \centering
    \begin{subfigure}{0.45\textwidth}
    \centering
    \includegraphics[width=\textwidth]{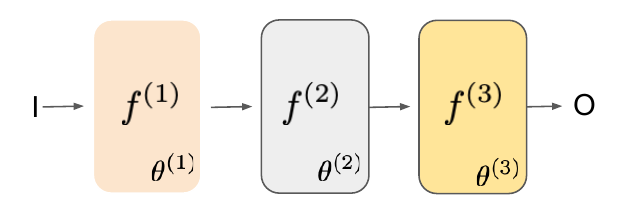}
    \caption{A representative 3-layer network}
    \end{subfigure}\hfill
    \begin{subfigure}{0.45\textwidth}
    \centering
    \includegraphics[width=\textwidth]{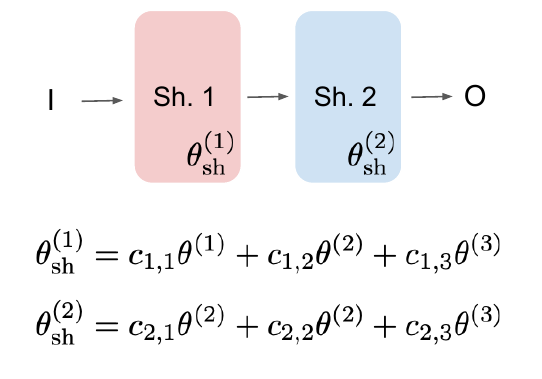}
    \caption{ Shared base sub-network across the $3$ possible $2$-layer sub-networks}
    \end{subfigure}\hfill
    \begin{subfigure}{0.8\textwidth}
    \centering
    \includegraphics[width=\textwidth]{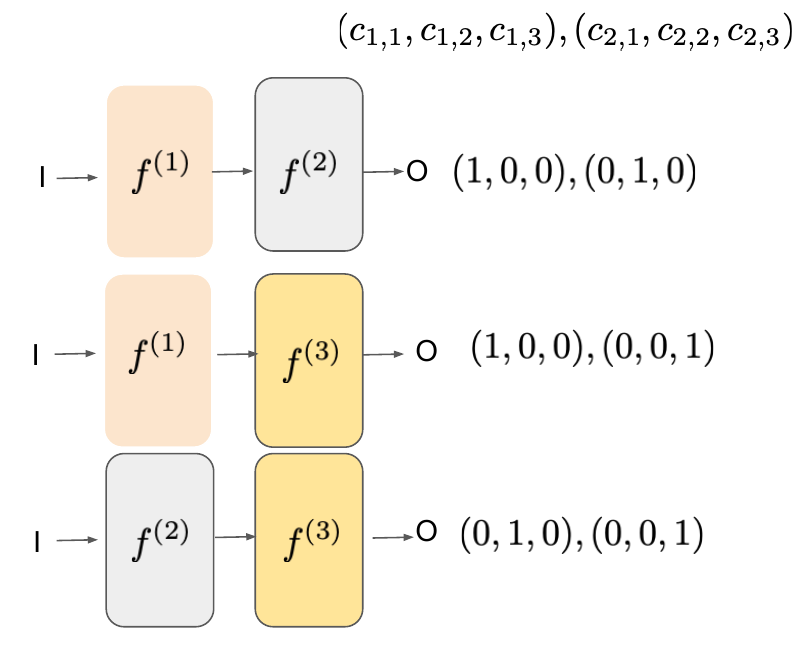}
    \caption{ Representing each $2$-layer sub-network in the shared base sub-network}
    \end{subfigure}
    \caption{\looseness-1 Representative overview of efficient implementation of \method{} in distributed training setting. A simple 'if-else' implementation of dropping layers doesn't work in distributed setting, as the optimizer expects gradients from each parameter from each GPU at each step of training. We instead represent all possible random sub-networks as a shared sub-network, whose parameters are represented as linear combination of all parameters. Then, we can back-propagate through a random sub-network by setting the linear coefficients appropriately.}
    \label{fig:shared_backbone_network}
\end{figure}

\paragraph{Issues with distributed training:} \method{} can be naively implemented with if-else statements that bypass a layer if a layer isn't part of a random sub-network at a specific time step. This will work with pytorch if we are training with a single GPU. However, if we train a model on multiple GPUs with pytorch, we get the following error.

\noindent %
\fbox{
    \parbox{\textwidth}{
            \vspace{2pt}
            RuntimeError: Expected to have finished reduction in the prior iteration before starting a new one. This error indicates that your module has parameters that were not used in producing loss. You can enable unused parameter detection by passing the keyword argument `find unused parameters=True' to `torch.nn.parallel.DistributedDataParallel', and by  making sure all `forward` function outputs participate in calculating loss. Parameter indices which did not receive grad for rank 1: 33 34 35 36 37 38 39 40 41 42 43 44 45 46 47 48 49 50 51 52 53 54 55 56 57 58 59 60 61 62 63 64 65 66 67 68 69 70 71 72 73 74 75 76 77 78 79 80 81 82 83 84 85 86 87 88 89 90 91 92 93 94 95 96 97 98 99 100 101 102 103 104 105 106 107 108 109 110 111 112 113 114 115 116 117 118 119 120 121 122 123 124 125 126 127 128 129 130 131 132 ...
        }
}

We get this error, because the optimizer expects gradients from each trainable model parameter from each GPU at each step of training. We get similar error, when we use jax framework to implement \method{}.

\paragraph{Solution -- Layer Selection as Linear Algebraic formulation on Parameters:} Instead of viewing \method{} as selecting a random sub-network composed of a random selection of layers, we should instead view sub-networks as a set of selected parameters in the Transformer model. All Transformer layers are identical in structure but with different parameters. 

Let $f^{(1)}, f^{(2)}, f^{(3)}, \cdots, f^{(L)}$ represent the $L$ Transformer layers with corresponding parameters $\theta^{(1)}, \theta^{(2)}, \cdots, \theta^{(L)}$. Each of these layers can be expressed as a common Transformer layer $f$, where the parameters $\theta^{(1)}, \theta^{(2)}, \cdots, \theta^{(L)}$ are used to represent $f^{(1)}, f^{(2)}, f^{(3)}, \cdots, f^{(L)}$, respectively. Furthermore, each layer $f^i$ can be represented with parameters $\sum_{j=1}^{L} c_{ij} \theta^{(j)}$, where $c_{ii} = 1$ and $c_{ij} = 0 $ for all $j \ne i$.

We can extend this idea to a sub-network that selects $k$ layers at random from the $L$ Transformer layers. We can maintain a $k$-layer shared sub-network, where shared layer $1 \le \ell \le k$ uses parameter $\theta^{(\ell)}_{\textrm{sh}} = \sum_{j=1}^{L} c_{\ell, j} \theta^{(j)}$. $c_{\ell, j}$ are fixed depending on the random sub-network of layers that we pick at each step of training. For example, for a sub-network that picks layers $1, 2, 3, \cdots, k$, $c_{\ell, i}$ is set as $1$ for $i=\ell$ for all $1 \le \ell \le k$ and $0$ otherwise. In doing so, all the parameters get gradients in the loss computation at each step of training.

\paragraph{Additional FLOPs involved in sub-network selection:} Because we have additional parameter additions to represent sub-networks at each step of training, one might wonder whether the additional FLOPs incurred will lead to drastic slowdown in model training. To do so, we will compare the number of FLOPs necessary for backpropagation for a $k$-random sub-network to the number of FLOPs necessary to compute gradients for a $k$-layer static sub-network used in stacking. The backpropagation step includes a forward pass, a backward pass, and a gradient descent step.

To further simplify our explanation, we will focus on an $L$-layer network that uses linear layers as its Transformer layers with parameters $\mW^{(1)}, \cdots, \mW^{(L)} \in \mathbb{R}^{d \times d}.$ We will focus on Gradient Descent algorithm for simplicity. Batch size is given by $B$.

\paragraph{Back-propagation FLOP Computation for a $k$-layer static sub-network:} For simplicity, we will use the first $k$ layer parameters to define the static sub-network.  We will refer to the output of each layer $\ell$ as $\mY^{(\ell)} \in \mathbb{R}^{d \times d}.$
\begin{itemize}
    \item \textbf{Forward pass:} For an input $\mY^{(\ell)} \in \mathbb{R}^{B \times d}$, the forward pass involves computing $\mW^{(\ell)} \vy^{(\ell)}_i$ for each $\vy^{(\ell)}_i \in \mY^{(\ell)}$, for each layer $\ell \in [k]$. Thus, the number of FLOPs involved is $k B d^2$.

    \item \textbf{Backward pass:} On gradient $\mG^{(\ell)} \in \mathbb{R}^{B \times d}$ back-propagating to a layer $\ell$, the backward pass computes the gradient w.r.t. the parameter $\mW^{(\ell)}$ and the gradient that flows to the next layer. 
    
    The gradient w.r.t. the weight is given as $\frac{1}{B} \sum_{i=1}^{B} \vg_i^{(\ell)} (\vx_i^{(\ell)})^{\top}$. This involves $B d^2$ computation, and overall a $kBd^2$ computation for all weights. 
    
    Furthermore, the gradient to be back-propagated beyond layer $\ell$ is given as 
    
    $\left[ \vg_1^{(\ell)} (\mW^{(\ell)})^{\top}, \cdots \vg_i^{(\ell)} (\mW^{(\ell)})^{\top}, \cdots, \vg_B^{(\ell)} (\mW^{(\ell)})^{\top} \right]$, which incurs an additional $B d^2$ computation. Overall, a $kBd^2$ computation across all layers.

    \item \textbf{Gradient Descent step:} Finally, the descent step is a simple addition of negative gradients to each weight, which requires $kd^2$ for all weights. 

\end{itemize}

\paragraph{Back-propagation FLOP Computation for a $k$-layer random sub-network:} For a randomly picked sub-network, we will use $\mW^{(\ell)}_{\textrm{sh}} = \sum_{j=1}^{L} c_{\ell, j} \mW^{(j)}$ to denote the parameters of the $k-$layer shared sub-network that we use to perform back-propagation at particular step.
\begin{itemize}    
    \item \textbf{Forward pass:} Our proposed method requires an additional sum of weights in each layer, with the coefficients $c_{\ell,j}$ for $1 \le \ell \le k, 1 \le j \le L$ in the forward pass. Hence, the total FLOP computation involved is $k B d^2 + k L d^2$. 

    \item \textbf{Backward pass:} Here, there are no changes for backpropagation through the $k$-layer shared network, as the gradients are computed and backpropagated w.r.t. $\mW^{(\ell)}_{\textrm{sh}}$.
    
    \item \textbf{Gradient Descent step:} Here, the gradients need to be distributed to the weight parameters $\mW^{(\ell)}$ before taking the descent step.
    \begin{align*}
        \nabla \mW^{(\ell)} = \sum_{i=1}^{k} c_{i, \ell} \nabla \mW^{(\ell)}_{\textrm{sh}}.
    \end{align*}
    This requires $kLd^2$ computation, in addition to the necessary gradient descent steps.
\end{itemize}

Thus, the relative FLOPs of backpropagating through a $k$-layer random sub-network to a $k$-layer static sub-network (used in stacking)   is
\begin{align*}
    \frac{2kBd^2+kd^2+2kLd^2}{2kBd^2+kd^2} = 1 + \frac{2L}{2B+1}.
\end{align*}

With $B$ set high, this relative FLOPs gap is close to $1$.

\paragraph{Discussion for language models:} For language models, the batch size represents the number of tokens that pass through the layer during backpropagation. In our UL2 experiments, sequences were of length $512$ and the number of sequences in each minibatch was set as $2$. Hence, the effective batch size is $B=512 \times 2 = 1024$. The number of layers $L$ in our UL2 setting is $24$. Hence, the relative FLOPs of backpropagating through a $k$-layer random network compared to a $k$-layer static network (used in stacking) would be $\approx (48 / (2048 + 1) + 1) = 1.02$ if we had used the same minibatch for our linear network setting. 

\textit{Remark:} %
When training large models, we can structure the selection of random sub-networks to minimize the performance overhead associated with random sampling. For instance, if we are training a $100$-layer network and wish to train a random sub-network of $50$ layers at a given step, rather than selecting $50$ layers arbitrarily from the entire model, we can organize the layers into $10$ groups, each containing $10$ layers. Then, we create a $50$-layer shared sub-network, where each shared layer is a linear combination of $10$ layer parameters from a specific group. This approach reduces the computational overhead associated with combining parameters from $50 \times 10 \times d^2$ to $5 \times 10 \times d^2$.

\section{Noise stability for transformers}
\begin{theorem}\label{thm:err_prop}
    Suppose $\netloss$ satisfies the following condition $\netloss$ satisfies the following condition for any input $\Tilde{\vx}, \eta \in \RR^{\dims}$ and label $\vv \in \{0,1\}^{\vocab}$ for some constant $\mu_{loss}$ :
     \begin{align}
         \netloss(\Tilde{\vx} + \eta \norm{\Tilde{\vx}}, \vv) - \netloss(\Tilde{\vx}, \vv) \le \liploss  \norm{ \eta}_2. \label{eq:losslip}
     \end{align}
    The difference between expected loss of $(\layer-1)$-{\method} and $\layer$-{\method}, i.e. $|\netloss_2(F) - \netloss_1(F)|$, is upper bounded by $\frac{C}{\layer} \expect_{\vx \in \distrib}  (\sum_{\ell=1}^{\layer} \Psi_{\ell}(\vx)) / \norm{F(\vx)}_2$
    for some constant $C$ that depends on regularity properties of the head $\head$ on top of the backbone $F$.
\end{theorem}

\begin{proof}[Proof of \cref{thm:err_prop_inf}]
    Denote by $F_{-\ell}$ the $\layer$ candidates that we can randomly choose from during $\layer-1$ {\method} training. The output of $F_{-\ell}$ on any input $\vx$ is equivalent to the output of the model $F$ on $\vx$ and $\alpha_{1:\layer}$, with $\alpha_{\ell}=0$ and the rest set to $1$.

    Then, the average loss of $\layer-1$ random training is given by
    \begin{align*}
        \netloss_{2}(F) = \frac{1}{\layer}  \sum_{i=1}^{\layer} \expect_{\vx, \vy \sim \distrib} \netloss\left( F_{-\ell} (\vx), \vv \right).
    \end{align*}
    
    On the other hand, the loss of the full model is given by
    \begin{align*}
        \netloss_{1}(F) = \expect_{\vx, \vv \sim \distrib} \netloss\left( F (\vx), \vv \right).
    \end{align*}

    Hence, the expected difference between loss of $\layer-1$ random training and full model is given by

    \begin{align*}
        \netloss_{2} (F) - \netloss_{1} (F) &= \frac{1}{\layer}  \sum_{i=1}^{\layer} \expect_{\vx, \vv \sim \distrib} \left( \netloss\left( F_{-\ell} (\vx), \vv \right) -  \netloss\left( F (\vx), \vv \right) \right) \\&
        \leq  \frac{1}{\layer}  \sum_{i=1}^{\layer} \expect_{\vx, \vv \sim \distrib} \quad \liploss \frac{ \norm{ F_{-\ell} (\vx) - F (\vx) }_2 }{ \norm{ F (\vx) }_2 } \\&
        \leq  \frac{1}{\layer}  \sum_{i=1}^{\layer} \expect_{\vx, \vv \sim \distrib} \quad \liploss \frac{ \psi_{\ell} }{ \norm{ F (\vx) }_2 }.
    \end{align*}
    Here the pre-final step uses the definition of $\liploss$ from \cref{eq:losslip}, and the final step uses the definition of $\Psi_{\ell}$ from \cref{eq:lipstacklayers}.
\end{proof}

\paragraph{Discussion} An example of a loss function with bounded $\liploss$ is a transformer model, that uses a layer normalization layer before using an embedding matrix $\embed$ to compute the logits. The logits are then compared with the true label with a cross-entropy loss with a softmax function on logits. $\psi(\netloss)$ hence depends on the $\ell_2$ norm of $\embed$ and the weight parameters of the layer normalization layer.

\subsection{Noise stability of linear residual networks}

\subsubsection{Useful lemmas}
To prove, we will require the following theorem (and its variaions). For random initialization, we use $\mW_1, \cdots, \mW_{\layer} \sim \mathcal{N}(0, d^{-1}).$

\begin{lemma}[Norm of the output of linear layers at initialization]\label{lem:linearnorm}
    Consider the case of a linear network with no residual connections and layernorm, i.e.
    \begin{align*}
        \vy^{(\ell)} = \mW_{\ell} \vy^{(\ell-1)}, \text{ for all } \ell \ge 1, 
    \end{align*}
    with $\vy^{(0)} = \vx$.
    
    For an input $\vx$, any $\varepsilon \ge 0$, with probability at least $1 - \layer e^{-\Omega \left(d \varepsilon^2 / \layer^2 \right)}$ over the randomness of $\mW_1, \cdots, \mW_{\layer}$, we have
    $$
      \norm{\vy^{(\ell)}}_2  \in [ 1-\varepsilon, 1+\varepsilon ] \quad \text{for all} \ell \ge 1.
    $$
\end{lemma}

\begin{proof}
    Because $\mW_{\ell}$ is gaussian, $\vy^{(\ell)} = \mW_{\ell} \vy^{(\ell-1)}$ follows $\mathcal{N}(0, \norm{\vy^{(\ell-1)}}_2^2 / d )$. Thus, $\{ (\vy^{(\ell)}_i)^2 \}_{i=1}^{d}$ follow a Chi-squared distribution, with $\mathbb{E} \norm{\vy^{(\ell)}}_2^2 = \mathbb{E} \sum_{i=1}^{d} (\vy^{(\ell)}_i)^2$ given by $\norm{\vy^{(\ell-1)}}_2^2$. The proof then follows from applying concentration bounds on the norm of gaussian vectors (Formula 3.7, \cite{vershynin2020high}), i.e. for any $\varepsilon \ge 0$, 
    \begin{align*}
        \Pr\left( \abs{ \norm{\vy^{(\ell)}}_2 - \norm{\vy^{(\ell-1)}}_2 } \ge \varepsilon \norm{\vy^{(\ell-1)}}_2 \right) \le e^{-\Omega(d \varepsilon^2) }.
    \end{align*}

    We apply a union bound over all layers, to get the final bound.
\end{proof}

\begin{corollary}[Norm of the output of linear layers with residual connections at initialization]\label{len:cor_linearnorm}
    Consider the case of a linear network with residual connections but no layernorm, i.e.
    \begin{align*}
        \vy^{(\ell)} = \vy^{(\ell-1)} + \mW_{\ell} \vy^{(\ell-1)}, \text{ for all } \ell \ge 1, 
    \end{align*}
    with $\vy^{(0)} = \vx$.
    
    For an input $\vx$, any $\varepsilon \ge 0$, with probability at least $1 - \layer e^{-\Omega \left(d \varepsilon^2 / \layer^2 \right)}$ over the randomness of $\mW_1, \cdots, \mW_{\layer}$, we have
    $$
      \norm{\vy^{(\ell)}}_2 \in [ 2^{\ell/2} (1-\varepsilon), 2^{\ell/2} ( 1+\varepsilon) ] \quad \text{for all } \ell \ge 1.
    $$
\end{corollary}

\begin{proof}
    The proof follows similar to \ref{lem:linearnorm}. W.h.p., we can show that
    for any $\varepsilon \ge 0$, 
    \begin{align*}
        \Pr\left( \norm{\mW_{\ell} \vy^{(\ell-1)}}_2 - \norm{\vy^{(\ell-1)}}_2 \ge \varepsilon \right) \le e^{-\Omega(d \varepsilon^2 / \layer^2 ) }.
    \end{align*}
    
    The only remaining piece is to show that $\vy^{(\ell-1)}$ and $\mW_{\ell} \vy^{(\ell-1)}$ are roughly orthogonal. First of all, $\mathbb{E} (\vy^{(\ell-1)})^{\top} \mW_{\ell} \vy^{(\ell-1)} = 0$ because $\mW_{\ell}$ is concentrated around $0$. Furthermore, $\mathbb{E} ((\vy^{(\ell-1)})^{\top} \mW_{\ell} \vy^{(\ell-1)})^2 = \sum_{i, j} w_{\ell, i, j}^2 (\vy^{(\ell-1)}_i)^2 (\vy^{(\ell-1)}_j)^2 = \frac{1}{d} \norm{\vy^{(\ell-1)}}^4$. Thus, using Hanson wright inequality (theorem 6.2.1, \cite{vershynin2020high}), we have w.p. atleast 
    \begin{align*}
        \Pr [ \abs{ (\vy^{(\ell-1)})^{\top} \mW_{\ell} \vy^{(\ell-1)} } \ge \varepsilon \norm{\vy^{(\ell-1)}}^2 ] \le e^{-\Omega(d \varepsilon^2) }.
    \end{align*}

    We apply a union bound over all layers, to get the final bound.

\end{proof}

\begin{corollary}[Norm of the output of linear layers with residual connections at initialization]\label{len:cor_linearnorm_layernorm}
    Consider the case of a linear network with residual connections and layer-normalization, i.e.
    \begin{align*}
        \vy^{(\ell)} = \vy^{(\ell-1)} + \mW_{\ell} \vy^{(\ell-1)} / \norm{\vy^{(\ell-1)}}_2, \text{ for all} \ell \ge 1, 
    \end{align*}
    with $\vy^{(0)} = \vx$.
    
    For an input $\vx$, any $\varepsilon \ge 0$, with probability at least $1 - \layer e^{-\Omega \left(d \varepsilon^2 / \layer^2 \right)}$ over the randomness of $\mW_1, \cdots, \mW_{\layer}$, we have
    $$
      \norm{\vy^{(\ell)}}_2 \in [ \sqrt{\ell+1} (1 - \epsilon), \sqrt{\ell+1} (1 + \epsilon) ] \quad \text{for all } \ell \ge 1.
    $$
\end{corollary}

\begin{proof}
    The proof stays extremely similar to \ref{len:cor_linearnorm}. For simplicity, we simply refer to \ref{len:cor_linearnorm} for any such variations.
\end{proof}

\subsubsection{Proofs of \cref{lem:random}}

\begin{proof}[Proof of \cref{lem:random}]
    We outline the proof for each case.
    \begin{enumerate}
        \item With residual connection and layer normalization, the function $F$ computes the intermediate activations $\vy^{(1)}, \cdots, \vy^{(\layer)}$ on an input $\vx$ as follows.
        \begin{align*}
            \vy^{(\ell)} = \left(\mI + \frac{\mW_{\ell}}{\norm{\vy^{(\ell-1)}}} \right) \vy^{(\ell-1)}.
        \end{align*}

        By \cref{len:cor_linearnorm_layernorm}, we can show w.h.p. with respect to the randomness of $\mW_1, \cdots, \mW_{\layer}$, for all $\ell \le \layer$, $\sqrt{\ell}( 1 - \mathcal{O}(1)) \le \norm{ \vy^{(\ell)} }_2 \le \sqrt{\ell}( 1 + \mathcal{O}(1) ).$

        Ignoring the error terms, we can simply write
        \begin{align*}
            \vy^{(\ell)} = \left(\mI + \frac{\mW_{\ell}}{ \sqrt{ \ell } } \right) \vy^{(\ell-1)}.
        \end{align*}

        Pick a general layer $\ell$ to be dropped. Then,
        \begin{align*}
            F_{-\ell} (\vx) -  F_{} (\vx) = \prod_{\ell'=\ell+1}^{\layer} \left(1 + \frac{\mW_{\ell'}}{\sqrt{\ell'}}\right) \frac{\mW_{\ell}}{\sqrt{\ell}} \vy^{(\ell-1)} + err,
        \end{align*}
        where the $err$ term appears because of the change in scales of activations $\vy^{(\ell+1)}, \cdots, \vy^{(\layer)}$ with the dropping of layer $\ell$. This error can be bounded as $\mathcal{O}(\sqrt{\layer}/\ell)$ using the same procedure followed below to bound the first term on R.H.S..

        Similar to bounding the norms of $\vy^{(\ell)}$, w.h.p. we can show that
        \begin{align*}
            \norm{ \prod_{\ell'=\ell+1}^{\layer} \left(1 + \frac{\mW_{\ell'}}{\sqrt{\ell'}}\right) \frac{\mW_{\ell}}{\sqrt{\ell}} \vy^{(\ell-1)} }_2 \le \sqrt{ \frac{\layer}{\ell} }.
        \end{align*}

        Hence, 
        \begin{align*}
           \psi_{\ell} := \norm{  F_{-\ell} (\vx) -  F_{\ell} (\vx) }_2 \le \mathcal{O}(\sqrt{\layer/\ell}). 
        \end{align*}

        This implies
        \begin{align*}
            \frac{1}{\layer} \sum_{\ell} \psi_{\ell} (\vx) = \frac{1}{\layer} \sum_{\ell} \mathcal{O}(\sqrt{\layer/\ell}) = \mathcal{O}(1).
        \end{align*}

        We then apply a union over all $\vx \in S$ such that the above condition holds true for all inputs in set $S$.
        Since the gap in $\mathcal{L}_2$ and $\mathcal{L}_1$ is bounded by $\mathcal{O}( \expect_{\vx} \frac{1}{\layer} \sum_{\ell} \psi_{\ell} (\vx) / ||F(\vx)|| )$ from \cref{thm:err_prop_inf}, we have the gap as $\mathbf{O}\left(\frac{1}{\sqrt{\layer}}\right)$.

        \item With no normalization, the function $F$ looks as follows.
        \begin{align*}
             \vy^{(\ell)} = \left(\mI + \mW_{\ell} \right) \vy^{(\ell-1)}.
        \end{align*}

        By \cref{len:cor_linearnorm},  
        we can show w.h.p. with respect to the randomness of $\mW_1, \cdots, \mW_{\layer}$, for all $\ell \le \layer$, $2^{\ell/2}( 1 - \mathcal{O}(1)) \le \norm{ \vy^{(\ell)} }_2 \le 2^{\ell/2}( 1 + \mathcal{O}(1) ).$

        With a drop in layer, we get
        \begin{align*}
            F_{-\ell} (\vx) -  F_{} (\vx) = \prod_{\ell'=\ell+1}^{\layer} \left(\mI + \mW^{\ell'} \right)\mW_{\ell} \prod_{\ell'=1}^{\ell-1} \left(\mI + \mW^{\ell'} \right) \vx.
        \end{align*}

        Similar to \cref{len:cor_linearnorm},  
        we can show w.h.p. with respect to the randomness of $\mW_1, \cdots, \mW_{\layer}$, 
        
        $\norm{\prod_{\ell'=\ell+1}^{\layer} \left(\mI + \mW^{\ell'} \right)\mW_{\ell} \prod_{\ell'=1}^{\ell-1} \left(\mI + \mW^{\ell'} \right) \vx}_2 \ge \mathcal{O}( 2^{(\ell-1)/2}).$

        This implies $\norm{F_{-\ell} (\vx) -  F_{} (\vx)}_2 = \Omega (2^{(\ell-1)/2}).$

        \item Without residual connection, the function $F$ looks as follows.
        \begin{align*}
             \vy^{(\ell)} =\mW_{\ell} \frac{\vy^{(\ell-1)}}{ \norm{ \vy^{(\ell-1)} }_2 }.
        \end{align*}

        Similar to \cref{len:cor_linearnorm_layernorm},  
        we can show w.h.p. with respect to the randomness of $\mW_1, \cdots, \mW_{\layer}$, for all $\ell \le \layer$, $( 1 - \mathcal{O}(1)) \le \norm{ \vy^{(\ell)} }_2 \le ( 1 + \mathcal{O}(1) ).$ Thus, ignoring error terms, the 
        network $F$ roughly looks like
        \begin{align*}
             \vy^{(\ell)} =\mW_{\ell} \vy^{(\ell-1)}.
        \end{align*}

        With a drop in layer, we get
        \begin{align*}
            F_{-\ell} (\vx) -  F_{} (\vx) =  \prod_{\ell'=\ell+1}^{\layer} \mW_{\ell'} (\mW^{(\ell)} - \mI) \prod_{\ell'=1}^{\ell-1} \mW_{\ell'} \vx.
        \end{align*}

        Using randomness of $\mW^{(\ell)}$ this can be shown to be of norm $\Omega(1)$.

    \end{enumerate}

\end{proof}

\begin{lemma}\label{lem:identical}
    When the layers are perfectly aligned, i.e. all weights $\mW_i = \mA$ for some matrix $\mA$ with $\norm{\mA}_2 = 1$ and for simplicity, assume second eigenvalue $\lambda_2(\mA) = 1-\delta$ for some $\delta > 0$, we have
    \begin{enumerate}[label=(\alph*)]
        \item With residual connection \& layernorm, we have $\Psi_{\ell} (\vx) = \mathcal{O}(1)$ and $\norm{F(\vx)} = \Omega(\layer)$. Thus, the gap in losses between stages is $\mathcal{O}(1/\layer)$. 
        \item Without residual connection, we have both $\Psi_{\ell} (\vx) = \mathcal{O}( \delta^{-1} e^{-\ell} )$ and $\norm{F(\vx)} = \mathcal{O}(1)$. Thus the gap in losses between stages is $\mathcal{O}( \frac{1}{\layer} )$. 
        \item Without layernorm, we have $\Psi_{\ell} (\vx) = \Omega(2^{\layer-1})$ and $\norm{F(\vx)} = \mathcal{O}(2^{\layer})$ . Thus the gap in losses between stages can be $\Omega(1)$. 
    \end{enumerate}    
\end{lemma}

\begin{proof}[Proof of \cref{lem:identical}]

We outline the proof for each case.

\begin{enumerate}
    \item With residual connection \& layernorm, the function $F$ computes the intermediate activations $\vy^{(1)}, \cdots, \vy^{(\layer)}$ on an input $\vx$ as follows.
        \begin{align*}
            \vy^{(\ell)} = \left(\mI + \frac{\mA}{\norm{\vy^{(\ell-1)}}} \right) \vy^{(\ell-1)}.
        \end{align*}
        We show that the above method will behave like power method, i.e. $\vy^{(\ell)}$ will be of magnitude $\Omega(\ell)$ and will be $\varepsilon$-close in angle to the top eigenvector of $\mA$, denoted as $\vv_1(\mA)$, provided $\ell \ge \Omega( (1/\delta) \log (1/\varepsilon) ).$

        We give an inductive argument as follows. Consider any layer $\ell$. %
        Say $\theta_{\ell}$ denotes the angle between $\vy^{(\ell)}$ and $\vv_1(\mA)$. Also, say $\project^{\perp}_{\vv1}$ denotes an orthogonal projection to subspace spanned by the rest of the eigenvectors of $\mA$. Then,
        \begin{align*}
            \abs{ \tan \theta_{\ell+1} } &= \frac{ \abs{ \langle \vv_1(\mA), \vy^{(\ell+1)} \rangle} }{ \norm{ \project^{\perp}_{\vv1} (\vy^{(\ell+1)})  }_2  } \\
            &= \frac{ \abs{ \langle \vv_1(\mA), \vy^{(\ell)} \rangle  + \langle \vv_1(\mA), \mA \vy^{(\ell)} \rangle} }{ \norm{ \project^{\perp}_{\vv1} (\vy^{(\ell)}) + \project^{\perp}_{\vv1} \mA \vy^{(\ell)}  }_2  } \\
            & = \frac{ 2 \abs{ \langle \vv_1(\mA), \vy^{(\ell)} \rangle } }{ \norm{ \project^{\perp}_{\vv1} (\vy^{(\ell)}) + \project^{\perp}_{\vv1} \mA \vy^{(\ell)}  }_2  }  \\&
            \leq \frac{ 2 \langle \vv_1(\mA), \vy^{(\ell)} \rangle  }{ \norm{ \project^{\perp}_{\vv1} (\vy^{(\ell)}) + \lambda_2 (\mA) \project^{\perp}_{\vv1} \vy^{(\ell)}  }_2  } \\&
            = \frac{2}{1 + \lambda_2(\mA)} \abs{ \tan \theta_{\ell} }.
        \end{align*}

        This implies, under $\lambda_2(\mA) < 1$, $\abs{\tan \theta_{\ell}}$ decreases with $\ell$. Under the assumption that $\vx$ isn't orthogonal to $\vv_1(\mA)$, the above condition simply shows that if $\ell \ge \mathcal{O}( (1/\delta) \log (1/\varepsilon) )$, than $\abs{\tan \theta_{\ell}} < \varepsilon$. 

        Furthermore, once aligned (or closely aligned) to $\vv_1(\mA)$, the norm of $\vy^{(\ell)}$ grows linearly. Hence, $\norm{\vy^{(\ell)}} = \Omega(\layer)$. Furthermore, for any $\ell$, the gap between $F_{-\ell}$ and $F_{\ell}$ simply becomes equal to the gap between $\layer$ and $\layer-1$ steps of the modified power method, which can be bounded as $\mathcal{O}(1).$

    \item With no residual connection,  the function $F$ computes the intermediate activations $\vy^{(1)}, \cdots, \vy^{(\layer)}$ on an input $\vx$ as follows.
        \begin{align*}
            \vy^{(\ell)} = \frac{\mA}{\norm{\vy^{(\ell-1)}}} \vy^{(\ell-1)}.
        \end{align*}
        From the update, it's trivial to come to the conclusion that $\vy^{(\ell)}$ will stay unit norm for any layer $\ell$.

        This update is exactly power-method, where $\vy^{(\ell)}$ gets $\varepsilon$-close to the top eigenvector direction of $\mA$ in $\mathcal{O}( (1/\delta) \log (1/\varepsilon))$ steps. The result then follows from bounding the gap between $\layer$ and $\layer-1$ steps of power-method.

    \item The proof is very similar to the proof of \cref{lem:random} (3), where the major difference comes from the blowup in the norms of the intermediate activations.  
\end{enumerate}

\end{proof}

\section{\method{} motivation with boolean data: Theory}

\paragraph{Data and labels:} We use uniformly sampled boolean data $\vx \sim \mathcal{U}(\{\pm1\}^{\dims})$. The true label function $f^*: \mathbb{R}^{\dims} \to \mathbb{R}$ for each data is given as
\begin{align*}
    f^*(\vx) = \frac{\sqrt{3}}{2} +  \frac{\sqrt{3}}{2} x_1 - x_1 x_2.
\end{align*}

\paragraph{Network to train:} We train a 2-layer residual network $f_{p_0, \vw_1, \vw_2, b_1, b_2}: \mathbb{R}^{\dims} \to \mathbb{R}$ comprised of two single neuron residual blocks, where $p_0$ behaves as a position bias that we add to the output of the first block. On an input $\vx$, the network computes output $y \in \mathbb{R}$ with intermediate output $y^{(1)}, y^{(2)}  \in \mathbb{R}$ as follows. 
\begin{align}
    y^{(1)} &= p_0 + \sin(\langle \vw_{1}, \vx \rangle + b_{1}), \nonumber \\
    y^{(2)} &= y^{(1)} + \sin(\langle \vw_{2}, \vx \rangle + y^{(1)} + b_{2}). \label{eq:trained_network}
\end{align}
Here $y := y^{(2)}$ is returned as the final output. The network is trained with mean squared error, given by
\begin{align}
    \netloss = \mathbb{E}_{\vx \sim \mathcal{U}(\{\pm1\}^{\dims})} ( y - f^*(\vx) )^2. \label{eq:network_output}
\end{align}
We consider population gradients for the sake of clean exposition of the important ideas. Interested readers can modify the proof to work for finite sample regimes.

\paragraph{Initialization of weights and biases:} The elements of the weights and biases $\vw_1, \vw_2,$ $p_0, b_1, b_2$ have been initialized from all $0$s.

\paragraph{\method{} Initial position bias training:} In the initial bias phase training, we simply train the bias $p_0$, i.e. at each step, we train with the following mean squared loss,
\begin{align}
	\netloss = \mathbb{E}_{\vx \sim \mathcal{U}(\{\pm1\}^{\dims})} (p_0 - f^*(\vx) )^2.  \label{eq:initbias_train}
\end{align}

\paragraph{\method{} First phase training:} In the first phase, both the layers are trained independently. That is, at each step, we pick a random layer $\ell \in \{1, 2\}$ and train with the mean squared loss, given by 
\begin{align}
    \netloss = \mathbb{E}_{\vx \sim \mathcal{U}(\{\pm1\}^{\dims})} (p_0 +  \sin(\langle \vw_{\ell}, \vx \rangle + b_{\ell}) - f^*(\vx) )^2. \label{eq:network_output_firstphase}
\end{align}

\paragraph{\method{} Second phase training:} In the second phase of training, we train the full model (\cref{eq:trained_network}) with the mean squared loss (\cref{eq:network_output}). For simplicity, we fix the parameters of the first layer and simply train the parameters $\vw_2, b_2$. 

We first recall the theorem corresponding to the stage-wise training of \method{}.
\begin{theorem}\label{thm:raptr_stagewise}
    For a small enough learning rate $\eta < \frac{1}{poly(\dims)},$ under parameter initialization from all $0$s, an initial position bias only training, followed by a $2$-stage \method{} on the network $f_{p_0, \vw_1, \vw_2, b_1, b_2}$ (\cref{eq:trained_network}) with mean squared error shows the following behavior.
    \begin{itemize}
	
        \item After  $\Theta(1/\eta)$ steps of position bias only training and $\Theta(1/\eta)$ steps of first phase, for both layers $\ell \in \{1, 2\}$,
        \begin{align*}
            \abs{ p_0 + \sin(\langle \vw_{\ell}, \vx \rangle + b_{\ell}) - \frac{\sqrt{3}}{2} - \frac{\sqrt{3}}{2} x_1 } \le \mathcal{O}(\eta).
        \end{align*}

        \item After $\Theta(1/\eta)$ steps of second phase, the ouput of the second layer is given as
        \begin{align*}
             \abs{ y^{(2)} - y^{(1)} - (-x_1 x_2) } := \abs{ \sin(\langle \vw_{2}, \vx \rangle + y^{(1)} + p_0+ b_{2}) - (-x_1 x_2) } \le \mathcal{O}(\eta).
        \end{align*}
    \end{itemize}
    Hence, after $\Theta(1/\eta)$ steps of initial position bias and 2-phase \method{} training, loss $\netloss < \mathcal{O}(\eta)$ (\ref{eq:network_output}) .
\end{theorem}

\begin{proof}
    The proof proceeds by analyzing the behavior of the weights and the biases in the different phases. 
    \begin{itemize}
    	\item In the initial position bias only training, we show that in $\Theta(1/\eta)$ steps, 
	\begin{align*}
	\abs{p_0 - \sqrt{3}/2} \le \mathcal{O}(\eta).
	\end{align*}
	
        \item In the first phase, after $\Theta(\eta \log (d/\eta))$ steps, we show in \cref{lem:phase1} that for any layer $\ell \in \{1, 2\}$, the weights and the biases satisfy the following conditions. 
        \begin{align*} 
            &\abs{w_{\ell, 1} - \frac{\pi}{3}} \le \mathcal{O}(\eta), \\
            &\abs{w_{\ell, j}} \le \mathcal{O}(\eta), \text{ for all } 2 \le j \le \dims, \\
            &\abs{b_{\ell}} \le \mathcal{O}(\eta).
        \end{align*}
        Hence, after first phase of training, the network's intermediate outputs $y^{(1)}$ and $y^{(2)}$ (from \cref{eq:trained_network}) are given as
        \begin{align*}
            y^{(1)} &= p_0 +  \sin(\langle \vw_{1}, \vx \rangle + b_{1}) =  \frac{\sqrt{3}}{2}  + \frac{\sqrt{3}}{2} x_1 + \mathcal{O}(\eta), \\
            y^{(2)} &= y^{(1)} + \sin(\langle \vw_{2}, \vx \rangle + y^{(1)} + p_0 + b_{2}) 
            \\&=  \frac{\sqrt{3}}{2}  + \frac{\sqrt{3}}{2} x_1 + \sin( \langle \vw_{2}, \vx \rangle + \sqrt{3}/2 x_1 + \sqrt{3}/2 + b_2 )  + \mathcal{O}(\eta).
        \end{align*}

        \item In the second phase, after $\Theta(1/\eta)$ steps, we show in \cref{lem:phase2} that the weights and the biases $\vw_2, b_2$ satisfy the following conditions. 
        \begin{align*}
            &\abs{w_{2, 1} + \sqrt{3}/2 - \frac{\pi}{2}} \le \mathcal{O}(\eta), \\
            &\abs{w_{2, 2} - \frac{\pi}{2}} \le \mathcal{O}(\eta), \\
            &\abs{b_2 + p_0 - \frac{\pi}{2}} \le \mathcal{O}(\eta), \\
            &\abs{w_{2, j}} \le \mathcal{O}(\eta), \text{ for } 3 \le j \le \dims.
        \end{align*}

        Thus, at the end of the second phase, the network's output $y := y^{(2)}$ (from \cref{eq:trained_network}) is given as
        \begin{align*}
            y^{(2)} &= y^{(1)} + \sin(\langle \vw_{2}, \vx \rangle + y^{(1)} + p_0 + b_{2}) \\&
            = \frac{\sqrt{3}}{2} + \frac{\sqrt{3}}{2} x_1 + \sin( \langle \vw_{2}, \vx \rangle + \sqrt{3}/2 x_1 + \sqrt{3}/2 + b_2 )  + \mathcal{O}(\eta) \\&
            =  \frac{\sqrt{3}}{2} + \frac{\sqrt{3}}{2} x_1 + \sin( \frac{\pi}{2} x_1 + \frac{\pi}{2} x_2 + \frac{\pi}{2} ) + \mathcal{O}(\eta) =  \frac{\sqrt{3}}{2} + \frac{\sqrt{3}}{2} x_1  - x_1 x_2 + \mathcal{O}(\eta).
        \end{align*}
    \end{itemize}
\end{proof}

Furthermore, we can show that each layer in itself isn't expressive enough to represent the true label function.

\begin{theorem}\label{thm:layerdrop_expressivity}
    For each layer when trained independently, the best $\vw_{\ell}, \vb_{\ell}$ parameters will have an $\Omega(1)$ error.
\end{theorem}

\begin{proof}
    The proof of the theorem is simple. We argue for the first layer (same argument holds for the second layer). We simply calculate the projection of the function learned by $\vw_1, \vb_1$ into the components of the true label function. Suppose the function define dby them is given by $\alpha x_1 + \beta x_1 x_2$

    Then, using orthogonality of the basis polynomials $x_1$ and $x_1 x_2$, we have
    \begin{align*}
        \alpha =&\mathbb{E}_{\vx \in \{\pm1\}^d} \sin(\langle \vw_1, \vx \rangle + b) x_1 \\&
        = \mathbb{E}_{\vx \in \{\pm1\}^d} \sin(\pi/2 - \langle \vw_1, \vx \rangle -  b) x_1 \\&
        = \sin(b) (\prod_{i \ne 1} \cos w_i) \sin (w_1). \tag*{Using \cref{Simplicy_cosine_cor}}
    \end{align*}
    
    \begin{align*}
        \beta = &\mathbb{E}_{\vx \in \{\pm1\}^d} \sin(\langle \vw_1, \vx \rangle + b) x_1 x_2 \\&
        = \mathbb{E}_{\vx \in \{\pm1\}^d} \sin(\pi/2 - \langle \vw_1, \vx \rangle -  b) x_1 x_2 \\&
        = -\cos(b) (\prod_{i \ne 1, 2} \cos w_i) \sin (w_1) \sin (w_2). \tag*{Using \cref{Simplicy_cosine_cor}}
    \end{align*}

    Adding/Subtracting these terms, we get
    \begin{align*}
        \alpha - \beta &= \sin (b+w_2) (\prod_{i \ne 1, 2} \cos w_i) ,\\
        \alpha + \beta &= \sin (b-w_2) (\prod_{i \ne 1, 2} \cos w_i) ,\\
        -\alpha + \beta &= \sin (-b+w_2) (\prod_{i \ne 1, 2} \cos w_i), \\
        -\alpha - \beta &= \sin (-b-w_2) (\prod_{i \ne 1, 2} \cos w_i)
    \end{align*}
    Thus, we must be restricted by the set $|\alpha|, |\beta| < 1, |\alpha| + |\beta| < 1$. As the coefficients of the true label function don't lie in this set, the proof follows.
\end{proof}

\subsection{Analysis of the different phases}

Initial position bias only training linearly increases $p_0$ to $\sqrt{3}/2$.
\begin{restatable}{lemma}{biasonly}[Position bias only training]\label{lem:bias_only}
For any learning rate $\eta>0$, after $\Theta(1/\eta)$ steps, $p_0$ satisfies
\begin{align*} 
    \abs{p_0 - \sqrt{3}/2} \le \mathcal{O}(\eta).
\end{align*}
\end{restatable}

The first lemma analyses the behavior of the weights during phase 1.
\begin{restatable}{lemma}{raptrphaseone}[Phase 1 of \method{}]\label{lem:phase1}
For a small enough learning rate $\eta < \frac{1}{poly(\dims)},$ with position bias $p_0$ at $\sqrt{3}/2 + \mathcal{O}(\eta)$, and under parameter   from all $0$s, after $\Theta(1/\eta)$ steps, for any layer $\ell \in \{1, 2\}$, the weights and the biases satisfy the following conditions. 
\begin{align*} 
    &\abs{w_{\ell, 1} - \frac{\pi}{3}} \le \mathcal{O}(\eta), \\
    &\abs{w_{\ell, j}} \le \mathcal{O}(\eta), \text{ for all } 2 \le j \le \dims, \\
    &\abs{b_{\ell}} \le \mathcal{O}(\eta).
\end{align*}
\end{restatable}

The second lemma analyses the behavior of the second layer weights in phase 2.
\begin{restatable}{lemma}{raptrphasetwo}[Phase 2 of \method{}]\label{lem:phase2}
For a small enough learning rate $\eta < \frac{1}{poly(\dims)},$ starting from the weights and biases reached by $\Theta (1/\eta)$ steps of position bias only and first phase of \method{} training (\cref{lem:phase1,lem:bias_only}), after  $\Theta (1/\eta)$ steps of second phase of \method{} training, the weights and biases $\vw_2$ and $b_2$ satisfy the following conditions.
    \begin{align*}
    &\abs{w_{2, 1} + \sqrt{3}/2 - \frac{\pi}{2}} \le \mathcal{O}(\eta), \\
    &\abs{w_{2, 2} - \frac{\pi}{2}} \le \mathcal{O}(\eta), \\
    &\abs{b_2 + \sqrt{3}/2  - \frac{\pi}{2}} \le \mathcal{O}(\eta), \\
    &\abs{w_{2, j}} \le \mathcal{O}(\eta), \text{ for } 3 \le j \le \dims.
    \end{align*}
\end{restatable}

\subsubsection{Proof of \cref{lem:bias_only}}

\biasonly*

\begin{proof}
Under all $0$s initialization, the output of both of the layers is $0$. Thus, the gradient of $p_0$ is given by
\begin{align*}
	\mathbb{E}_{\vx \sim \{-1, +1\}^{\dims}} \text{  } \left(p_0 - \frac{\sqrt{3}}{2} - \frac{\sqrt{3}}{2} x_1 + x_1 x_2 \right) 
	= p_0 -  \frac{\sqrt{3}}{2}.
\end{align*}
Hence,  $p_0$ increases up until it reaches $\sqrt{3}/2 + \mathcal{O}(\eta)$.
\end{proof}

\subsubsection{Proof of \cref{lem:phase1}}

\raptrphaseone*

\begin{proof}

Similar to the proof of  \cref{lem:bias_only}, we can show that the position bias $p_0$ stays $\mathcal{O}(\eta)$ close to $\sqrt{3}/2$. Thus, for simplicity of exposition, we simply assume it without formally proving so.

With the first stage of \method{}, we only keep one of the two layers during training. Hence, both the layer parameters train similarly during training. Without loss of generality, we discuss for a single layer weights $\ell \in \{1, 2\}$.

\paragraph{General formulation for gradient w.r.t weights:} The population gradients of $\vw_{\ell}$ is given by
\begin{align*}
    &\mathbb{E}_{\vx \sim \{-1, +1\}^{\dims}} \text{  } \left(p_0 + \sin(\langle \vw_{\ell}, \vx \rangle + b_{\ell}) - \frac{\sqrt{3}}{2} - \frac{\sqrt{3}}{2} x_1 + x_1 x_2 \right) \cdot \cos(\langle \vw_{\ell}, \vx \rangle + b_{\ell}) \cdot \vx \\&
    = \mathbb{E}_{\vx \sim \{-1, +1\}^{\dims}} \text{  } \sin(\langle \vw_{\ell}, \vx \rangle + b_{\ell})  \cdot \cos(\langle \vw_{\ell}, \vx \rangle + b_{\ell}) \cdot \vx - \mathbb{E}_{\vx \sim \{-1, +1\}^{\dims}} \text{  } \left( \frac{\sqrt{3}}{2} x_1 - x_1 x_2 \right) \cdot \cos(\langle \vw_{\ell}, \vx \rangle + b_{\ell}) \cdot \vx  \\&
    + \mathbb{E}_{\vx \sim \{-1, +1\}^{\dims}} \left(p_0 - \frac{\sqrt{3}}{2} \right) \cdot  \cos(\langle \vw_{\ell}, \vx \rangle + b_{\ell}) \cdot \vx .
\end{align*}

We will argue about the contributions of the three terms separately.

\begin{lemma}
[Contribution of Term 1]\label{lem:term1_contrib} For any coordinate $j \in [d]$,
\begin{align*}
    &\mathbb{E}_{\vx \sim \{-1, +1\}^{\dims}} \text{  } \sin(\langle \vw_{\ell}, \vx \rangle + b_{\ell})  \cdot \cos(\langle \vw_{\ell}, \vx \rangle + b_{\ell}) x_j \\&
    = \frac{1}{2} \cos (2b_{\ell})  \left( \prod_{i=1 \to \dims, i \ne j} \cos(2 w_{\ell, i}) \right) \sin (2 w_{\ell, j}).
\end{align*}
\end{lemma}

\begin{proof}
Term 1 in the gradients is given as 
\begin{align*}
    &\mathbb{E}_{\vx \sim \{-1, +1\}^{\dims}} \text{  } \sin(\langle \vw_{\ell}, \vx \rangle + b_{\ell})  \cdot \cos(\langle \vw_{\ell}, \vx \rangle + b_{\ell}) \cdot \vx \\&
    =  \frac{1}{2} \mathbb{E}_{\vx \sim \{-1, +1\}^{\dims}} \text{  } \sin(2\langle \vw_{\ell}, \vx \rangle + 2 b_{\ell}) \cdot \vx.
\end{align*}

Consider a coordinate $j \in [\dims]$. The gradient concerning the coordinate is given as
\begin{align*}
    &\frac{1}{2} \mathbb{E}_{\vx \sim \{-1, +1\}^{\dims}} \text{  } \sin(2\langle \vw_{\ell}, \vx \rangle + 2 b_{\ell}) x_j \\&
    = \frac{1}{2} \mathbb{E}_{\vx \sim \{-1, +1\}^{\dims}} \text{  } \sin \left(2\sum_{i=1 \to \dims, i \ne j} w_{\ell, i} x_i +  2w_{\ell, j} x_j + 2b_{\ell}\right) x_j \\&
    = \frac{1}{2} \cos (2b_{\ell})  \left( \prod_{i=1 \to \dims, i \ne j} \cos(2 w_{\ell, i}) \right) \sin (2 w_{\ell, j}),
\end{align*}
using \cref{Simplify_sine}.
\end{proof}

\paragraph{Contribution of Term 2:}
Term 2 in the gradients is given as 
\begin{align*}
    &\mathbb{E}_{\vx \sim \{-1, +1\}^{\dims}} \text{  } \left( \frac{\sqrt{3}}{2} x_1 - x_1 x_2 \right) \cdot \cos(\langle \vw_{\ell}, \vx \rangle + b_{\ell}) \cdot \vx.
\end{align*}
Consider the coordinate $i=1$. The gradient concerning the coordinate is given as
\begin{align*}
    &\mathbb{E}_{\vx \sim \{-1, +1\}^{\dims}} \text{  } \left( \frac{\sqrt{3}}{2} x_1 - x_1 x_2 \right) \cdot \cos(\langle \vw_{\ell}, \vx \rangle + b_{\ell}) x_1 \\&
    = \frac{\sqrt{3}}{2}  \mathbb{E}_{\vx \sim \{-1, +1\}^{\dims}} \cos(\langle \vw_{\ell}, \vx \rangle + b_{\ell}) - \mathbb{E}_{\vx \sim \{-1, +1\}^{\dims}} x_2 \cos(\langle \vw_{\ell}, \vx \rangle + b_{\ell}) \tag*{Replace $x_1^2$ by $1$} \\&
    = \frac{\sqrt{3}}{2}  \mathbb{E}_{\vx \sim \{-1, +1\}^{\dims}} \cos(\langle \vw_{\ell}, \vx \rangle + b_{\ell}) + \sin (b_{\ell}) \left(\prod_{i=1 \to \dims, i \ne 2} \cos(w_{\ell, i}) \right) \sin(w_{\ell, 2}) \tag*{Simplified second term with \cref{Simplicy_cosine}} \\&
    = \frac{\sqrt{3}}{2}  \cos(b_{\ell}) \left(\prod_{i=1 \to \dims} \cos(w_{\ell, i}) \right) + \sin (b_{\ell}) \left(\prod_{i=1 \to \dims, i \ne 2} \cos(w_{\ell, i}) \right) \sin(w_{\ell, 2}). \tag*{Simplified first term with \cref{Simplicy_cosine_cor}}
\end{align*}

Consider the coordinate $i=2$. The gradient concerning the coordinate is given as
\begin{align*}
    &\mathbb{E}_{\vx \sim \{-1, +1\}^{\dims}} \text{  } \left( \frac{\sqrt{3}}{2} x_1 - x_1 x_2 \right) \cdot \cos(\langle \vw_{\ell}, \vx \rangle + b_{\ell}) x_2 \\&
    = \frac{\sqrt{3}}{2}  \mathbb{E}_{\vx \sim \{-1, +1\}^{\dims}} \cos(\langle \vw_{\ell}, \vx \rangle + b_{\ell}) x_1 x_2 - \mathbb{E}_{\vx \sim \{-1, +1\}^{\dims}} x_1 \cos(\langle \vw_{\ell}, \vx \rangle + b_{\ell}) \tag*{Replace $x_2^2$ by $1$}  \\&
    = -\frac{\sqrt{3}}{2} \cos(b_{\ell}) \left( \prod_{i=1 \to \dims, i \notin \{1, 2\}} \cos(w_{\ell, i}) \right) \left( \prod_{j \in \{1, 2\}} \sin(w_{\ell, j}) \right) + \sin(b_{\ell}) \left( \prod_{i=1 \to \dims, i \ne 1} \cos(w_{\ell, i}) \right)  \sin(w_{\ell,1}) .
    \tag*{Using \cref{Simplicy_cosine_cor}}
\end{align*}

For any other coordinate $t \notin \{1, 2\}$, we have
\begin{align*}
    &\mathbb{E}_{\vx \sim \{-1, +1\}^{\dims}} \text{  } \left( \frac{\sqrt{3}}{2} x_1 - x_1 x_2 \right) \cdot \cos(\langle \vw_{\ell}, \vx \rangle + b_{\ell}) x_t \\&
    = \frac{\sqrt{3}}{2}  \mathbb{E}_{\vx \sim \{-1, +1\}^{\dims}} \cos(\langle \vw_{\ell}, \vx \rangle + b_{\ell}) x_1 x_t - \mathbb{E}_{\vx \sim \{-1, +1\}^{\dims}} x_1 x_2 x_t \cos(\langle \vw_{\ell}, \vx \rangle + b_{\ell})  \\&
    = -\frac{\sqrt{3}}{2} \cos(b_{\ell}) \left( \prod_{i=1 \to \dims, i \notin \{1, t\}} \cos(w_{\ell, i}) \right) \left( \prod_{j \in \{1, t\}} \sin(w_{\ell, j}) \right) \\&- \sin(b_{\ell}) \left( \prod_{i=1 \to \dims, i \notin \{1, 2, t\} } \cos(w_{\ell, i}) \right)  \left( \prod_{j \in \{1, 2, t\}} \sin(w_{\ell, j}) \right) .
    \tag*{Using \cref{Simplicy_cosine_cor}}
\end{align*}

\paragraph{Contribution of Term 3:} 
For any coordinate $j \in [\dims]$,
\begin{align*}
&\mathbb{E}_{\vx \sim \{-1, +1\}^{\dims}} \text{ } \left( p_0 - \frac{\sqrt{3}}{2} \right) \cdot \cos( \langle \vw_{\ell}, \vx \rangle + b_{\ell} ) x_j
\\&= -\left( p_0 - \frac{\sqrt{3}}{2} \right) \sin(b_{\ell}) \left( \prod_{i=1\to d, i \ne j} \cos(w_{\ell, i}) \right) \sin(w_{\ell, j}) \\&= \mathcal{O}(\eta) ,
\end{align*}
where the pre-final step follows from using \cref{Simplicy_cosine_cor} and the final step follows since  $\abs{p_0 - \frac{\sqrt{3}}{2}} = \mathcal{O}(\eta)$ from the initial bias training phase.

Thus, the combination of the 3 terms gives
\begin{align}
   \nabla w_{\ell, 1} &= 
    - \frac{\sqrt{3}}{2}  \cos(b_{\ell}) \left(\prod_{i=1 \to \dims} \cos(w_{\ell, i}) \right)   - \sin (b_{\ell}) \left(\prod_{i=1 \to \dims, i \ne 2} \cos(w_{\ell, i}) \right) \sin(w_{\ell, 2}) \nonumber\\&+ \frac{1}{2} \cos (2b_{\ell})  \left( \prod_{i=1 \to \dims, i \ne 1} \cos(2 w_{\ell, i}) \right) \sin (2 w_{\ell, 1})  + \mathcal{O}(\eta) \label{eq:gradient_w1_formulation}
\end{align}

\begin{align}
   \nabla w_{\ell, 2} &= 
   \frac{\sqrt{3}}{2} \cos(b_{\ell}) \left( \prod_{i=1 \to \dims, i \notin \{1, 2\}} \cos(w_{\ell, i}) \right) \left( \prod_{j \in \{1, 2\}} \sin(w_{\ell, j}) \right) - \sin(b_{\ell}) \left( \prod_{i=1 \to \dims, i \ne 1} \cos(w_{\ell, i}) \right)  \sin(w_{\ell, 1}) \nonumber\\&+ \frac{1}{2} \cos (2b_{\ell})  \left( \prod_{i=1 \to \dims, i \ne 2} \cos(2 w_{\ell, i}) \right) \sin (2 w_{\ell, 2})   + \mathcal{O}(\eta) \label{eq:gradient_w2_formulation}
\end{align}

For $t \notin \{1, 2\}$,
\begin{align}
   \nabla w_{\ell, t} &= 
     \frac{\sqrt{3}}{2} \cos(b_{\ell}) \left( \prod_{i=1 \to \dims, i \notin \{1, t\}} \cos(w_{\ell, i}) \right) \left( \prod_{j \in \{1, t\}} \sin(w_{\ell, j}) \right) \nonumber \\& + \sin(b_{\ell}) \left( \prod_{i=1 \to \dims, i \notin \{1, 2, t\} } \cos(w_{\ell, i}) \right)  \left( \prod_{j \in \{1, 2, t\}} \sin(w_{\ell, j}) \right) \nonumber \\&+ \frac{1}{2} \cos (2b_{\ell})  \left( \prod_{i=1 \to \dims, i \ne t} \cos(2 w_{\ell, i}) \right) \sin (2 w_{\ell, t}) + \mathcal{O}(\eta)\label{eq:gradient_wt_formulation}
\end{align}

\paragraph{General formulation for gradient w.r.t bias:}

Following a similar approach as the weights, the population gradients of $b_{\ell}$ is given by
\begin{align*}
    &\mathbb{E}_{\vx \sim \{-1, +1\}^{\dims}} \text{  } \left(\sin(\langle \vw_{\ell}, \vx \rangle + b_{\ell}) - \frac{\sqrt{3}}{2} - \frac{\sqrt{3}}{2} x_1 + x_1 x_2 \right) \cdot \cos(\langle \vw_{\ell}, \vx \rangle + b_{\ell}) \\&
    = \mathbb{E}_{\vx \sim \{-1, +1\}^{\dims}} \text{  } \sin(\langle \vw_{\ell}, \vx \rangle + b_{\ell}) \cos(\langle \vw_{\ell}, \vx \rangle + b_{\ell}) \\&- \mathbb{E}_{\vx \sim \{-1, +1\}^{\dims}} \left(\frac{\sqrt{3}}{2} x_1 - x_1 x_2 \right)  \cos(\langle \vw_{\ell}, \vx \rangle + b_{\ell}) \\
    &+ \mathbb{E}_{\vx \sim \{-1, +1\}^{\dims}}  \left(p_0 - \frac{\sqrt{3}}{2} \right) \cos(\langle \vw_{\ell}, \vx \rangle + b_{\ell})
\end{align*}

Following similar approach as above, the terms can be simplified as 
\begin{align}
    &\frac{1}{2} \sin (2b_{\ell})  \left( \prod_{i=1 \to \dims} \cos(2 w_{\ell, i}) \right) \\&+ \frac{\sqrt{3}}{2} \sin(b_{\ell}) \left( \prod_{i=1 \to \dims, i \ne 1 } \cos(w_{\ell, i}) \right)  \sin(w_{\ell, 1}) \nonumber \\&- \cos(b_{\ell}) \left( \prod_{i=1 \to \dims, i \notin \{1, 2\} } \cos(w_{\ell, i}) \right)  \left( \prod_{j \in \{1, 2\}} \sin(w_{\ell, j}) \right) + \mathcal{O}(\eta). \label{eq:gradient_bias_formulation}
\end{align}

\paragraph{Behavior of gradients at initialization:}
Since the coordinates of $\vw_{\ell}$ and biases $\vb_{\ell}$  are initialized from $0$s, we have for all $j$,
\begin{align}
    &\sin (b_{\ell}) = 0, \quad \sin (w_{\ell, j}) = 0, \quad \cos(w_{\ell, j}) = 1, \quad \cos(b_{\ell}) = 1.  \label{eq:sin_cosine_smallval}
\end{align}
Furthermore, recall that $p_0 = \sqrt{3}/2 + \mathcal{O}(\eta)$ after the initial bias training phase.

Using the above, we can simplify the gradients from \cref{eq:gradient_w1_formulation,eq:gradient_w2_formulation,eq:gradient_wt_formulation,eq:gradient_bias_formulation}
\begin{align}
   \nabla w_{\ell, 1} &= 
    - \frac{\sqrt{3}}{2}  + \mathcal{O}(\eta), \label{gradient_w1_init} \\
     \nabla w_{\ell, t} &= \mathcal{O}(\eta),  \quad t \ne 1 \label{gradient_w2_init} \\
     \nabla b_{\ell} &= \mathcal{O}(\eta),  \label{gradient_bias_init}
\end{align}

Hence, we observe three kinds of behavior at initialization.
\begin{enumerate}
    \item First coordinate of the weight grows by atleast
    \begin{align*}
        w_{\ell, 1} \gets w_{\ell, 1} + \frac{\eta \sqrt{3}}{2} + \mathcal{O}(\eta^2).
    \end{align*}
    \item Other coordinates of the weight get only $\mathcal{O}(\eta^2)$ updates.
    \begin{align*}
        w_{\ell, j} \gets w_{\ell, j}  + \mathcal{O}(\eta^2).
    \end{align*}
    \item The magnitude of the bias drops as well.
    \begin{align*}
        b_{\ell} \gets b_{\ell} + \mathcal{O}(\eta^2).
    \end{align*}
\end{enumerate}

\paragraph{Beyond Initialization:} 
Due to increasing magnitude of $w_{\ell, 1}$, the behavior of the gradients changes slightly from \Cref{gradient_bias_init,gradient_w1_init,gradient_w2_init,gradient_bias_init}. However, assuming that the weight coordinates $[2, \dims]$ and the biases are smaller than $\mathcal{O}(\eta)$, we can  still give similar formulations as before using \cref{eq:sin_cosine_smallval} using the following inequalities for all $j \ne 1$.
\begin{align*}
    &\sin (b_{\ell}) = \mathcal{O}(\eta), \quad \sin (w_{\ell, j}) = 1-\mathcal{O}(\eta).
\end{align*}

We state them directly.

\begin{align}
   \nabla w_{\ell, 1} &= 
    - \frac{\sqrt{3}}{2}  \cos(b_{\ell}) \left(\prod_{i=1 \to \dims} \cos(w_{\ell, i}) \right)   - \sin (b_{\ell}) \left(\prod_{i=1 \to \dims, i \ne 2} \cos(w_{\ell, i}) \right) \sin(w_{\ell, 2}) \nonumber\\&+ \frac{1}{2} \cos (2b_{\ell})  \left( \prod_{i=1 \to \dims, i \ne 1} \cos(2 w_{\ell, i}) \right) \sin (2 w_{\ell, 1}) + \mathcal{O}(\eta) \nonumber
   \\&= \left( 1 - \mathcal{O}(\eta)  \right)^{d} \left( \sin(w_{\ell, 1}) - \frac{\sqrt{3}}{2} \right) + \mathcal{O} (\eta) \le \frac{1}{2} \left( \sin(w_{\ell, 1}) - \frac{\sqrt{3}}{2} \right) + \mathcal{O}(\eta),
   \label{gradient_w1_beyondinit}
\end{align}
for $\eta \le 1/poly(d).$

\begin{align}
   \nabla w_{\ell, 2} &= 
   \frac{\sqrt{3}}{2} \cos(b_{\ell}) \left( \prod_{i=1 \to \dims, i \notin \{1, 2\}} \cos(w_{\ell, i}) \right) \left( \prod_{j \in \{1, 2\}} \sin(w_{\ell, j}) \right) - \sin(b_{\ell}) \left( \prod_{i=1 \to \dims, i \ne 1} \cos(w_{\ell, i}) \right)  \sin(w_{\ell, 1}) \nonumber\\&+ \frac{1}{2} \cos (2b_{\ell})  \left( \prod_{i=1 \to \dims, i \ne 2} \cos(2 w_{\ell, i}) \right) \sin (2 w_{\ell, 2}) + \mathcal{O}(\eta) 
   \nonumber\\&= \mathcal{O}(\eta). \label{gradient_w2_beyondinit}
\end{align}

For any $t \notin \{1, 2\}$,
\begin{align}
   \nabla w_{\ell, j} &=  \frac{\sqrt{3}}{2} \cos(b_{\ell}) \left( \prod_{i=1 \to \dims, i \notin \{1, t\}} \cos(w_{\ell, i}) \right) \left( \prod_{j \in \{1, t\}} \sin(w_{\ell, j}) \right) \nonumber \\&+ \sin(b_{\ell}) \left( \prod_{i=1 \to \dims, i \notin \{1, 2, t\} } \cos(w_{\ell, i}) \right)  \left( \prod_{j \in \{1, 2, t\}} \sin(w_{\ell, j}) \right) \nonumber \\&+ \frac{1}{2} \cos (2b_{\ell})  \left( \prod_{i=1 \to \dims, i \ne t} \cos(2 w_{\ell, i}) \right) \sin (2 w_{\ell, t}) + \mathcal{O}(\eta),
   \nonumber\\&= 
    \mathcal{O}(\eta).  \label{gradient_wj_beyondinit}
\end{align}

And finally for the bias,
\begin{align}
   \nabla b_{\ell} &= \frac{1}{2} \sin (2b_{\ell})  \left( \prod_{i=1 \to \dims} \cos(2 w_{\ell, i}) \right) + \frac{\sqrt{3}}{2} \sin(b_{\ell}) \left( \prod_{i=1 \to \dims, i \ne 1 } \cos(w_{\ell, i}) \right)  \sin(w_{\ell, 1}) \nonumber \\&- \cos(b_{\ell}) \left( \prod_{i=1 \to \dims, i \notin \{1, 2\} } \cos(w_{\ell, i}) \right)  \left( \prod_{j \in \{1, 2\}} \sin(w_{\ell, j}) \right) + \mathcal{O}(\eta) \\&= \mathcal{O}(\eta). \label{gradient_bias_beyondinit}
\end{align}

Thus, we observe three properties.
\begin{itemize}
    \item First coordinate of the weight grows until $\sin(w_{\ell, 1}) = \frac{\sqrt{3}}{2}$ or $w_{\ell, 1}$ reaches $\pi/3$,
    \begin{align*}
        w_{\ell, 1} \gets w_{\ell, 1} + \frac{\eta}{2} \left( \frac{\sqrt{3}}{2} - \sin(w_{\ell, 1}) \right) + \mathcal{O}(\eta^2).
    \end{align*}
    Thus, in $\Theta(1/\eta)$ steps, $w_{\ell, 1}$ can reach arbitrarily close to $\pi/3$.

    \item Any other coordinate $t \ne 1$ still receives $\mathcal{O}(\eta^2)$ updates as 
    \begin{align*}
        w_{\ell, t} \gets w_{\ell, t} + \mathcal{O}(\eta^2).
    \end{align*}
	Hence, in  $\Theta(1/\eta)$ steps, $w_{\ell, t}$ can only reach $\mathcal{O}(\eta)$ magnitude.

\item $b_{\ell}$ also receives $\mathcal{O}(\eta^2)$ updates as 
    \begin{align*}
        b_{\ell} \gets b_{\ell} + \mathcal{O}(\eta^2).
    \end{align*}
	Hence, in  $\Theta(1/\eta)$ steps, $b_{\ell}$ can only reach $\mathcal{O}(\eta)$ magnitude.

\end{itemize}

\end{proof}

\subsubsection{Proof of \cref{lem:phase2}}

\raptrphasetwo*

\begin{proof}
    The proof follows along similar lines as \cref{lem:phase1}. 

    \paragraph{First layer output:}
    At the end of the first phase training, the weights $\vw_1, \vw_2$ and the biases look as follows. For $\ell \in \{1, 2\}$,
    \begin{align*}
        \abs{w_{\ell, 1} - \frac{\pi}{3}} \le \mathcal{O}(\eta),
       \abs{ w_{\ell, j} } \le \mathcal{O}(\eta), \text{ for all } j \ge 2,
       \abs{ b_{\ell} } \le \mathcal{O}(\eta),
    \end{align*}
    and $\abs{p_0 - \frac{\sqrt{3}}{2}} \le \mathcal{O}(\eta).$

    Hence, using a Taylor expansion, the output of the first layer is given as (for simplicity, we denote it as $o^{(1)}$)
    \begin{align}
        o^{(1)} := \sin(\langle \vw_1, \vx \rangle + b_1) = \frac{\sqrt{3}}{2} x_1 + \mathcal{O}(\eta). \label{eq:o1_formulation}
    \end{align}

    \paragraph{Second layer output:}

    The output of the second layer is now given by (for simplicity, we denote it as $o^{(2)}$)
    \begin{align*}
        & o^{(2)} := \sin \left(\langle \vw_2, \vx \rangle + o^{(1)} + p_0 +  b_2 \right) \\&= 
        \sin \left(\langle \vw_2, \vx \rangle + \sin(\langle \vw_1, \vx \rangle + b_1) + p_0 +  b_2 \right) \\&
        = \sin \left(\langle \vw_2, \vx \rangle + \left(\frac{\sqrt{3}}{2} x_1 + \mathcal{O}(\eta)\right)  + (\sqrt{3}/2 + \mathcal{O}(\eta))  +  b_2 \right) \\&
        = \sin \left(\langle \vw_2, \vx \rangle + \frac{\sqrt{3}}{2} x_1  + \sqrt{3}/2  +  b_2 \right) + \mathcal{O}(\eta) \tag*{With Taylor expansion} \\&
        = \sin \left( ( w_{2, 1} + \sqrt{3}/2 ) x_1 + \sum_{j=2}^{d} w_{2, j} x_j + \sqrt{3}/2 + b_2 \right) + \mathcal{O}(\eta).
    \end{align*}
     
    For brevity, we will use new notations to represent the coefficients of $x_1, \cdots, x_d$ in the above formulation.

    \begin{align*}
        \Tilde{w}_{1} &= w_{2, 1} + \sqrt{3}/2,\\
        \Tilde{w}_{j} &= w_{2, j}, \text{ for } j \ge 2,\\
        \Tilde{b} &= \sqrt{3}/2 + b_2.
    \end{align*}
    At initialization, their corresponding values are
    \begin{align}
        \Tilde{w}_{1} &= \frac{\pi}{3} + \sqrt{3}/2 + \mathcal{O}(\eta), \nonumber \\
        \Tilde{w}_{j} &=  \mathcal{O}(\eta), \text{ for all } j \ge 2 \nonumber\\
        \Tilde{b} &= \frac{\sqrt{3}}{2} + \mathcal{O}(\eta). \label{eq:tilde_init_values}
    \end{align}

    Then, the above formulation is given as
    \begin{align}
        o^{(2)} = \sin \left( \langle \Tilde{\vw}, \vx \rangle + \Tilde{b} \right) + \mathcal{O}(\eta). \label{eq:o2_formulation}
    \end{align}

    \paragraph{General formulation for population gradients:}
    The gradients w.r.t. the weight $\vw_2$ and $b_2$ are given as

    \begin{align*}
        \nabla \vw_2 &= \mathbb{E}_{\vx \sim \{-1, +1\}^{d}} \left(o^{(1)} + o^{(2)} - \frac{\sqrt{3}}{2} x_1 + x_1 x_2\right) \cos \left( \langle \Tilde{\vw}, \vx \rangle + \Tilde{b} \right) \vx + \mathcal{O}(\eta). \\
        \nabla b_2 &= \mathbb{E}_{\vx \sim \{-1, +1\}^{d}} \left(o^{(1)} + o^{(2)} - \frac{\sqrt{3}}{2} x_1 + x_1 x_2\right) \cos \left( \langle \Tilde{\vw}, \vx \rangle + \Tilde{b} \right) + \mathcal{O}(\eta).
    \end{align*}
    
    We will consider the first two terms in the above gradient formulations and add the error term later.

    First of all, we observe that
    \begin{align*}
        &o^{(1)} + o^{(2)} - \frac{\sqrt{3}}{2} x_1 + x_1 x_2 \\& = \sin \left( \langle \Tilde{\vw}, \vx \rangle + \Tilde{b} \right) + x_1 x_2 + \mathcal{O}(\eta).
    \end{align*}

    We have the gradients w.r.t. $\vw_2$ as
    \begin{align*}
        \nabla \vw_2 &= \mathbb{E}_{\vx \sim \{-1, +1\}^{d}} \sin \left( \langle \Tilde{\vw}, \vx \rangle + \Tilde{b} \right) \cos \left( \langle \Tilde{\vw}, \vx \rangle + \Tilde{b} \right) \vx \\&
        + \mathbb{E}_{\vx \sim \{-1, +1\}^{d}}  x_1 x_2 \cos \left( \langle \Tilde{\vw}, \vx \rangle + \Tilde{b} \right) \vx \\&
        + \mathcal{O}(\eta).
    \end{align*}
    
    We treat the two terms separately.
    \subparagraph{Contribution of Term 1}: Using a similar strategy as \cref{lem:term1_contrib}, for any coordinate $t \in [d]$,
    \begin{align*}
        &\mathbb{E}_{\vx \sim \{-1, +1\}^{d}} \sin \left( \langle \Tilde{\vw}, \vx \rangle + \Tilde{b}\right) \cos \left( \langle \Tilde{\vw}, \vx \rangle + \Tilde{b} \right) x_t \\&
        = \frac{1}{2} \cos (2\Tilde{b})  \left( \prod_{i=1 \to \dims, i \ne t} \cos(2 \Tilde{w}_i) \right) \sin (2 \Tilde{w}_t).
    \end{align*}

     \subparagraph{Contribution of Term 2:}
     For coordinate $t=2$,
     \begin{align*}
         &\mathbb{E}_{\vx \sim \{-1, +1\}^{d}}  x_1 x_2 \cos \left(\langle \Tilde{\vw}, \vx \rangle + \Tilde{b}\right) x_2 \\& = \mathbb{E}_{\vx \sim \{-1, +1\}^{d}} \cos \left(\langle \Tilde{\vw}, \vx \rangle + \Tilde{b}\right) x_1 \tag*{Replace $x_2^2$ by 1} \\&
         =  - \sin (\Tilde{b}) \left(\prod_{i=1 \to \dims, i \ne 1} \cos(\Tilde{w}_{i}) \right) \sin(\Tilde{w}_1). \tag*{Using \cref{Simplicy_cosine_cor}}
     \end{align*}

     For coordinate $t=1$,
     \begin{align*}
         &\mathbb{E}_{\vx \sim \{-1, +1\}^{d}}  x_1 x_2 \cos \left(\langle \Tilde{\vw}, \vx \rangle + \Tilde{b} \right) x_1 \\& = \mathbb{E}_{\vx \sim \{-1, +1\}^{d}} \cos \left(\langle \Tilde{\vw}, \vx \rangle + \Tilde{b} \right) x_2 \tag*{Replace $x_1^2$ by 1} \\&
         =  - \sin (\Tilde{b}) \left(\prod_{i=1 \to \dims, i \ne 2} \cos(\Tilde{w}_{i}) \right) \sin(\Tilde{w}_2). \tag*{Using \cref{Simplicy_cosine_cor}}
     \end{align*}

    For any other coordinate $t \notin \{1, 2\}$,
    \begin{align*}
         &\mathbb{E}_{\vx \sim \{-1, +1\}^{d}}   \cos \left( \langle \Tilde{\vw}, \vx \rangle + \Tilde{b} \right) x_1 x_2 x_t  \\&
         =  \sin (\Tilde{b}) \left(\prod_{i=1 \to \dims, i \notin \{1,2,t\}} \cos(\Tilde{w}_{i}) \right) \left( \prod_{i \notin \{1,2,t\}} \sin(\Tilde{w}_t) \right). \tag*{Using \cref{Simplicy_cosine_cor}}
     \end{align*}

     \paragraph{Combination of all terms:}
     For coordinate $t=1$,
     \begin{align}
         \nabla w_{2, 1} = &\frac{1}{2} \cos (2\Tilde{b})  \left( \prod_{i=1 \to \dims, i \ne 1} \cos(2 \Tilde{w}_i) \right) \sin (2 \Tilde{w}_1) - \sin (\Tilde{b}) \left(\prod_{i=1 \to \dims, i \ne 2} \cos(\Tilde{w}_{i}) \right) \sin(\Tilde{w}_2) + \mathcal{O}(\eta). \label{eq:update_w21}
     \end{align}

    For coordinate $t=2$,
    \begin{align}
         \nabla w_{2, 2} & = \frac{1}{2} \cos (2\Tilde{b})  \left( \prod_{i=1 \to \dims, i \ne 2} \cos(2 \Tilde{w}_i) \right) \sin(2\Tilde{w}_2) - \sin (\Tilde{b}) \left(\prod_{i=1 \to \dims, i \ne 1} \cos(\Tilde{w}_{i}) \right) \sin(\Tilde{w}_1) + \mathcal{O}(\eta).\label{eq:update_w22}
     \end{align}

    For any other coordinate $t \notin \{1, 2\}$,
    \begin{align}
         \nabla w_{2, t} = \frac{1}{2} \cos (2\Tilde{b})  \left( \prod_{i=1 \to \dims, i \ne t} \cos(2 \Tilde{w}_i) \right) \sin (2 \Tilde{w}_t) - \sin (\Tilde{b}) \left(\prod_{i=1 \to \dims, i \notin \{1,2,t\}} \cos(\Tilde{w}_{i}) \right) \left( \prod_{i \in \{1,2,t\}} \sin(\Tilde{w}_i) \right) + \mathcal{O}(\eta).\label{eq:update_w2t}
     \end{align}

    We have similar computation for the gradient of the bias $b_2$, which we report directly.
    \begin{align}
    \nabla b_2 &= \mathbb{E}_{\vx \sim \{-1, +1\}^{d}} \left(o^{(1)} + o^{(2)} - \frac{\sqrt{3}}{2} x_1 + x_1 x_2\right) \cos \left( \langle \Tilde{\vw}, \vx \rangle + \Tilde{b} \right) + \mathcal{O}(\eta) \nonumber\\&
    = \frac{1}{2} \sin (2\Tilde{b})  \left( \prod_{i=1 \to \dims} \cos(2 \Tilde{w}_i ) \right)  - \cos(\Tilde{b}) \left(\prod_{i=1 \to \dims, i \notin\{1, 2\}} \cos(\Tilde{w}_{i}) \right) \sin(\Tilde{w}_1) \sin(\Tilde{w}_2) + \mathcal{O}(\eta).\label{eq:update_bias}
    \end{align}

    \subparagraph{At initialization:}
    With $\Tilde{b}$ being initialzied at $\frac{\sqrt{3}}{2} + \mathcal{O}(\eta)$ (from \cref{eq:tilde_init_values}),
    \begin{align*}
        \cos( \Tilde{b}) &= \cos(\sqrt{3}/2) + \mathcal{O}(\eta) > 0 \\
        \cos(2 \Tilde{b}) &= \cos(\sqrt{3}) + \mathcal{O}(\eta) < 0 \\
        \sin( \Tilde{b}) &= \sin(\sqrt{3}/2) + \mathcal{O}(\eta) > 0. \\
        \sin(2\Tilde{b}) &= \sin(\sqrt{3}) + \mathcal{O}(\eta) > 0.
    \end{align*}

    The initial value of $\Tilde{w}_1$ is $\frac{\sqrt{3}}{2} + \frac{\pi}{3} + \mathcal{O}(\eta)$ (from \cref{eq:tilde_init_values}) and hence the values of
    \begin{align*}
        \cos(2\Tilde{w}_1) &= \cos(\sqrt{3} + 2\pi/3) + \mathcal{O}(\eta) < 0 \\
        \sin(2\Tilde{w}_1) &= \sin(\sqrt{3} + 2\pi/3) + \mathcal{O}(\eta) < 0 \\
        \cos(\Tilde{w}_1) &= \cos(\sqrt{3}/2 + \pi/3) + \mathcal{O}(\eta)  < 0 \\
        \sin(\Tilde{w}_1) &= \sin(\sqrt{3}/2 + \pi/3) + \mathcal{O}(\eta) > 0.
    \end{align*}

    The other coordinates $\Tilde{w}_{j}$ are of order $\mathcal{O}(\eta)$. Thus, the behavior of the weights and biases at initialization can be summarized as follows.
    \begin{itemize}
        \item $w_{2, 1}$ gets a decreasing update as both the terms involved in \cref{eq:update_w21} are positive.
        \begin{align*}
           -\eta\nabla w_{2, 1} &=  -\eta\abs{\frac{1}{2} \cos (2\Tilde{b})  \left( \prod_{i=1 \to \dims, i \ne 1} \cos(2 \Tilde{w}_i) \right) \sin (2 \Tilde{w}_1)} - \eta\abs{ \sin (\Tilde{b}) \left(\prod_{i=1 \to \dims, i \ne 2} \cos(\Tilde{w}_{i}) \right) \sin(\Tilde{w}_2) } + \mathcal{O}(\eta^2) \\&
           = -\frac{\eta}{2} \abs{\cos(\sqrt{3})} (1 - \mathcal{O}(2\eta))^{d-1} \abs{\sin(\sqrt{3} + 2\pi/3)} + \mathcal{O}(\eta^2) \le -\frac{\eta}{4} \cos (\sqrt{3}) \sin(\sqrt{3} + 2\pi/3) +   \mathcal{O}(\eta^2),
        \end{align*}
        as $\eta = \mathcal{O}(1/poly(d))$ small.
        
        \item $w_{2, 2}$ gets an increasing update as both the terms involved in \cref{eq:update_w22} are negative. 
        \begin{align*}
            -\eta\nabla w_{2, 2} &= \frac{\eta}{2} \abs{\cos (2\Tilde{b})}  \abs{\left( \prod_{i=1 \to \dims, i \ne 2} \cos(2 \Tilde{w}_i) \right)} \abs{\sin(2\Tilde{w}_2)} + \abs{\sin (\Tilde{b})} \abs{ \left(\prod_{i=1 \to \dims, i \ne 1} \cos(\Tilde{w}_{i}) \right) } \abs{ \sin(\Tilde{w}_1) } + \mathcal{O}(\eta^2) \\
            &= \eta \sin(\sqrt{3})\sin(\sqrt{3}/2 + \pi/3) (1-\mathcal{O}(\eta))^{d-1} + \mathcal{O}(\eta^2) \ge \frac{\eta}{2}\sin(\sqrt{3})\sin(\sqrt{3}/2 + \pi/3) + \mathcal{O}(\eta^2),
        \end{align*}
        as $\eta = \mathcal{O}(1/poly(d))$ small.
        
        \item Other coordinates $w_{2, t}$ for $t \notin \{1, 2\}$ get small gradients as both the involved terms in \cref{eq:update_w2t} depend on how large $\Tilde{w}_t := w_{2, t} = \mathcal{O}(\eta)$ is. 
        \begin{align*}
            -\eta \nabla w_{2, t} &= -\frac{\eta}{2} \cos (2\Tilde{b})  \left( \prod_{i=1 \to \dims, i \ne t} \cos(2 \Tilde{w}_i) \right) \sin (2 \Tilde{w}_t) \\&+ \eta \sin (\Tilde{b}) \left(\prod_{i=1 \to \dims, i \notin \{1,2,t\}} \cos(\Tilde{w}_{i}) \right) \left( \prod_{i \in \{1,2,t\}} \sin(\Tilde{w}_i) \right) + \mathcal{O}(\eta^2) = \mathcal{O}(\eta^2).
        \end{align*}

        \item Bias $b_2$ get an increasing update as both the involved terms in \cref{eq:update_bias} are negative.
        \begin{align*}
            -\eta \nabla b_2 &= \frac{\eta}{2} \abs{ \sin (2\Tilde{b}) }  \abs{\left( \prod_{i=1 \to \dims} \cos(2 \Tilde{w}_i ) \right)}  + \eta \abs{\cos(\Tilde{b})} \abs{\left(\prod_{i=1 \to \dims, i \notin\{1, 2\}} \cos(\Tilde{w}_{i}) \right) \sin(\Tilde{w}_1)} \abs{\sin(\Tilde{w}_2) } + \mathcal{O}(\eta) \\&
            = \frac{\eta}{2} \sin(\sqrt{3}) \abs{\cos(\sqrt{3}+2\pi/3)} (1-\mathcal{O}(\eta))^{d-1}   + \mathcal{O}(\eta) \ge \frac{\eta}{4} \sin(\sqrt{3}) \abs{\cos(\sqrt{3}+2\pi/3)}.
        \end{align*}

    \end{itemize}

    \subparagraph{Beyond initialization:}   
    In fact, we can extend the above argument beyond initialization, where $w_{2, 1}$ decreases and $w_{2, 2}$, $b_2$ increase by $\Theta(\eta)$ updates, up until we reach the conditions
    \begin{align*}
        \abs{\Tilde{w}_1 - \frac{\pi}{2}} &\le \mathcal{O}(\eta), \quad \text{(or)} \quad \abs{w_{2, 1} - \frac{\pi}{2} + \sqrt{3}/2} \le \mathcal{O}(\eta), \\
        \abs{\Tilde{w}_2 - \frac{\pi}{2}} &\le \mathcal{O}(\eta), \quad \text{(or)} \quad \abs{w_{2, 2} - \frac{\pi}{2}} \le \mathcal{O}(\eta) \\
        \abs{\Tilde{b} - \frac{\pi}{2}} &\le  \mathcal{O}(\eta), \quad \text{(or)} \quad \abs{b_2 - \frac{\pi}{2} + \sqrt{3}/2} \le \mathcal{O}(\eta).
    \end{align*}
    Since, we get $\Theta(\eta)$ updates in each step, this condition can be reached in $\Theta(1/\eta)$ steps. 

    The argument is as follows.
    For any $\Tilde{w}_1, \Tilde{w}_2, \Tilde{b}$ satisfying the following ranges,
    \begin{align*}
        \Tilde{w}_{1} &\in (\frac{\pi}{2}, \frac{\sqrt{3}}{2} + \frac{\pi}{3}), \\
        \Tilde{w}_{2} &\in (0, \frac{\pi}{2}), \\
        \Tilde{b} &\in (\sqrt{3}/2, \frac{\pi}{2}),
    \end{align*}
    the following conditions hold true.
    \begin{align*}
        \cos( \Tilde{b}), \text{ } \sin( \Tilde{b}), &\text{ } \sin(2\Tilde{b}) > 0, \quad \cos(2 \Tilde{b}) < 0\\
        \cos(2\Tilde{w}_1), \text{ } \sin(2\Tilde{w}_1), &\text{ } \cos(\Tilde{w}_1) < 0, \quad \sin(\Tilde{w}_1) > 0. \\
         \sin(\Tilde{w}_2), \text{ } \cos(\Tilde{w}_2), &\text{ }  \sin(2\Tilde{w}_2) > 0, \quad \cos(2\Tilde{w}_2) < 0.
    \end{align*}

    We can then show that the terms involved in the update rule of $w_{2, 1}$ (\cref{eq:update_w21}) are positive, implying $-\eta \nabla w_{2, 1} < -\Omega(\eta)$ up until $\Tilde{w}_{1}$ reaches $\mathcal{O}(\eta)$ close to $\frac{\pi}{2}$. Similar argument can be given for the updates of $b_2$ and $w_{2, 2}$ respectively.

    Furthermore, the coordinates $w_{2, t}$ for $t \notin \{1, 2\}$ get $\mathcal{O}(\eta^2)$ updates per step and stay $\mathcal{O}(\eta)$ small for $\Theta(1/\eta)$ steps of training. 
    
\end{proof}

\subsection{Useful Lemmas}

\begin{lemma}\label{Simplify_sine}
    \begin{align*}
     \mathbb{E}_{\vx \sim \{-1, +1\}^{\dims}} \text{  } \sin \left( \langle \vw, \vx \rangle + b\right) x_j = \cos (b)  \left( \prod_{i=1 \to \dims, i \ne j} \cos(w_i) \right) \sin ( w_j).
    \end{align*}
\end{lemma}

\begin{proof}
    \begingroup
\allowdisplaybreaks
    \begin{align*}
     &\mathbb{E}_{\vx \sim \{-1, +1\}^{\dims}} \text{  } \sin \left( \langle \vw, \vx \rangle + b\right) x_j \\&
     = \mathbb{E}_{\vx \sim \{-1, +1\}^{\dims}} \text{  } \sin \left(\sum_{i=1 \to \dims, i \ne j}w_i x_i +  w_j x_j + b\right) x_j \\&
    = \cos (b)  \left( \mathbb{E}_{ \vx \sim \{-1, +1\}^{\dims} } \sin  \left( \sum_{i=1 \to \dims, i \ne j}w_i x_i \right) \cdot \mathbb{E}_{x_j}\cos  \left( w_j x_j \right) x_j \right) \tag*{Equals  $0$ as an odd function} \\&+  \cos (b)  \left(\mathbb{E}_{ \vx \sim \{-1, +1\}^{\dims} } \cos  \left(\sum_{i=1 \to \dims, i \ne j}w_i x_i \right) \cdot \mathbb{E}_{x_j} \sin  \left( w_j x_j \right) x_j \right) \tag*{Equals  $0$ as an odd function} \\&
    +  \sin (b) \left( \mathbb{E}_{ \vx \sim \{-1, +1\}^{\dims} } \cos  \left( \sum_{i=1 \to \dims, i \ne j}w_i x_i \right) \cdot \mathbb{E}_{x_j}\cos  \left( w_j x_j \right) x_j \right) \\&- \sin (b)  \left(\mathbb{E}_{ \vx \sim \{-1, +1\}^{\dims} } \sin  \left(\sum_{i=1 \to \dims, i \ne j}w_i x_i \right) \cdot \mathbb{E}_{x_j} \sin  \left( w_j x_j \right) x_j \right) \tag*{Equals  $0$ as an odd function} \\&
    =  \cos (b)  \left(\mathbb{E}_{ \vx \sim \{-1, +1\}^{\dims} } \cos  \left(\sum_{i=1 \to \dims, i \ne j}w_i x_i \right) \cdot \mathbb{E}_{x_j} \sin  \left( w_j x_j \right) x_j \right) \\&
    = \cos (b)  \left( \prod_{i=1 \to \dims, i \ne j} \cos(w_i) \right) \sin ( w_j) 
    \end{align*}
    \endgroup
\end{proof}

\begin{lemma}\label{Simplicy_cosine}
    \begin{align*}
     &\mathbb{E}_{\vx \sim \{-1, +1\}^{\dims}} \text{  } \cos \left( \langle \vw, \vx \rangle + b\right) x_j
     = -\sin (b) \left(\prod_{i=1 \to \dims, i \ne j} \cos(w_i) \right) \sin(w_j).
    \end{align*}
\end{lemma}

\begin{proof}

    \begingroup
\allowdisplaybreaks
    \begin{align*}
     &\mathbb{E}_{\vx \sim \{-1, +1\}^{\dims}} \text{  } \cos \left( \langle \vw, \vx \rangle + b\right) x_j \\&
     = \mathbb{E}_{\vx \sim \{-1, +1\}^{\dims}} \text{  } \cos \left( \frac{\pi}{2} - \langle \vw, \vx \rangle - b\right) x_j \\&
     = \cos(\pi/2 - b) \left(\prod_{i=1 \to \dims, i \ne j} \cos(-w_i) \right) \sin (- w_j) \tag*{Using \cref{Simplify_sine}} \\&
     = -\sin (b) \left(\prod_{i=1 \to \dims, i \ne j} \cos(w_i) \right) \sin(w_j).
    \end{align*}
    \endgroup
\end{proof}

\begin{corollary}\label{Simplicy_cosine_cor}
    Consider a set $S \subseteq [\dims]$.
    \begin{align*}
        \mathbb{E}_{\vx \sim \{-1, +1\}^{\dims}} \text{  } \cos \left( \langle \vw, \vx \rangle + b\right) \prod_{j \in S} x_i = c_b \left( \prod_{i=1 \to \dims, i \notin S} \cos(w_i) \right) \left( \prod_{j \in S} \sin(w_j) \right),
    \end{align*}
    where 
    \[
    c_b = 
    \begin{cases}
    -\sin(b),& \text{if } |S| = 4t+1 \text{ for some } t \in \mathbb{N}\\
    -\cos(b),& \text{if } |S| = 4t+2 \text{ for some } t \in \mathbb{N}\\
    \sin(b),& \text{if } |S| = 4t+3 \text{ for some } t \in \mathbb{N}\\
    \cos(b),& \text{if } |S| = 4t \text{ for some } t \in \mathbb{N}.\\
    \end{cases}
    \]
\end{corollary}

\section{Scaling}\label{sec:scaling}

\begin{figure}[!tbp]%
    \centering
    \begin{subfigure}{0.45\textwidth}
    \centering    \includegraphics[width=\textwidth]{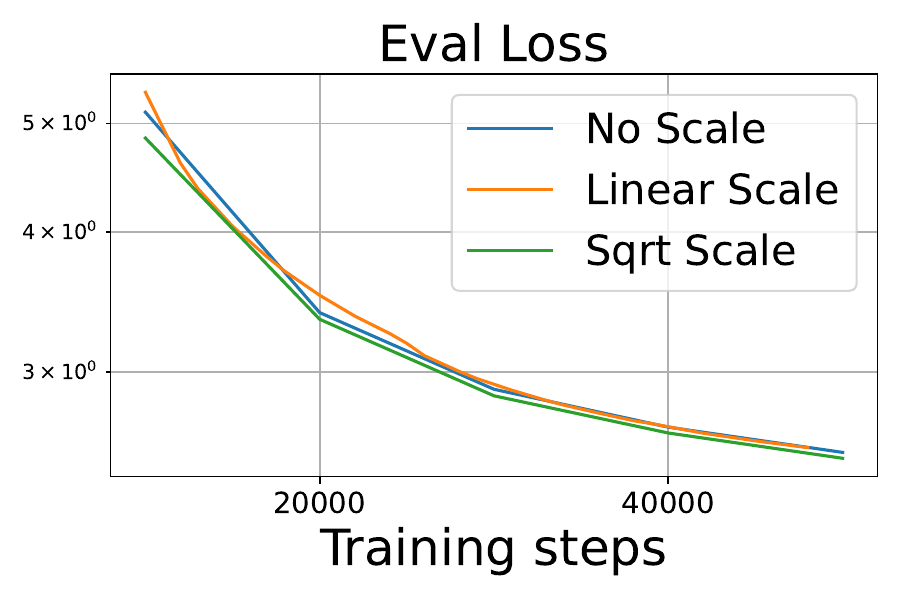}
    \label{fig:train_ablate}
    \end{subfigure}\hfill
    \begin{subfigure}{0.45\textwidth}
    \centering
    \includegraphics[width=\textwidth]{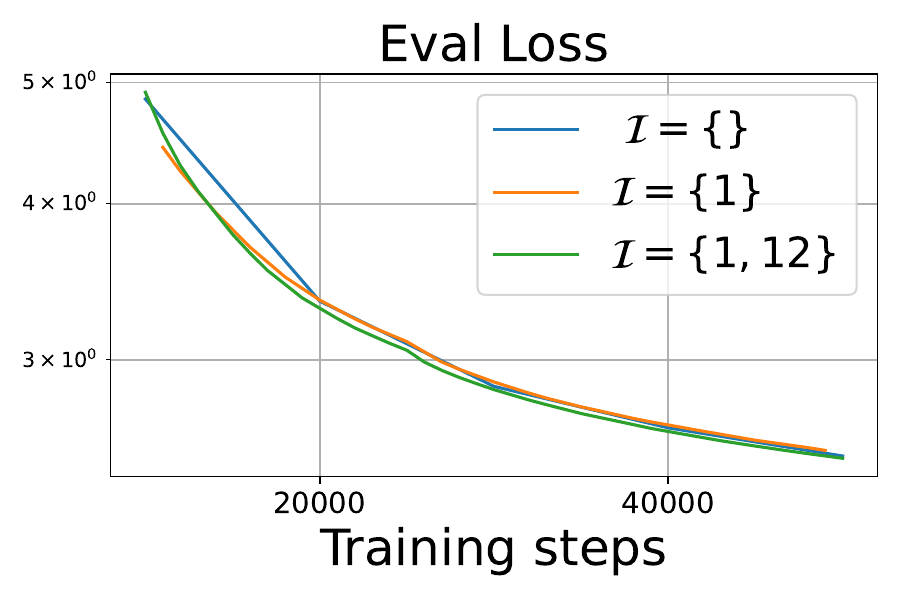}
    \label{fig:eval_ablate}
    \end{subfigure}\hfill

    \caption{\looseness-1Ablation study on choices for {\method} algorithm. We train BERT-base with {\method} for 100k steps with a 6-8-10-12 schedule (see \Cref{sec:experiments}). Left to right: (a) Square-root scaling ($\funcsqrt$) has better validation loss compared with linear scaling and no scaling at training, especially in the earlier stages. With \emph{no scale}, we scale layer's output by $1/p_i$ during inference  (following \cite{huang2016deep}). (b) Different candidates for \rebuttal{the} fixed set $\mathcal{I}$ are compared for {\method}. We find that training with first and last layers fixed  helps faster training. }
    \label{fig:bert100kruns}
\end{figure}

While working on subnetworks, it is important to appropriately rescale the network. In particular, bypassing layer(s) in {\method} can shift the input distribution for layers compared to the full model. To mitigate this shift, we scale the output of the layers to maintain the norms of the input to each layer throughout training. The idea of scaling layer outputs has also been explored in prior work \cite{huang2016deep,fan2019reducing}.%
We use a different square-root scaling mechanism that is motivated by analysis on a randomly initialized Transformer architecture. We show that at initialization, the norm of the intermediate outputs $\vy^{(\ell)}$ scale as $\sqrt{\ell}$.

\begin{restatable}[Square-root scaling of input norms]{theorem}{sqrtscale}\label{thm:sqrt_scale}    At initialization, given any input sequence $\vx_{1:\seqlen}$ and scalars $\alpha_{1:\layer}$, w.h.p. the intermediate sequences of $F$ (\cref{def:functiondef}) satisfy
\begin{align*}
    \norm{\vy^{(\ell)}_i}^2_2 =  \norm{\vx_i}^2_2 + \Theta ( \sum_{j=1}^{\ell} \alpha^2_{j} d), \quad \text{for all } 1 \le i \le \seqlen, 1 \le \ell \le \layer.
\end{align*}

\end{restatable}

Inspired by this, we define the function $\funcsqrt$ that takes in a sequence of binary values $\bernoulli_{1:\layer} \in \{0, 1\}^{\layer}$ and returns scalars $\overline{\bernoulli}_{1:\layer}$ where $\overline{\bernoulli}_{i}$ is the scaling for output of layer $i$. We find a scaling that satisfies the following two conditions hold (a) $\sum_{i=1}^{j} \overline{\bernoulli_i}^2 = j$ for all $j \le \layer$ (maintaining the norm), and (b) $\overline{\bernoulli_i} = 0$ if $\bernoulli_i = 0$ for all $i \le \layer$ (maintaining layers in random {\layerdrop}). Formally, for each index $j$ with $\bernoulli_j=1$, it finds the index minimum $\bar{j} > j$ with $\bernoulli_{\bar{j}} = 1$ and sets $\overline{\zeta}_j = \sqrt{\bar{j}-j}.$
The {\method} algorithm with square root scaling is presented in \cref{alg:layerdrop}.
In \Cref{fig:bert100kruns} we compare square root scaling with some other scaling methods and find it to be slightly better.

\subsection{Proofs}
For simplicity, we present the results for transformers with single head in the self-attention layer. Furthermore, for simplicity, we use $\relu$ activation in the MLPs.

\begin{algorithm}[!htbp]
\caption{Transformer Layer}\label{alg:transformerlayer}
\begin{algorithmic} [1]
\REQUIRE 2 layer normalization layers $\attnlnfun$, $\mlplnfun$ (\ref{def:layernorm}), an MLP layer $\mlpfun$ (\ref{def:mlp}), and a softmax self-attention layer $\attnfun$ (\ref{def:self-attn_single}), input sequence $\{ \vx_n \in \RR^{\dims} \}_{n=1}^{\seqlen}$.
\STATE Attention Layer normalization: returns $\{ \outlnattn_n \}_{n=1}^{\seqlen}$ with $\outlnattn_n = \attnlnfun(\vx_n)$ for all $n \le \seqlen$. 
\STATE Softmax self-attention: returns $\{ \outattn_n \}_{n=1}^{\seqlen} = \attnfun( \{ \outlnattn_n \}_{n=1}^{\seqlen} )$.  
\STATE Residual connection: returns $\{ \outattnblock_n \}_{n=1}^{\seqlen}$, with $\outattnblock_n = \vx_n + \outattn_n$ for all $n \le \seqlen$.
\STATE MLP Layer normalization: returns $\outlnmlp_n = \mlplnfun(\outattnblock_n)$ for all $n \le \seqlen$.
\STATE MLP function: returns $\{ \outmlp_n \}_{n=1}^{\seqlen}$ with $\outmlp_n = \mlpfun (\outlnmlp_n) $ for all $n \le \seqlen$.
\STATE Compute $\vy_n = \outmlp_n + \outattnblock_n$ for all $n \le \seqlen$.
\STATE Return $\{\vy_n\}_{n=1}^{\seqlen}$
\end{algorithmic}
\end{algorithm}

\begin{restatable}{definition}{layernormdefine}[Layer Normalization]\label{def:layernorm}
Define a normalization function $f:\RR^d\to\RR^d$ that performs $f(\vx) = (\vx-\mu) / \sigma$, where $\mu$ and $\sigma$ are the mean and standard deviation of $\vx$, respectively. Then, layer normalization $\lnfun: \RR^{\dims} \to \RR^{\dims}$ with parameters $\vgamma, \vb \in \RR^{\dims}$ takes as input  $\vx \in \RR^{\dims}$ and outputs $\vy\in\RR^{\dims}$, which is computed as $\vz = f(\vx), \vy = \vgamma \odot \vz + \vb.$
\end{restatable}

\begin{definition}[Softmax self-attention]
\label{def:self-attn_single}
    A self-attention layer $\attnfun: \RR^{\seqlen \times \dims} \to \RR^{\seqlen \times \dims}$ with parameters $\{\mW_Q, \mW_K, \mW_V, \mC^{attn} \in \RR^{\dims \times \dims} \}$ takes a sequence $\{ \vx_n \}_{n \le \seqlen}$ and outputs a sequence $\{ \vy_n \}_{n \le \seqlen}$, such that
    \begin{align*}
    &\vy_n = \mC^{attn} \sum_{\overline{n}=1}^{\seqlen} a_{n, \overline{n}} \vv_{\overline{n}}, \\&
    \text{with } a_{n, \overline{n}} =\mathrm{softmax}(\mK \vq_n)_{\overline{n}}, \quad 
    \vq_n = \mW_Q \vx_n, \quad \vk_n = \mW_K \vx_n, \quad \vv_n = \mW_V \vx_n,
    \end{align*}
    for all $n \le \seqlen$, and $\mK \in \RR^{ \seqlen \times  \seqlen }$ defined with rows $\{ \vk_n \}_{n=1}^{\seqlen}.$
\end{definition}

\begin{definition}[MLP]
\label{def:mlp}
    An MLP layer $\mlpfun: \RR^{\dims} \to \RR^{\dims}$ with parameters $\{\mW \in \RR^{\hid \times \dims}, \mC^{mlp} \in \RR^{\hid \times \dims} \}$ and activation $\relu$,
    takes an input $\vx \in \RR^{\dims}$ and outputs $\vy \in \RR^{\dims}$, such that
    \begin{align*}
        \vy = \mC^{mlp} \relu (\mW \vx).
    \end{align*}
\end{definition}

\begin{definition}[Transformer layer]\label{def:transformer_block}
   A pre-layernorm transformer layer with three sub-layers; 2 layer normalization layers $\attnlnfun$, $\mlplnfun$ (\ref{def:layernorm}), an MLP layer $\mlpfun$ (\ref{def:mlp}), and a softmax self-attention layer $\attnfun$ (\ref{def:self-attn_single}); takes a sequence $\{ \vx_{n \le \seqlen}$ and outputs a sequence $\{ \vy_n \}_{n \le \seqlen}$ in four steps.
   \begin{enumerate}
       \item First computes $\{ \outlnattn_n \}_{n=1}^{\seqlen}$ using a layer normalization, i.e. $\outlnattn_n = \attnlnfun(\vx_n)$ for all $n \le \seqlen$.
       \item Then it runs softmax self-attention to get $\{ \outattn_n \}_{n=1}^{\seqlen} = \attnfun( \{ \outlnattn_n \}_{n=1}^{\seqlen} )$.  
       \item  The net output of the self-attention block is given by $\{ \outattnblock_n \}_{n=1}^{\seqlen}$, with $\outattnblock_n = \vx_n + \outattn_n$ for all $n \le \seqlen$.
       \item Before passing to the MLP function, it is passed through another layer normalization, i.e. $\outlnmlp_n = \mlplnfun(\outattnblock_n)$ for all $n \le \seqlen$.
       \item MLP function then returns $\{ \outmlp_n \}_{n=1}^{\seqlen}$ with $\outmlp_n = \mlpfun (\outlnmlp_n) $ for all $n \le \seqlen$.
       \item $\vy_n = \outmlp_n + \outattnblock_n$ for all $n \le \seqlen$.
   \end{enumerate}
\end{definition}

\begin{definition}[Initialization of the weights in the transformer layer]\label{def:initialization}
    The weights are initialized as follows:
    \begin{align*}
        &\mC^{mlp}, \mC^{attn}, \mW_Q, \mW_K, \mW_V \sim \normal(\mathbf{0}, \frac{1}{\sqrt{\dims}} \mathbf{I}), \\&
        \mW \sim \normal(\mathbf{0}, \frac{\sqrt{2}}{\sqrt{\hid}} \mathbf{I}).
    \end{align*}
    The parameters $\vgamma, \vb$ of the functions $\attnlnfun, \mlplnfun$ have been initialized as $\mathbf{1}$ and $\mathbf{0}$ respectively.
\end{definition}

\begin{lemma}[Norm of the output of the MLP function, modification of lemma 7.1 in \cite{allen2019convergence}]\label{lem:mlpnorm}
    For a given input sequence $\{ \outlnmlp_n \}_{n \le \seqlen}$, if $\varepsilon \in(0,1]$, with probability at least $1-O(\seqlen) \cdot e^{-\Omega \left(\hid \varepsilon^2 \right)}$ over the randomness of $\mC^{mlp}, \mW$, we have
    $$
    \forall i \in [\seqlen]: \quad \norm{\outmlp_n}/\norm{\outlnmlp_n} \in [ 1-\varepsilon, 1+\varepsilon ].
    $$
\end{lemma}

\begin{lemma}[Norm of the output of the layer normalization layers]\label{lem:lnnorm}
Given any input sequence $\{ \vx_n \}_{n=1}^{\seqlen}$, under the assumption for all $n \le \seqlen$, $\vx_n - \sum_{i=1}^{\dims} x_{n, i}$ isn't identically $\mathbf{0}$, we have
\begin{align*}
    \norm{\outlnattn_n} = \sqrt{\dims}
\end{align*}
for all $n \leq \seqlen$.
\end{lemma}

Similar result holds for $\{ \outlnmlp_n \}_{n=1}^{\seqlen}.$

\begin{lemma}[Norm of the output of the MLP block] \label{lem:mlp_norm}
    For a given input sequence $\{ \outattnblock_n \}_{n \le \seqlen}$, with probability at least $1-\mathcal{O}(1)$ over the randomness of $\mC^{mlp}, \mW$, we have
    \begin{align*}
        \norm{ \vy_n }^2 &= \norm { \outmlp_n + \outattnblock_n }^2 =  \norm{\outattnblock_n}^2 + \dims \\& + \mathcal{O}\left( \norm{ \outattnblock_n } \log \seqlen + (\dims + \norm{\outattnblock_n}) \frac{\log^{3/2} \seqlen}{\sqrt{m}} \right),
    \end{align*}
    for all $n \le \seqlen.$
\end{lemma}

\begin{proof}
    Combining \cref{lem:lnnorm,lem:mlpnorm}, we have for the sequence $\{ \outattnblock_n \}_{n \le \seqlen}$,
    \begin{align*}
        \abs{ \norm{\outmlp_n} / \sqrt{\dims} - 1 } \le \mathcal{O}\left( \frac{\sqrt{\log \seqlen}}{\sqrt{m}} \right),
    \end{align*}
    w.p. atleast $1 - \mathcal{O}(1)$.

    Furthermore, due to the randomness of $\mC^{mlp}$, we can further show that w.p. at least $1-\mathcal{O}(1)$, 
    \begin{align*}
        \abs{\langle \outattnblock_n, \outmlp_n \rangle} \le \mathcal{O}(\norm{\outattnblock_n} \log \seqlen),
    \end{align*}
    for all $n \le \seqlen$. Combining the two results, we have
    \begin{align*}
        \norm{ \vy_n }^2 &= \norm { \outmlp_n + \outattnblock_n }^2 =  \norm{\outattnblock_n}^2 + \dims \\& + \mathcal{O}\left( \norm{ \outattnblock_n } \log \seqlen + (\dims + \norm{\outattnblock_n}) \frac{\log^{3/2} \seqlen}{\sqrt{m}} \right).
    \end{align*}
\end{proof}

\begin{lemma}[Norm of the output of the attention block]
    For a given input sequence $\{ \vx_n \}_{n \le \seqlen}$, if $\varepsilon \in(0,1]$, with probability at least $1-O(1)$ over the randomness of $\mC^{attn}, \mW_V$, we have
    $$
    \forall i \in [\seqlen]: \quad \norm{ \vx_n }^2 \leq \norm{ \outattnblock_n }^2 \leq \norm{ \vx_n }^2 + \norm{ \vx_n } + \dims + \mathcal{O}( \frac{\sqrt{\log N}}{\sqrt{\dims}} (\dims + \norm{ \vx_n }) ).
    $$
\end{lemma}

\begin{proof}
    From the definitions of $\{\outlnattn_n\}$ and $\{  \outattn_n \}$, we have
    \begin{align*}
     \outattn_n =   \mC^{attn} \mW_V \sum_{j \le \seqlen} a_{n, j} \outlnattn_j.
    \end{align*}

    Thus, we can use a proof similar to the proof of \cref{lem:mlp_norm} to argue that with the randomness of $\mC^{attn}$ and $\mW_V$,
    \begin{align*}
        \abs{ \frac{ \norm{ \outattn_n } } { \norm{ \sum_{j \le \seqlen} a_{n, j} \outlnattn_j } } - 1} \le \mathcal{O}\left( \frac{\sqrt{\log \seqlen}}{\sqrt{\dims}} \right),
    \end{align*}
    for all $n \le \seqlen$ w.p. atleast $1 - \mathcal{O}(1).$
    
    Furthermore, due to the randomness of $\mC^{attn}$, we can further show that w.p. atleast $1 - \mathcal{O}(1)$,
    \begin{align*}
        \abs{ \langle \outattn_n , \vx_n \rangle } \le \mathcal{O}(\norm{\vx_n} \log \seqlen),
    \end{align*}
    for all $n \le \seqlen$.

    Using Cauchy-Schwarz inequality, we must have
    \begin{align*}
        \norm{ \sum_{j \le \seqlen} a_{n, j} \outlnattn_j } \leq \max_{j \le \seqlen} \norm {\outlnattn_j} = \sqrt{\dims}.
    \end{align*}

    Thus, combining the results, we have 
    \begin{align*}
        \norm{ \outattnblock_n }^2 &= \norm { \outattn_n + \vx_n }^2 \leq \norm{\vx_n}^2 + \dims \\& + \mathcal{O}\left( \norm{ \vx_n } \log \seqlen + (\dims + \norm{\vx_n}) \frac{\log^{3/2} \seqlen}{\sqrt{\dims}} \right).
    \end{align*}
\end{proof}

\begin{lemma}[Norm of the output of the transformer layer] \label{lem:single_layer}
    For a given input sequence $\{ \vx_n \}_{n \le \seqlen}$, if $\varepsilon \in(0,1]$, with probability at least $1-\mathcal{O}(1)$ over the randomness of $\mC^{attn}, \mW_V, \mC^{mlp}$, we have
    $$
    \forall i \in [\seqlen]: \quad  \norm{ \vx_n }^2 +  \dims + \mathcal{O}(err) \leq \norm{ \vy_n }^2 \leq \norm{ \vx_n }^2  + 2\dims + \mathcal{O}(err),
    $$
    where $err = \mathcal{O}\left( \norm{ \vx_n } \log \seqlen + (\dims + \norm{\vx_n}) \frac{\log^{3/2} \seqlen}{\sqrt{m}} \right)$.
\end{lemma}

\sqrtscale*

\begin{proof}
    This follows from the fact that the transformer architecture is a stack of struturally identical $\layer$ transformer layers, and each transformer layer's output norm increases by $\Theta(\dims)$ compared to its input norm, as given by \cref{lem:single_layer}.
\end{proof}

\end{document}